\documentclass[10pt,journal,cspaper,compsoc]{IEEEtran}
\usepackage[export]{adjustbox}
\usepackage{multirow}
\usepackage{times}
\usepackage{epsfig}
\usepackage{graphicx}
\usepackage{amsmath}
\usepackage{amssymb}
\usepackage{color}
\usepackage{rotating}
\usepackage[linesnumbered]{algorithm2e}

\ifCLASSOPTIONcompsoc
\else
\fi
\ifCLASSINFOpdf
\else
\fi

\def\1{\boldsymbol{1}}
\def\0{\boldsymbol{0}}

\def\sign{\textrm{sign}}

\def\real{\mathbf{R}}

\newtheorem{theorem}{Theorem}[section]

\newtheorem{proposition}[theorem]{Proposition}

\def\QED{~\rule[-1pt]{5pt}{5pt}\par\medskip}
\newenvironment{proof}{\emph{Proof.}}{\hfill\QED}

\def\ie{\emph{i.e.}}
\def\eg{\emph{e.g.}}

\def\etal{\emph{et al.}}

\begin{document}
%
\title{A Clearer Picture of Blind Deconvolution}

%
%
%
%

\author{Daniele Perrone,
        Paolo Favaro,~\IEEEmembership{Member,~IEEE}
\IEEEcompsocitemizethanks{\IEEEcompsocthanksitem D. Perrone and P. Favaro are with the Institute of Computer Science and Applied Mathematics, University of Bern, Nebr\"{u}ckstrasse 10, 3012 Bern, Switzerland.\protect\\
E-mail: perrone@iam.unibe.ch; paolo.favaro@iam.unibe.ch
}
\thanks{}}

%
%

\markboth{ Submitted to IEEE TRANSACTION ON PATTERN ANALYSIS AND MACHINE INTELLIGENCE}%
{Perrone and Favaro: A Clearer Picture of Blind Deconvolution}
%


\IEEEcompsoctitleabstractindextext{%
\begin{abstract}
Blind deconvolution is the problem of recovering a sharp image and a blur kernel from a noisy blurry image. Recently, there has been a significant effort on understanding the basic mechanisms to solve blind deconvolution. While this effort resulted in the deployment of effective algorithms, the theoretical findings generated contrasting views on why these approaches worked. On the one hand, one could observe experimentally that alternating energy minimization algorithms converge to the desired solution. On the other hand, it has been shown that such alternating minimization algorithms should fail to converge and one should instead use a so-called Variational Bayes approach.
To clarify this conundrum, recent work showed that a good image and blur prior is instead what makes a blind deconvolution algorithm work. Unfortunately, this analysis did not apply to algorithms based on total variation regularization. In this manuscript, we provide both analysis and experiments to get a clearer picture of blind deconvolution. Our analysis reveals the very reason why an algorithm based on total variation works. We also introduce an implementation of this algorithm and show that, in spite of its extreme simplicity, it is very robust and achieves a performance comparable to the state of the art. 
\end{abstract}

}

\maketitle

\IEEEdisplaynotcompsoctitleabstractindextext

%
\IEEEpeerreviewmaketitle

%

Blind deconvolution is the problem of recovering a signal and a degradation kernel from their noisy convolution. This problem is found in diverse fields such as astronomical imaging, medical imaging, (audio) signal processing, and image processing. 
More recently, blind deconvolution has received renewed attention due to the emerging need for removing motion blur in images captured by mobile phones \cite{Fergus2006}. 
Yet, despite over three decades of research in the field (see \cite{KundurH96} and references therein), the design and analysis of a principled, stable and robust algorithm that can handle real images remains a challenge. However, present-day progress has shown that recent models for sharp images and blur kernels can yield remarkable results \cite{Fergus2006,Shan2008,Cho2009,Xu2010,Levin2011}.

Many of these recent approaches are evolutions of the variational formulation \cite{You1996}. A common element in these methods is the explicit use of priors for both blur and sharp image to encourage smoothness in the solution.
Among these recent methods, \emph{total variation} emerged as one of the most popular priors  \cite{Chan1998,You1996Anisotropic}. Such popularity is probably due to its ability to encourage gradient sparsity, a property that can describe many signals of interest well \cite{HuangM99}. 

However, recent work by Levin et al.~\cite{Levin2011Understanding} has shown that the joint optimization of both image and blur kernel can have the no-blur solution as its global minimum. That is to say, blind deconvolution algorithms that use a total variation prior either are local minimizers and, hence, require a lucky initial guess, or they cannot depart too much from the no-blur solution. Nonetheless, algorithms based on the joint optimization of blur and sharp image show good convergence behavior even when initialized with the no-blur solution \cite{Chan1998,Shan2008}.

This incongruence called for an in-depth analysis of total variation blind deconvolution (TVBD). As we show in the next sections, the answer is not as straightforward as one might have hoped. Firstly, we confirm both experimentally and analytically the analysis of Levin et al.~\cite{Levin2011Understanding}. Secondly, we also find that the algorithm of Chan and Wong~\cite{Chan1998} converges to the desired solution, even when starting at the no-blur solution. The answer to this puzzle lies in the specific implementation of \cite{Chan1998}, as it does not minimize the originally-defined energy. This algorithm, as commonly done in many other algorithms, separates some constraints from the gradient descent step and then applies them sequentially. When the cost functional is convex this alteration may not have a major impact. However, in blind deconvolution, where the cost functional is not convex, this completely changes the convergence behavior. Indeed, we show that if one imposed all the constraints simultaneously  then the algorithm would never leave the no-blur solution independently of the regularization amount. 

To further demonstrate our findings, we implement a TVBD algorithm without the use of all recent improvements, such as filtering \cite{Shan2008,Cho2009,Levin2011Understanding}, blur kernel prior \cite{Chan1998,You1996Anisotropic} or edge enhancement \cite{Cho2009,Xu2010}, and show that applying sequentially the constraints on the blur kernel is sufficient to avoid the no-blur solution. We also show that the use of a filtered version of the original signal may be undesirable and that the use of exact boundary conditions can improve the results. Finally, we apply the algorithm on currently available datasets, compare it to the state of the art methods and show that, notwithstanding its simplicity, it achieves a comparable performance to the top performers.

\section{Blur Model and Priors}
Suppose that a blurry image $f$ can be modeled by
\begin{align}
f = k_0 \ast u_0+n
\end{align}
where $k_0$ is a blur kernel, $u_0$ a sharp image, $n$ noise and $k_0\ast u_0$ denotes convolution between $k_0$ and $u_0$. Given only the blurry image, one might want to recover both the sharp image and the blur kernel. This task is called \emph{blind deconvolution}. A classic approach to this problem is to maximize the posterior distribution (MAP$_{u,k}$)
\begin{equation}\label{eq:map}
\arg\max_{u,k} p(u,k|f) = \arg\max_{u,k} p(f|u,k)p(u)p(k). 
\end{equation}
where $p(f|u,k)$ models the noise affecting the blurry image, $p(u)$ models the distribution of typical sharp images, and $p(k)$ is the prior knowledge about the blur function. Typical choices for $p(f|u,k)$ are the Gaussian distribution \cite{Fergus2006,Levin2011} or the exponential distribution \cite{Xu2010}. In the following discussion  we will assume that $p(f|u,k)$  is modeled by a Gaussian distribution.

Through some standard transformations and assumptions, problem~\eqref{eq:map} can be written also as a regularized minimization
\begin{equation}\label{eq:gendeb}
\arg\min_{u,k} \|k \ast u - f\|_2^2 + \lambda J(u) + \gamma G(k)
\end{equation}
where the first term corresponds to $p(f|u,k)$, the functionals $J(u)$ and $G(k)$ are the smoothness priors for $u$ and $k$ (for example, Tikhonov regularizers \cite{Tikhonov1977} on the gradients), and $\lambda$ and $\gamma$ two nonnegative parameters that weigh their contribution. 
Furthermore, additional constraints on $k$, such as positivity of its entries and integration to $1$, can be included.
For any $\lambda>0$ and $\gamma>0$ the cost functional will not have as global solution neither the true solution $(k_0,~u_0)$ nor the \emph{no-blur solution} $(k = \delta,~u = f)$, where $\delta$ denotes the Dirac delta. Indeed, eq.~\eqref{eq:gendeb} will find an optimal tradeoff between the data fitting term and the regularization term. Nonetheless, one important aspect that we will discuss later on is that both the true solution $(k_0,~u_0)$ and the \emph{no-blur solution} make the data fitting term in eq.~\eqref{eq:gendeb} equal to zero. Hence, we can compare their cost in the functional simply by evaluating the regularization terms. Notice also that the minimization objective in eq.~\eqref{eq:gendeb} is non-convex, and, as shown in Fig.~\ref{fig:1dtoy}, has several local minima.

\begin{figure*}[t]
\centering
\begin{minipage}[c]{.42\textwidth}
\centering
\adjincludegraphics[width=\textwidth,trim={0 1.6cm 0  {.5\height}},clip]{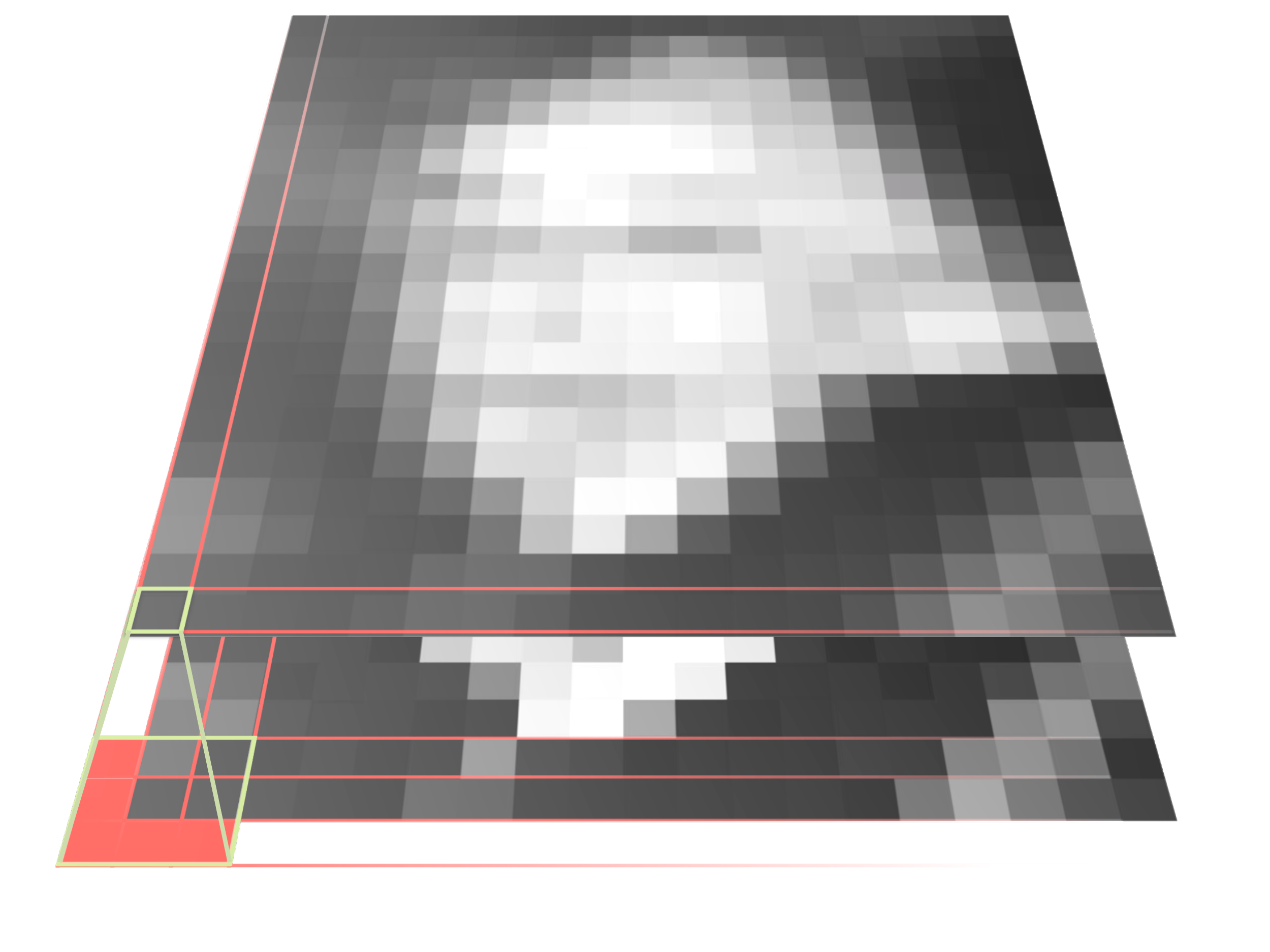}
\footnotesize{(a)} 
\end{minipage}
\begin{minipage}[c]{.42\textwidth}
\centering
\adjincludegraphics[width=\textwidth,trim={0 0.8cm 0  {0.555\height}},clip]{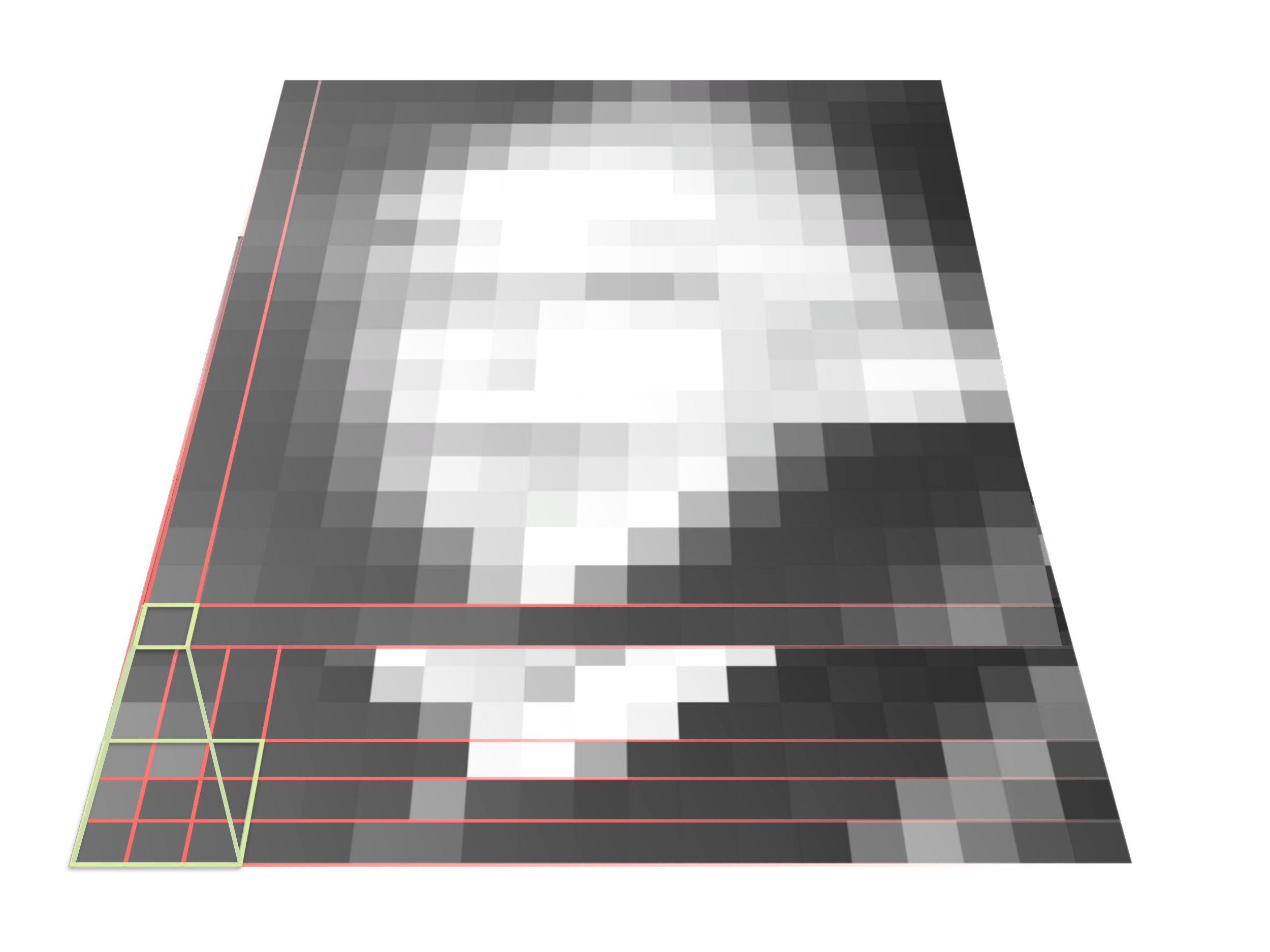}
\footnotesize{(b)}
\end{minipage}
\medskip
\caption{Illustration of  the convolution operators~\eqref{eq:conv} and~\eqref{eq:vconv} assuming a blur support of $3 \times 3$. a) With the use of~\eqref{eq:conv} we assume that the blurry image $f$ and the sharp image $u$ have the same support, therefore we must choose how the values at the boundaries of $u$ are defined (red pixels); b) With~\eqref{eq:vconv} we assume that $f$ has a smaller support than $u$, therefore the pixels of $f$ are completely defined by the pixels of $u$.\label{fig:conv} }
\end{figure*}

\section{Prior work}
To solve problem~\eqref{eq:gendeb}  one has to carefully choose the functions $J(u)$ and $G(k)$. A common choice is the $L_p$ norm of $u$ and $k$ or of some filtered versions of $u$ and $k$.

$G(k)$ has been defined as the $L_2$ norm of $k$ \cite{You1996,Xu2010,Cho2009}, a sparsity-inducing norm ($p \le 1$)~\cite{Fergus2006,Shan2008} or a constant~\cite{Levin2011}. Nonetheless, its contribution to the estimation of the desired solution is so far marginal. In fact, some methods successfully solve problem~\eqref{eq:gendeb} by setting $G(k) = \mbox{const}$. 
Yet, it has been shown that its use may help avoid undesired local minima~\cite{Rameshan2012}.

The regularization term for the sharp image $J(u)$ instead has a stronger impact on the performance of the algorithm, since it helps choose a sharp image over a blurry one. 
You and Kaveh~\cite{You1996} have proposed to use  the $L_2$ norm of the derivatives of $u$. Unfortunately, this norm is not able to model the sparse nature of common image gradients and results in sharp images that are either oversmoothed or have ringing artifacts. Yet, the $L_2$ norm has the desirable property of being efficient to minimize. Cho and Lee~\cite{Cho2009} and Xu and Jia~\cite{Xu2010} have reduced the generation of artifacts while still retaining its computational efficiency by using heuristics to select sharp edges. 

An alternative to the $L_2$ norm is the use of \emph{total variation} (TV)  \cite{Chan1998,You1996Anisotropic,Chan2000,He2005}. TV regularization was firstly introduced for image denoising in the seminal work of Rudin, Osher and Fatemi~\cite{Rudin1992}, and since then it has been applied successfully in many image processing applications. Total variation is typically defined via two different formulations. Its anisotropic version is the sum of the $L_1$ norms of the components of the gradient of $u$, while its isotropic version is the $L_2$ norm of the gradient of $u$. 

You and Kaveh~\cite{You1996Anisotropic} and Chan and Wong~\cite{Chan1998} have proposed the use of TV regularization in blind deconvolution on both $u$ and $k$. They also consider the following additional convex constraints to enhance the convergence of their algorithms
\begin{equation}\label{eq:constr}
\|k\|_1 \doteq \int |k(\mathbf{x})|d\mathbf{x} = 1,\quad
k(\mathbf{x}) \geq 0,\;u(\mathbf{x}) \geq 0
\end{equation} 
where with $\mathbf{x}$ we denote either 1D or 2D coordinates. He \etal~\cite{He2005} have incorporated the above constraints in a variational model, claiming that this enhances the stability of the algorithm. 
A different approach is a strategy proposed by Wang \etal~\cite{Wang2010} that seeks for the desired local minimum by using downsampled reconstructed images as priors during the optimization in a multi-scale framework. 

TV regularization has been widely popularized because it models natural image gradients well~\cite{HuangM99}. Wipf and Zhang~\cite{Wipf2013} have recently argued that $J(u)$ should not merely try to model statistics of sharp images, but, rather, it should have a strong discriminative power in separating sharp from blurry images. This is ideally achieved by using the $L_0$ pseudo-norm on the image gradients. Unfortunately its exact minimization requires solving an NP-hard problem. To make the problem tractable some methods have proposed to use approximations of the $L_0$ norm~\cite{Xu2013,Krishnan2011}. Yet, the proposed approximations are non-convex functions that require careful minimization strategies to avoid local minima. 

The algorithms presented so far are all successful implementation of the MAP$_{u,k}$ formulation in~\eqref{eq:map}, nonetheless Levin~\etal~\cite{Levin2011Understanding} have shown that using an $L_p$ norm of the image gradients with any $p > 0$ and a uniform distribution for the blur, the MAP$_{u,k}$ approach favors the no-blur solution $(u = f, \; k = \delta)$, for images blurred with a large enough blur.  They also argue that the success of existing MAP$_{u,k}$ methods is due to various heuristics or reweighing strategies employed during the optimization of~\eqref{eq:gendeb}.

Because of the above concerns, Levin \etal~\cite{Levin2011Understanding} look at a different strategy that marginalizes over all possible sharp images $u$. They solve the following MAP$_k$ problem
\begin{equation}\label{eq:mar}
\arg\max_{k} p(k|f) = \arg\max_{k} \int p(f|u,k)p(u)p(k) du,
\end{equation}
where the sharp image $u$ is estimated by solving a convex problem and where $k$ is given from the previous step. They have shown that, for sufficiently large images, the MAP$_k$ approach converges to the true solution.  Since the right hand side of problem~\eqref{eq:mar} is difficult to compute, it is common to use a Variational Bayesian approach (VB) where one aims at finding an approximation of the probability $p(k|f)$~\cite{Miskin2000,Fergus2006,Babacan2012,Wipf2013,Levin2011}. 

In recent work, Wipf and Zhang~\cite{Wipf2013} have shed new light on the MAP$_{u,k}$  vs MAP$_k$ dispute. 
They have shown that the VB formulation commonly used to solve the MAP$_k$ is equivalent to a non-conventional MAP$_{u,k}$ approach, and that $L_p$ priors with $p \le 0.5$ are able to favor sharp images. They also argue that a VB approach is still preferable because it is more able to avoid local minima compared to a classical MAP$_{u,k}$ approach. 

The work of Wipf and Zhang~\cite{Wipf2013} has focused on $L_p$ priors with $p< 1$ and given novel insights on the mechanism that make them work. 
Our work complements the results of Wipf and Zhang~\cite{Wipf2013} and focuses on the total variation ($p = 1$) prior. We confirm the theoretical results of Levin~\etal~\cite{Levin2011Understanding} and show that an apparently harmless delayed normalization induces a scaling of the signal that ultimately results in the success of total variation based algorithms. This shows that filtered version of the images \cite{Shan2008,Cho2009,Levin2011Understanding}, blur kernel prior \cite{Chan1998,You1996Anisotropic}, edge enhancement \cite{Cho2009,Xu2010} or any other additional strategy are not necessary for solving blind deconvolution.

\section{Problem Formulation}
By combining problem~\eqref{eq:gendeb} with the constraints in eq.~\eqref{eq:constr} and by setting $\gamma=0$, we study the following minimization
\begin{equation}\label{eq:glob}
\begin{array}{ll}
\min_{u,k} & ||k \ast u - f||_2^2 + \lambda J(u)\\
\mbox{subject to } & k \succcurlyeq 0,\quad \|k\|_1 = 1
\end{array}
\end{equation}
where $J(u) = ||u||_{BV} \doteq \int ||\nabla u(\mathbf{x})||_2 d\mathbf{x}$ or $J(u) = ||u_x||_1 + ||u_y||_1$, with $\nabla u \doteq [u_x~u_y]^T$ the gradient of $u$ and $\mathbf{x}\doteq[x~y]^T$, and $\|k\|_1$ corresponds to the $L_1$ norm in eq.~\eqref{eq:constr}. To keep the analysis simple we have stripped the formulation of all unnecessary improvements such as using a basis of filters in $J(u)$~\cite{Shan2008,Cho2009,Levin2011Understanding}, or performing some selection of regions by reweighing the data fitting term~\cite{Xu2010}, or enhancing the edges of the blurry image $f$~\cite{Xu2010,Cho2009}. Compared to previous methods, we also do not use any regularization on the blur kernel ($\gamma=0$).

The formulation in eq.~\eqref{eq:glob} involves the minimization of a constrained non-convex optimization problem. Also, notice that if $(u,k)$ is a solution, then $(u(\mathbf{x} + \mathbf{c}),k(\mathbf{x} + \mathbf{c}))$ are solutions as well for any $\mathbf{c}\in\real^2$. If the additional constraints on $k$ were not used, then the ambiguities would also include $(\alpha_1 u,\frac{1}{\alpha_1}k)$ for non zero $\alpha_1$.

\begin{figure*}[t]
\centering
\begin{minipage}[c]{.28\textwidth}
\centering
\includegraphics[width=\textwidth]{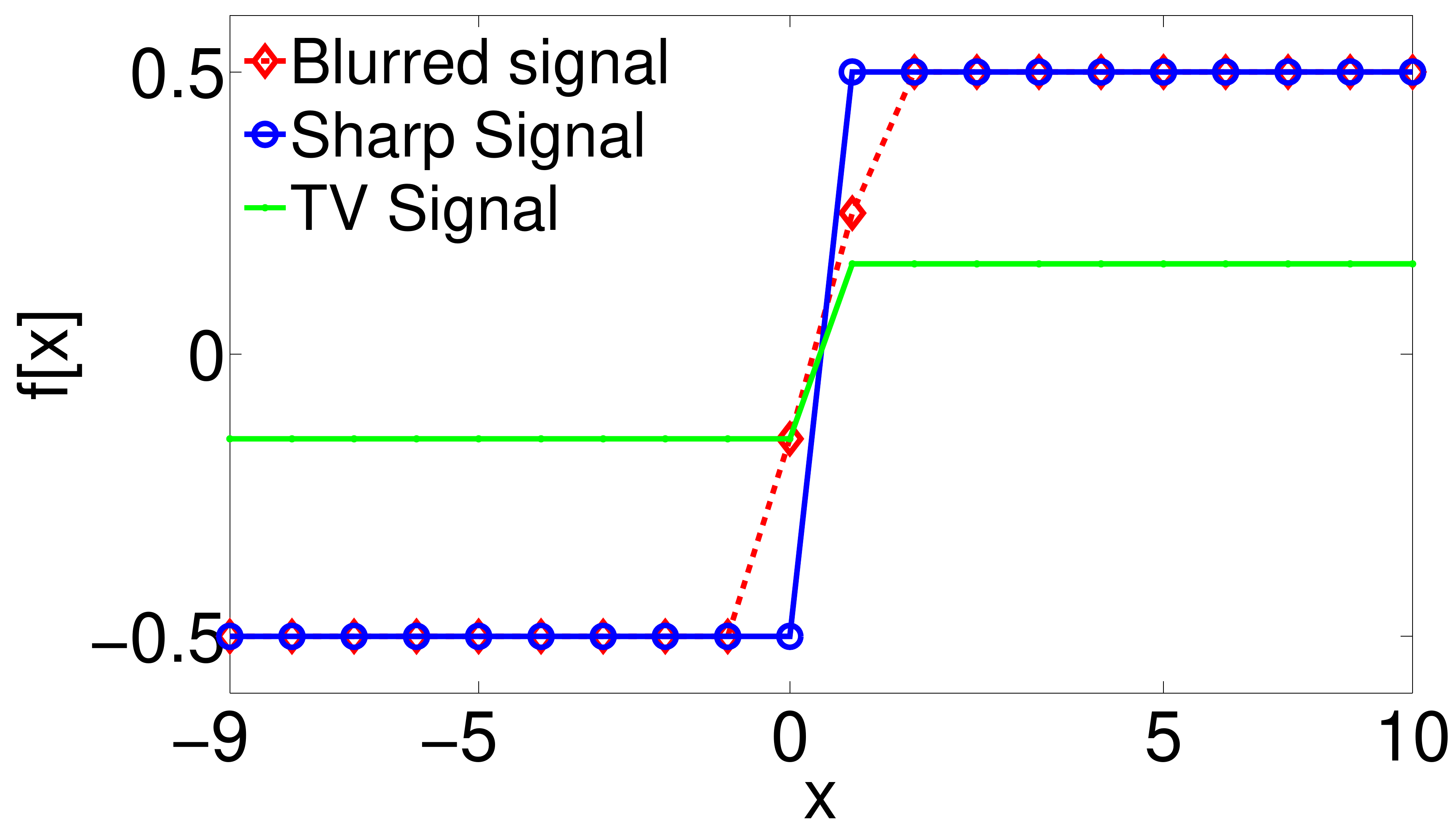} \\
\footnotesize{(a)}
\end{minipage}
\begin{minipage}[c]{.28\textwidth}
\centering
\includegraphics[width=\textwidth]{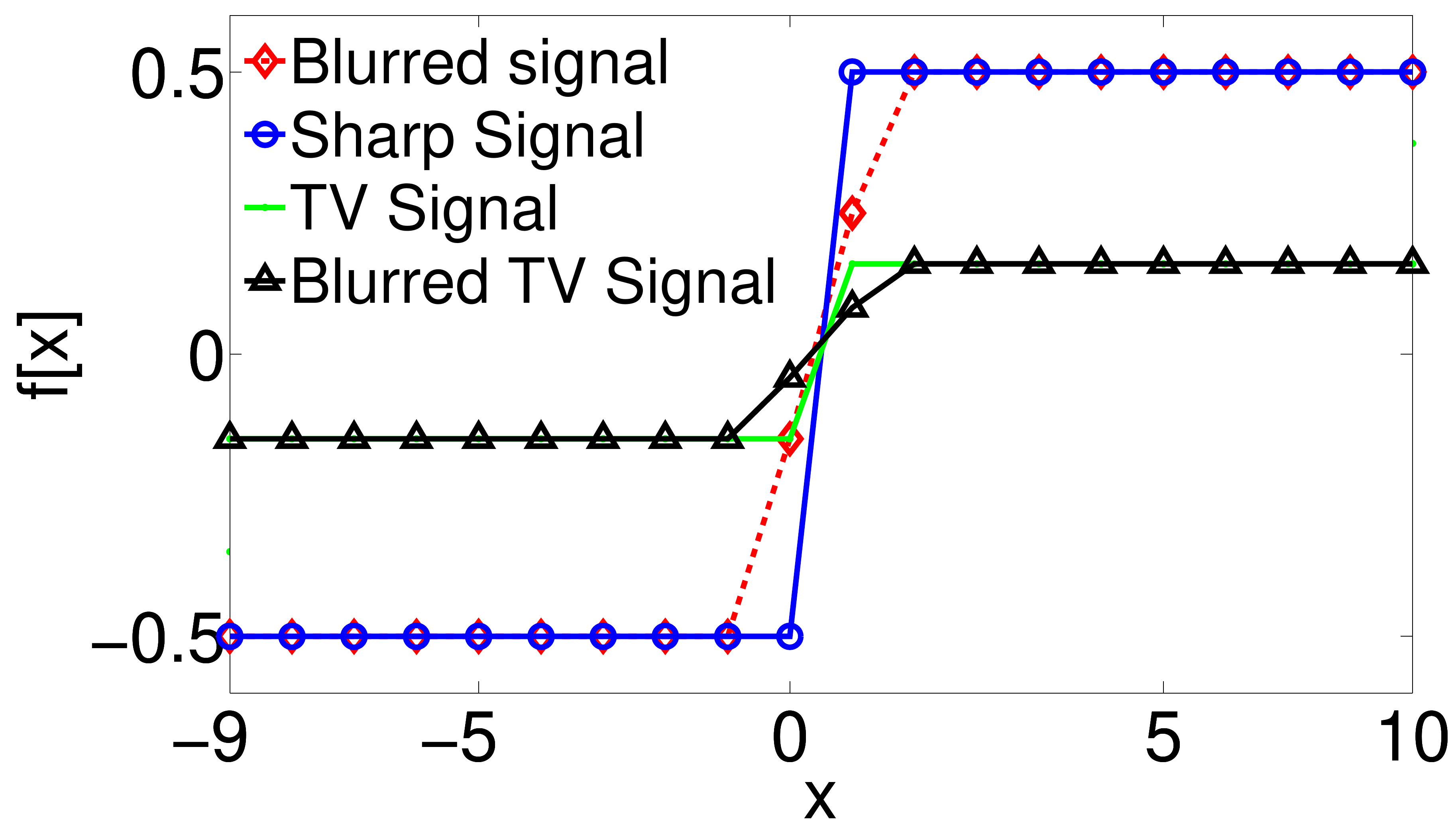}\\
\footnotesize{(b)}
\end{minipage}
\begin{minipage}[c]{.28\textwidth}
\centering
\includegraphics[width=\textwidth]{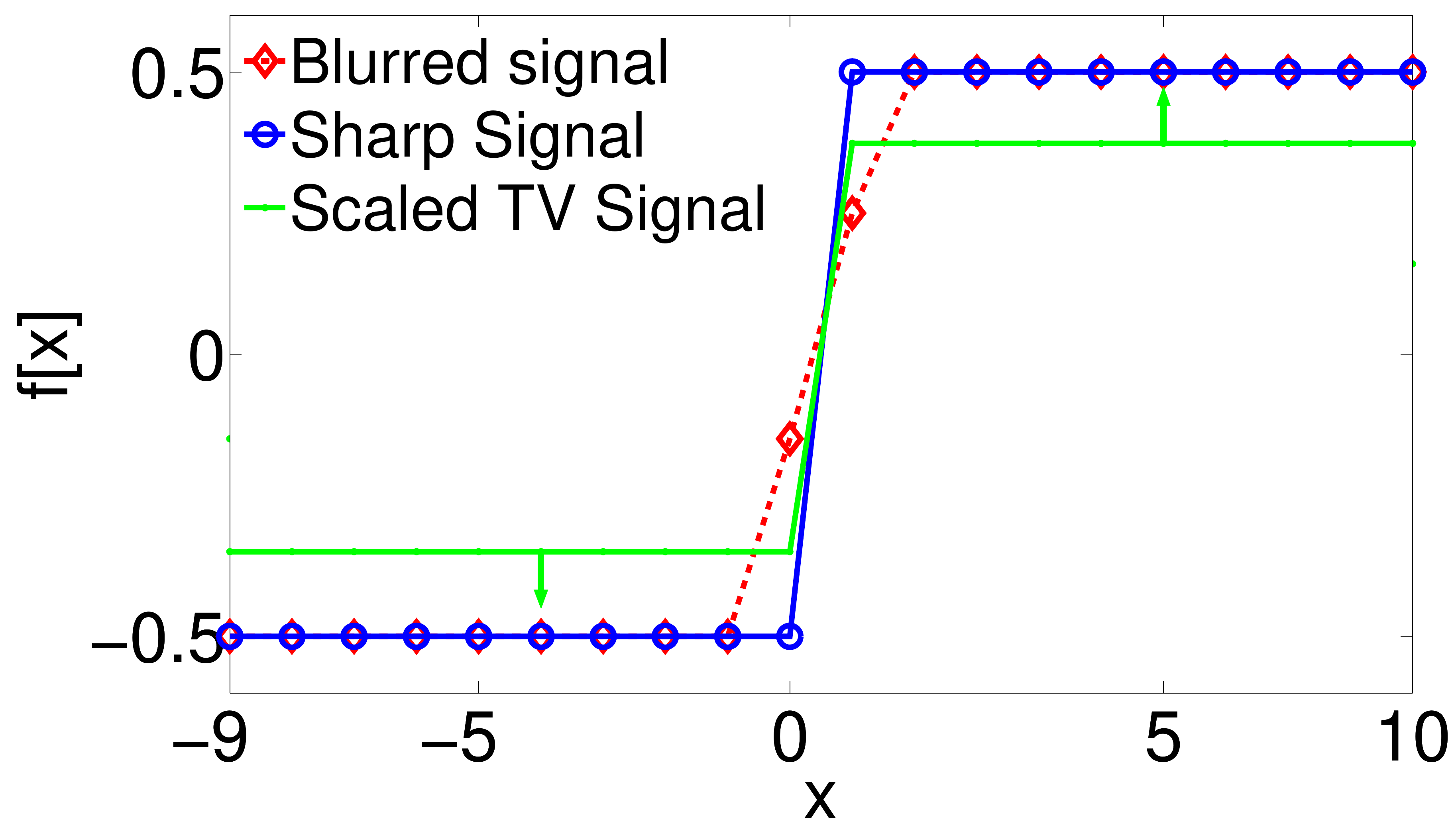}\\
\footnotesize{(c)}
\end{minipage}
\medskip
\caption{ Illustration of Theorem~\ref{the:pam} and Theorem~\ref{the:am} (best viewed in color). The original step function is denoted by a solid-blue line. The TV signal (green-solid) is obtained by solving $\arg\min_u \|u-f\|_2^2+\lambda J(u)$. In (a) we show how the TV denoising algorithm reduces the contrast of a blurred step function (red-dotted). In (b) we illustrate Theorem~\ref{the:am}: If the constraints on the blur are enforced, any blur different from the Dirac delta increases the distance between the input blurry signal and the blurry TV signal (black-solid). In (c) we illustrate Theorem~\ref{the:pam}: In the second step of the PAM algorithm, estimating a blur kernel without a normalization constraint is equivalent to scaling the TV signal. \label{fig:1dtoyt}} 
\end{figure*}

\subsection{Convolution Models and Notation }
The convolution operator in the minimization~\eqref{eq:glob} usually requires some assumptions on the boundaries of the image. We instead propose a new formulation that does not make any boundary assumptions. 

Let $u$ and $f$ be matrices with the same support $m \times n$, and $k$ a matrix with support $h \times w$.  Typically the discrete convolution of $u$ and $k$ is defined by 
\begin{equation}\label{eq:conv}
f = (u \ast k)[i,j] \doteq \sum_{r = -h/2}^{h/2} \sum_{c = -w/2}^{w/2} u[i-r,j-c]k[r,c] 
\end{equation}
for $i = 1,\dots, m$, $j = 1,\dots, n$, where some assumptions are made on the values outside the support of $u$ (Fig.~\ref{fig:conv}a). Commonly used assumptions in the literature are: \emph{symmetric}, where the boundary of the image is mirrored to fill the additional frame around the image; \emph{periodic}, where the image is padded with a periodic repetition of the boundary; \emph{replicate}, where the borders continue with a constant value. The periodic assumption is particularly convenient because it allows the use of the circular convolution theorem and the Discrete Fourier Transform (DFT) to achieve a fast performance. Because real images are rarely periodic,  Liu and Jia \cite{Liu2008} have proposed to extend the size of the blurry image to make it periodic.  Nonetheless, each of the above choices is an approximation of the real physical phenomenon and therefore it introduces an error in the reconstruction of the sharp image.

In this paper we propose to use a different approach, where the blurry image $f$ has support $m - h + 1 \times n - w -1$. In this case we define a new convolution operator, denoted by $\circ$, as
\begin{equation}\label{eq:vconv}
f = (u \circ k)[i,j] \doteq \sum_{r = 0}^{h-1} \sum_{c = 0}^{w-1} u[i+r,j+c] k_- [r,c] 
\end{equation}
for $i = 1,\dots, m - h + 1$, $j = 1,\dots, n - w + 1$ and where $k_- [r,c] = k[h - r, w - c]$ (Fig.~\ref{fig:conv}b). Notice that $k \circ u \neq u \circ k$ in general. Also,  $u \circ k$ is not defined if the support of $k$ is too large ($h>m+1$ and $w>n+1$ ). 

In the following we will choose  $J(u)$ to be the isotropic total variation $J(u) = ||u||_{BV} \doteq \int ||\nabla u(\mathbf{x})||_2 d\mathbf{x}$. By incorporating the above considerations in a discrete setting, the problem in~\eqref{eq:glob} can be written as

\begin{equation}\label{eq:dglob}
\begin{array}{ll}
\displaystyle\min_{u,k} &  \displaystyle\sum_{\mathbf x \in \mathcal{F}} ((k \circ u)[\mathbf x] - f[\mathbf x])^2 + \sum_{\mathbf x \in \mathcal{U}}  \| (\nabla u) [\mathbf x] \|_2\\
 &\mbox{subject to } \displaystyle k \succcurlyeq 0,\quad \|k\|_1 = 1
\end{array}
\end{equation}
where $\mathbf x = (i, j)$, $\mathcal{F} = \{1, \dots, m+h-1\} \times \{1, \dots, n+w-1\}$ and $\mathcal{U} = \{1, \dots, m\} \times \{1, \dots, n\}$\footnote{Notice that  $\mathcal{F}$ represents the support of $f$ and $\mathcal{U}$ the support of $u$.}. 

Since in~\eqref{eq:dglob} the domain of the sharp image $u$ is larger than the domain of the blurry image $f$, solving~\eqref{eq:dglob} requires the estimation of more variables than measurements. This problem is tackled by the total variation term that imposes a weak smoothness constraint beyond the boundary of $f$, instead of the hard assumptions needed when using eq.~\eqref{eq:conv}. We will show that using this formulation gives better results compared to the typical approach that uses eq.~\eqref{eq:conv}. 

While the analysis in our previous paper~\cite{Perrone2014} is based on the circular convolution operator, this paper is entirely based on the convolution operator defined in eq.~\eqref{eq:vconv}.

\section{Analysis of Total Variation Blind Deconvolution}
Recent analysis has highlighted important limitations of the total variation prior regarding the blind deconvolution problem. Still, as mentioned in the Prior Work section many algorithms successfully employ total variation for solving this problem. In this section we confirm the limitations of total variation highlighted in previous work~\cite{Levin2011Understanding}, and show how a small detail in the minimization strategy commonly used by many algorithms allows them to avoid local minima and estimate the correct blur. 

In Section~\ref{sec:local_minima} we show that an $L_p$ norm with $p \ge 1$ can not favor the reconstruction of a sharp image over a blurry one. Therefore, when used in problem~\eqref{eq:dglob} it does not yield the desired solution.

To understand why many algorithms still succeed,  in Section~\ref{sec:am} we study the Alternating Minimization (AM) algorithm, which computes a local optimum of problem~\eqref{eq:dglob}. To facilitate its analysis, it is desirable to have closed-form solutions at each step. However, in 2D there are no known closed-form solutions for any of the steps. Hence, we consider the 1D formulation and initialize the algorithm with the no-blur solution, which makes the first iteration equivalent to a total variation denoising problem (as opposed to a deconvolution problem). It is then possible to estimate a closed-form solution in the 1D case by using the \emph{taut string} algorithm of Davies and Kovacs~\cite{Davies2001}. Moreover, in Section~\ref{sec:tvd} we work with a step function and a $3$ pixel blur, and we show how it is possible to estimate a sharp, but scaled, signal from the first step of the AM algorithm (which we call \emph{TV-signal} in Fig.~\ref{fig:1dtoyt} (a)). In Section~\ref{sec:tvdf} we also show that it is possible to have a similar behavior by using a filtered version of the signals. However, in Section~\ref{sec:ex} we show how the use of the original signals may still be preferable.

By using the solution of the 1D total variation denoising problem, in Section~\ref{sec:aam} we show that a rigorous application of the AM algorithm gives the undesired no-blur solution in its second step. In fact, as shown in Fig.~\ref{fig:1dtoyt} (b), any blur different from the Dirac delta function would increase the difference between the input blurry signal (red plot) and the blurred TV signal (black plot). In Section~\ref{sec:pam} we point out that typical blind deconvolution algorithms do not solve the AM algorithm but a variant, which we call Projected Alternating Minimization (PAM) algorithm, where the constraints on the blur kernel are imposed separately in a delayed step. In Section~\ref{sec:apam} we show how this detail makes the PAM algorithm estimate in its second stage the true blur kernel. The delayed enforcement of the constraints can be seen as equivalent to a scaling of the TV denoised signal (see Fig.~\ref{fig:1dtoyt}  (c)).  Finally, in Section~\ref{sec:dis} we highlight the role of the regularization parameter $\lambda$ and further stress the importance of the scaling principle to make the PAM algorithm succeed. 

\subsection{Analysis of Relevant Local Minima}
\label{sec:local_minima}
In this section we study the $L_p$ norm of the image derivatives as image prior, with a particular emphasis to the case of $p \geq 1$. In a recent work Levin \etal~\cite{Levin2011} have shown that  eq.~\eqref{eq:glob} favors the no-blur solution $(f, \delta)$, when $J(u) = \int |u_x(\mathbf{x})|^p + |u_y(\mathbf{x})|^p d\mathbf{x}$, for any $p > 0$ and either the true blur $k_0$ has a large support or $||k_0||_2^2 \ll 1$. In the following theorem we show that the above result is also true for any kind of blur kernels and for an image prior with $p \geq 1$.
\begin{theorem}\label{the:youngs}
Let $J(u) = \|\nabla u \|_p \doteq \left (\int \|\nabla u(\mathbf{x})\|_p^p d\mathbf{x} \right )^{\frac{1}{p}}$, with $p\in[1,\infty]$, $f$ be the noise-free input blurry image ($n = 0$) and $u_0$ the sharp image. Then,
\begin{equation}
J(f) \leq J(u_0).
\end{equation}
\end{theorem}
\begin{IEEEproof}
Because $f$ is noise-free, $f = k_0 \ast u_0$; since the convolution and the gradient are linear operators, we have
\begin{equation}
J(f) =   \| \nabla  (k_0 \ast u_0)\|_p  = \| k_0 \ast \nabla u_0\|_p
\end{equation}
By applying Young's inequality \cite{Bogachev07} (see Theorem 3.9.4, pages 205-206) we have 
\begin{equation}
J(f) =\|k_0 \ast \nabla u_0\|_p \leq \|k_0\|_1   \|\nabla u_0\|_p =  \|\nabla u_0\|_p \doteq J(u_0)
\end{equation}
since $ \|k_0\|_1  = 1$.
\end{IEEEproof}
Since the first term (the data fitting term) in problem~\eqref{eq:glob} is zero for both the no-blur solution $(f, \delta)$ and the true solution $(u_0, k_0)$, Theorem~\ref{the:youngs} states that the no-blur solution has always a smaller, or at most equivalent, cost than the true solution. Notice that Theorem~\ref{the:youngs} is also valid for any $J(u) = \| \nabla u\|_p^r$ for any $r>0$. Thus, it includes as special cases the Gaussian prior $J(u) = ||u||_{H^1}$, when $p = 2$, $r=2$, and the anisotropic total variation prior $J(u) = \|u_x\|_1+\|u_y\|_1$, when $p = 1$, $r=1$.

Theorem~\ref{the:youngs} highlights a strong limitation of the formulation~\eqref{eq:glob}: The exact solution can not be retrieved when an iterative minimizer is initialized at the no-blur solution. 

\subsection{The Alternating Minimization (AM) Algorithm}
\label{sec:am}
To better understand the behavior of a total variation based blind deconvolution algorithm  we consider an \emph{alternating minimization} algorithm that minimizes~\eqref{eq:dglob}. The solution is found by alternating between the estimation of the sharp image given the kernel and the estimation of the kernel given the sharp image. This approach, which we call the AM algorithm, requires solving an unconstrained convex problem in $u$
\begin{equation}\label{eq:am_u}
u^{t+1} \leftarrow \arg\min_u  \sum_{x \in \mathcal{F}} ((k \circ u)[x] - f[x])^2 + \sum_{x \in \mathcal{U}}  \| (\nabla u) [x] \|_2
\end{equation}
and a constrained convex problem in $k$
\begin{equation}\label{eq:am_k}
\begin{array}{ll}
k^{t+1} \leftarrow & \displaystyle\arg\min_k  \displaystyle\sum_{x \in \mathcal{F}} ((k \circ u)[x] - f[x])^2\\
& \mbox{subject to }  k \succcurlyeq 0,\quad \|k\|_1 = 1.
\end{array}
\end{equation}

A convergence analysis of the AM algorithm is still challenging, but when the algorithm is initialized at the no-blur solution, the first step of the AM algorithm requires solving a denoising problem
\begin{equation}\label{eq:denos}
\hat{u} \leftarrow \arg\min_u  \sum_{x \in \mathcal{F}} (u[x] - f[x])^2 + \sum_{x \in \mathcal{U}}  \| (\nabla u) [x] \|_2.
\end{equation}
The total variation denoising algorithm~\eqref{eq:denos} has been widely studied in the literature, and its analysis can give key insights on the behavior of the AM algorithm. In the next section we study this problem and present an important building block for the other results presented in the paper.
\begin{table*}[t!]
\caption{Formulas of $\hat{U}_1$ and $\hat{U}_2$ for  $\lambda \in [\lambda_{min}^l,\lambda_{max}^l)$, $\lambda \in [\lambda_{min}^c,\lambda_{max}^c)$ and $\lambda \in [\lambda_{min}^r,\lambda_{max}^r)$ used in Proposition~\ref{proposition}.\label{table:proposition}} 
\centering
\bgroup
\def\arraystretch{1.7}
\begin{tabular}{|ll*{2}{|c}| }

  \hline 
 & & $\mathbf{\hat{U}_1(\lambda)}$ & $\mathbf{\hat{U}_2(\lambda)}$\\
  \hline 
  \multirow{2}{*} {$\lambda \in [\lambda_{min}^l,\lambda_{max}^l)$}& $\lambda_{min}^l =  (U_2-U_1) (L_2-L_2\delta_1-\delta_2)$ & \multirow{2}{*} {$ U_1+\frac{\lambda}{L_1-2}$} & \multirow{2}{*} {$ \frac{U_1 + U_2L_2}{L_2+1}+\frac{(\delta_1-\delta_2)(U_2-U_1)-\lambda}{L_2+1} $}\\
& $ \lambda_{max}^l = (U_2-U_1)\frac{L_1-2}{L_1+L_2-1}(L_2+\delta_1-\delta_2)$ & &\\
  \hline  
   \multirow{2}{*} {$\lambda \in [\lambda_{min}^c,\lambda_{max}^c)$}&$\lambda_{min}^c =  (U_2-U_1)\max\left\{(L_1-2)\delta_1,(L_2-1)\delta_2\right\}$&\multirow{2}{*} {$U_1+\frac{\delta_1(U_2-U_1)+\lambda}{L_1-1}$} & \multirow{2}{*} {$U_2+\frac{-\delta_2(U_2-U_1)-\lambda}{L_2} $} \\
& $ \lambda_{max}^c = (U_2-U_1)\frac{L_2(L_1-1)-L_2\delta_1-(L_1-1)\delta_2}{L_1+L_2-1}$ & &\\
  \hline 
  \multirow{2}{*} {$\lambda \in [\lambda_{min}^r,\lambda_{max}^r)$}& $\lambda_{min}^r =  (U_2-U_1) (L_1-\delta_1-(L_1-1)\delta_2-1)$&\multirow{2}{*} {$\frac{U_1(L_1-1)+U_2+(\delta_1-\delta_2)(U_2-U_1)+\lambda}{L_1}$} & \multirow{2}{*} {$U_2-\frac{\lambda}{L_2-1} $} \\
& $ \lambda_{max}^r = (U_2-U_1)\frac{L_2-1}{L_1+L_2-1}(L_1-\delta_1+\delta_2-1)$ & &\\
  \hline 
 \end{tabular}
 \egroup
\end{table*}

\subsection{Analysis of 1D Total Variation Denoising}
\label{sec:tvd}
In this section we look at the solution of a 1D total variation denoising problem because analysis of problem~\eqref{eq:dglob} in the literature is still fairly limited and a closed-form solution even for a restricted family of 2D signals is not available. Still, analysis in 1D can provide practical insights.

A total variation denoising problem can be written in 1D as
\begin{equation} \label{eq:rof0}
\hat{u}[x] = \arg\min_u \frac{1}{2}\!\!\! \sum_{x = -L_1+1}^{L_2-1} \!\!\!\!\!\!(u[x] - f[x])^2 + \lambda \!\!\!\!\sum_{x = -L_1}^{L_2-1} \!\!\! |u[x+1] - u[x]|,
\end{equation}
where $u \in [-L_1,L_2]$ and $f \in [-L_1+1,L_2-1]$. For a successful convergence of the AM algorithm it is desirable to have a solution equal or close to the true sharp signal $u^0$. By exploiting recent work of Condat~\cite{Condat2013}, Strong and Chan~\cite{Strong2003} and the \emph{taut string algorithm} of Davies and Kovacs~\cite{Davies2001}, for a simple class of signals, in the following proposition we give the analytical expression for $\lambda$ that gives a sharp, but scaled, version of $u^0$ as the solution of the denoising problem~\eqref{eq:rof0}.

\begin{proposition}\label{proposition}
Let $u_0$ be a 1D step function of the following form
\begin{equation}
\label{eq:defu}
u_0[x] = \left \{ 
\begin{array}{rl}
U_1 & x \in[-L_1,-1]\\
U_2 & x \in [0, L_2]
\end{array}
\right.
\end{equation}
for some $U_1<U_2$ and $L_1,L_2>2$, and $k_0$ be a 3-element blur kernel where $\delta_1\doteq k_0[1]$, $\delta_2\doteq k_0[-1]$ and $k_0[0] = 1-\delta_1-\delta_2$, $\delta_1 + \delta_2 \le 1$ and $\delta_1,\delta_2 \ge 0$. Then, $f$ is the convolution of $u$ with $k$
\begin{equation}
\label{eq:deff}
f[x] = \left \{ 
\begin{array}{ll}
U_1 & x \in[-L_1+1,-2]\\
U_1+\delta_1(U_2-U_1) & x =-1\\
U_2-\delta_2(U_2-U_1) & x =0\\
U_2 & x \in [1, L_2-1].
\end{array}
\right .
\end{equation}
The solution $\hat{u}[x]$ to the problem~\eqref{eq:rof0}  is 
\begin{equation}\label{eq:u1}
\hat{u}[x] = \left \{ 
\begin{array}{ll}
\hat{U}_1(\lambda) & x \in [-L_1, -2]\\
\hat{U}_2(\lambda) & x \in [-1, L_2]
\end{array}
\right .
\end{equation}
when $\lambda \in [\lambda_{min}^l,\lambda_{max}^l)$, is
\begin{equation}\label{eq:u2}
\hat{u}[x] = \left \{ 
\begin{array}{ll}
\hat{U}_1(\lambda)  & x \in [-L_1, -1]\\
\hat{U}_2(\lambda)  & x \in [0, L_2]
\end{array}
\right .
\end{equation}
when $\lambda \in [\lambda_{min}^c,\lambda_{max}^c)$, and is
\begin{equation}\label{eq:u3}
\hat{u}[x] = \left \{ 
\begin{array}{ll}
\hat{U}_1(\lambda) & x \in [-L_1, 0]\\
\hat{U}_2(\lambda)  & x \in [1, L_2]
\end{array}
\right .
\end{equation}
when $\lambda \in [\lambda_{min}^r,\lambda_{max}^r)$. 
Analytical expressions for $\hat{U}_1(\lambda)$,  $\hat{U}_2(\lambda)$, $\lambda_{min}^l$, $\lambda_{max}^l$, $\lambda_{min}^c$, $\lambda_{max}^c$, $\lambda_{min}^r$, $\lambda_{max}^r$,  are defined in Table~\ref{table:proposition}. 
If $\delta_2 = \frac{L_1  - \delta_1 - 1}{L_1 + L_2 - 2}$ or $\delta_2 = L_2 - (L_1 + L_2 -2) \delta_1$ then a $\lambda$ that gives the solution~\eqref{eq:u1},~\eqref{eq:u2} or ~\eqref{eq:u3} does not exist.

\end{proposition}
\begin{IEEEproof}
See Section \ref{proof:proposition}.
\end{IEEEproof}
Proposition~\ref{proposition} shows that for a wide range of signals it is possible to obtain a sharp signal by solving problem~\eqref{eq:rof0}. Notice, however, that total variation regularization locally scales the input signal (see also illustration in Fig~\ref{fig:1dtoyt} (a)). We will show in the next sections that this apparently insignificant scaling has a fundamental role in the convergence of the AM algorithm.

\subsection{Filtered Image Model}
\label{sec:tvdf}
A common practice in many MAP$_k$ methods that use a Variational Bayesian approach is to solve problem~\eqref{eq:dglob} using the gradients of $u$ and $f$ instead of the original signals~\cite{Fergus2006,Wipf2013,Levin2011}. In this section we briefly consider this case for the AM algorithm. In this case the 1D denoising problem of~\eqref{eq:rof0} becomes 

\begin{align} \label{eq:rof0d}
\hat{u}_x[x] = \arg\min_{u_x} \frac{1}{2}\!\!\!\!\sum_{x = -L_1+1}^{L_2-2}\!\!\!\! (u_x[x] - f_x[x])^2
 + \lambda \sum_{x = -L_1}^{L_2-1} |u_x[x]|.
\end{align}
where $u_x[x] = u[x+1] - u[x]$ and $f_x[x]  = f[x+1] - f[x]$.
Problem~\eqref{eq:rof0d} can be easily solved in closed form by soft-thresholding, therefore for a simple class of signals as in Proposition~\ref{proposition} we can seek for values of $\lambda$ such that the solution $\hat{u}_x$ is the derivative of a sharp signal, \ie, a Dirac delta.

\begin{theorem}\label{the:grad}
The solution of problem~\eqref{eq:rof0d} with $f_x[x] = f[x+1] - f[x]$, where $f[x]$ is defined in~\eqref{eq:deff}, is
\begin{equation}
\label{eq:ux1}
u_x[x] = \left \{ 
\begin{array}{ll}
0 & x \neq -2\\
(\delta_1 - \max(\delta_2, 1 - \delta_1 - \delta_2)) (U_2 - U_1) & x = -2
\end{array}
\right .
\end{equation}
if $\delta_1 > \max(\delta_2, \frac{1 - \delta_2}{2})$ and $\lambda \ge (U_2 - U_1)\max(\delta_2, 1 - \delta_1 - \delta_2)$, is
\begin{equation}
\label{eq:ux2}
u_x[x] = \left \{ 
\begin{array}{ll}
0 & x \neq -1\\
(1 - \delta_1 - \delta_2 - \max(\delta_1,\delta_2))(U_2 - U_1) & x = -1
\end{array}
\right .
\end{equation}
if $1 > \max(2 \delta_1 +\delta_2,2 \delta_2 +\delta_1)$ and $\lambda \ge (U_2 - U_1)\max(\delta_1, \delta_2)$, and is
\begin{equation}
\label{eq:ux3}
u_x[x] = \left \{ 
\begin{array}{ll}
0 & x \neq 0\\
(\delta_2 - \max(\delta_1, 1 - \delta_1 - \delta_2))(U_2 - U_1) & x = 0
\end{array}
\right .
\end{equation}
if $\delta_2 > \max(\delta_1, \frac{1 - \delta_1}{2})$ and $\lambda \ge (U_2 - U_1)\max(\delta_1, 1 - \delta_1 - \delta_2)$.

If $\delta_2 = \delta_1 \ge 1/3$, if $\delta_1 = (1 - \delta_2)/2 \ge 1/3$ or if $\delta_2 = (1 - \delta_1)/2 \ge 1/3$ then solving problem~\eqref{eq:rof0d} can never lead to the solutions in~\eqref{eq:ux1},~\eqref{eq:ux2} and~\eqref{eq:ux3} for any possible value of $\lambda$.
\end{theorem}
\begin{IEEEproof}
See Section \ref{proof:the:grad}.
\end{IEEEproof}
Theorem~\ref{the:grad} gives conditions on $\lambda$ such that the TV denoising of filtered signals gives a scaled version of  the true sharp signal. This result is similar to the one given by Proposition~\ref{proposition} for the classical TV denoising algorithm, nonetheless there are still important differences that we highlight in the next section.

\subsection{Existence of Sharp Signal Solutions}
\label{sec:ex}
Both Theorem~\ref{the:grad} and Proposition~\ref{proposition} give conditions such that it is possible to estimate a sharp signal from a blurry one by solving a denoising problem. An important question is to ask under what conditions a sharp signal can not be estimated from a blurry one. We will consider the space of all possible 3-element blurs parametrized in $\delta_1$ and $\delta_2$ and call \emph{degenerate region} the configurations for which a sharp signal can not be estimated. Notice that the degenerate region corresponds to signals where it is only possible to remove one of the two steps of the blurry signal, or where the minimum $\lambda$ required to get a sharp signal is equal to the one that makes the whole signal constant. 

For the classical TV denoising algorithm from Theorem~\ref{the:grad} we have that if the true blur lies on two lines, parametrized  on $L_1$ and $L_2$, then it is not possible to estimate a sharp signal (see Fig.~\ref{fig:nod} (a) and (b)). Because of the parametrization on $L_1$ and $L_2$, different signals have different degenerate regions. We can generalize the discussion to a general piecewise constant signal by dividing it in many step signals. Thus, the denoising problem of a piecewise signal is  equivalent to many denoising problems of different step signals with careful treatment at the boundaries. If each signal has different degenerate regions, it is likely that for each blur configuration the majority of the signal will not be in a degenerate case. 

For the TV denoising of the filtered signal, the degenerate region is denoted by three segments (see Fig.~\ref{fig:nod} (c)). Since it is the same for any possible signal, it would not be changed by considering a general piecewise signal as in the case of problem~\eqref{eq:dglob}. Thus, in this case it is not possible  to deblur  images blurred with blurs that lie in the degenerate case.

The above discussion is intentionally informal, as providing a formal proof would require too much space. Still, it suggests that there may be aspects, previously not considered in the literature, that support the use of the original signal in problem~\eqref{eq:dglob} in lieu of their filtered versions.

\begin{figure}[t]
\centering
\begin{minipage}[c]{.15\textwidth}
\centering
\includegraphics[width=\textwidth]{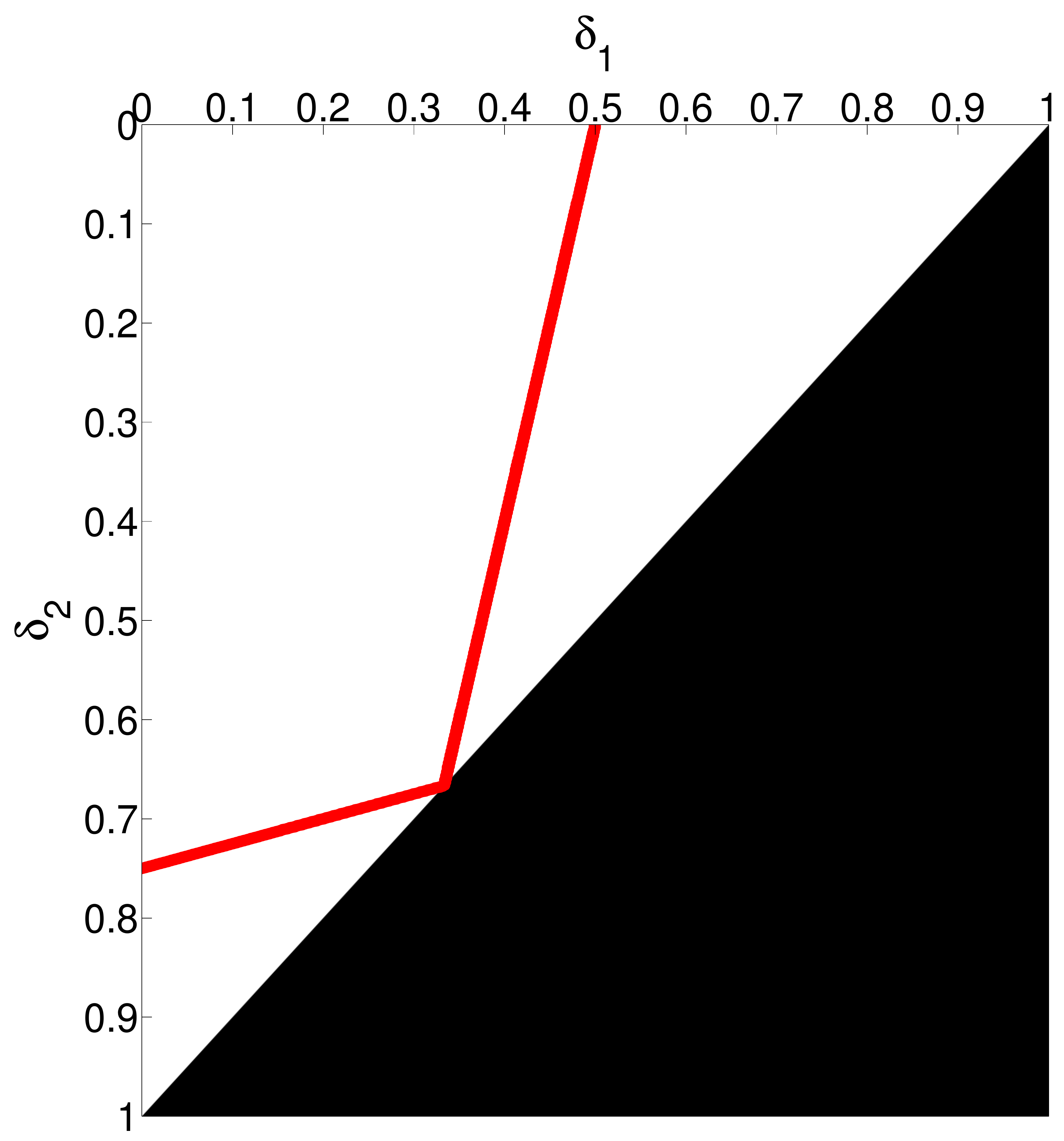}
\footnotesize{a)} 
\end{minipage}
\begin{minipage}[c]{.15\textwidth}
\centering
\includegraphics[width=\textwidth]{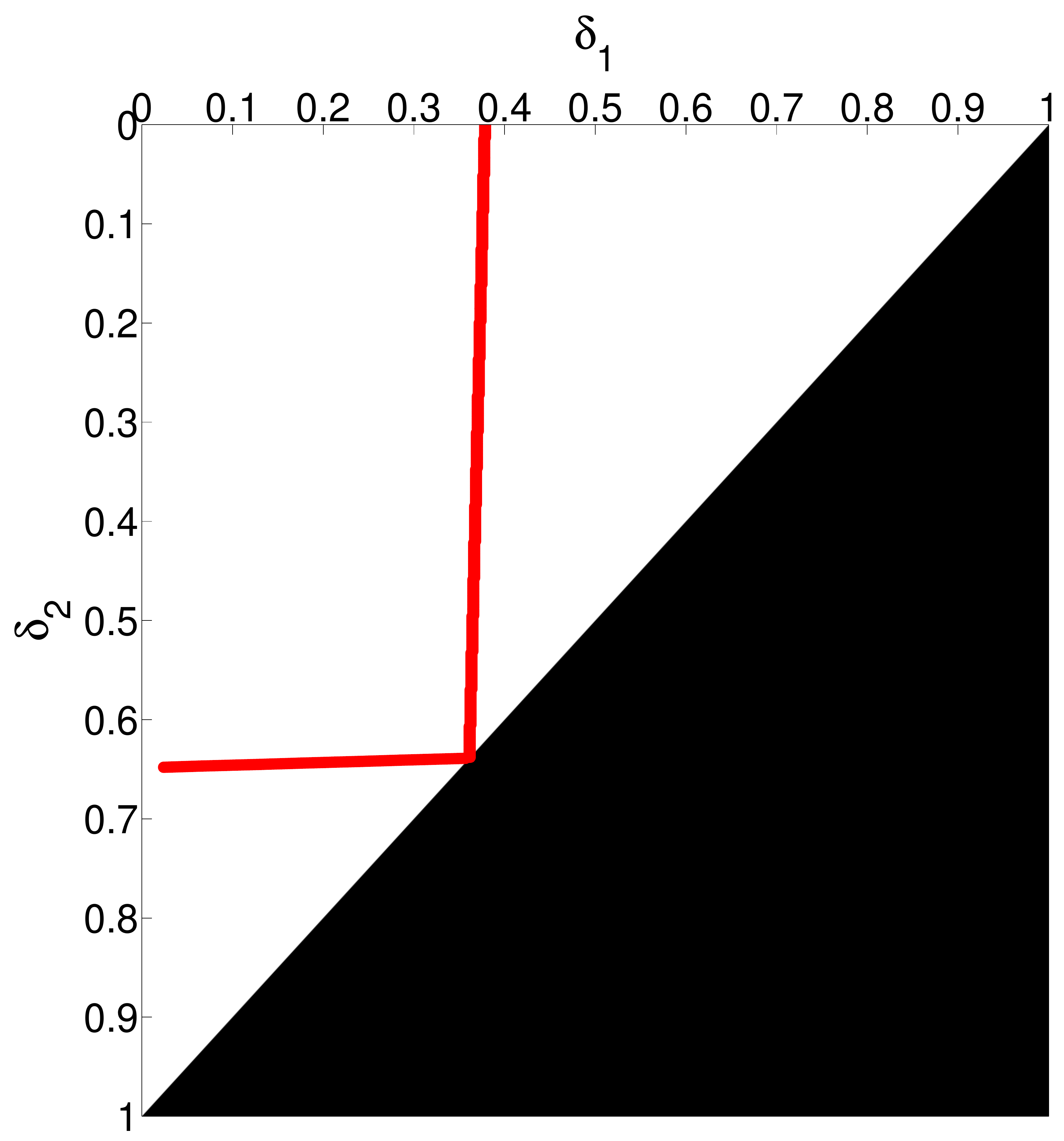}
\footnotesize{b)}
\end{minipage}
\begin{minipage}[c]{.15\textwidth}
\centering
\includegraphics[width=\textwidth]{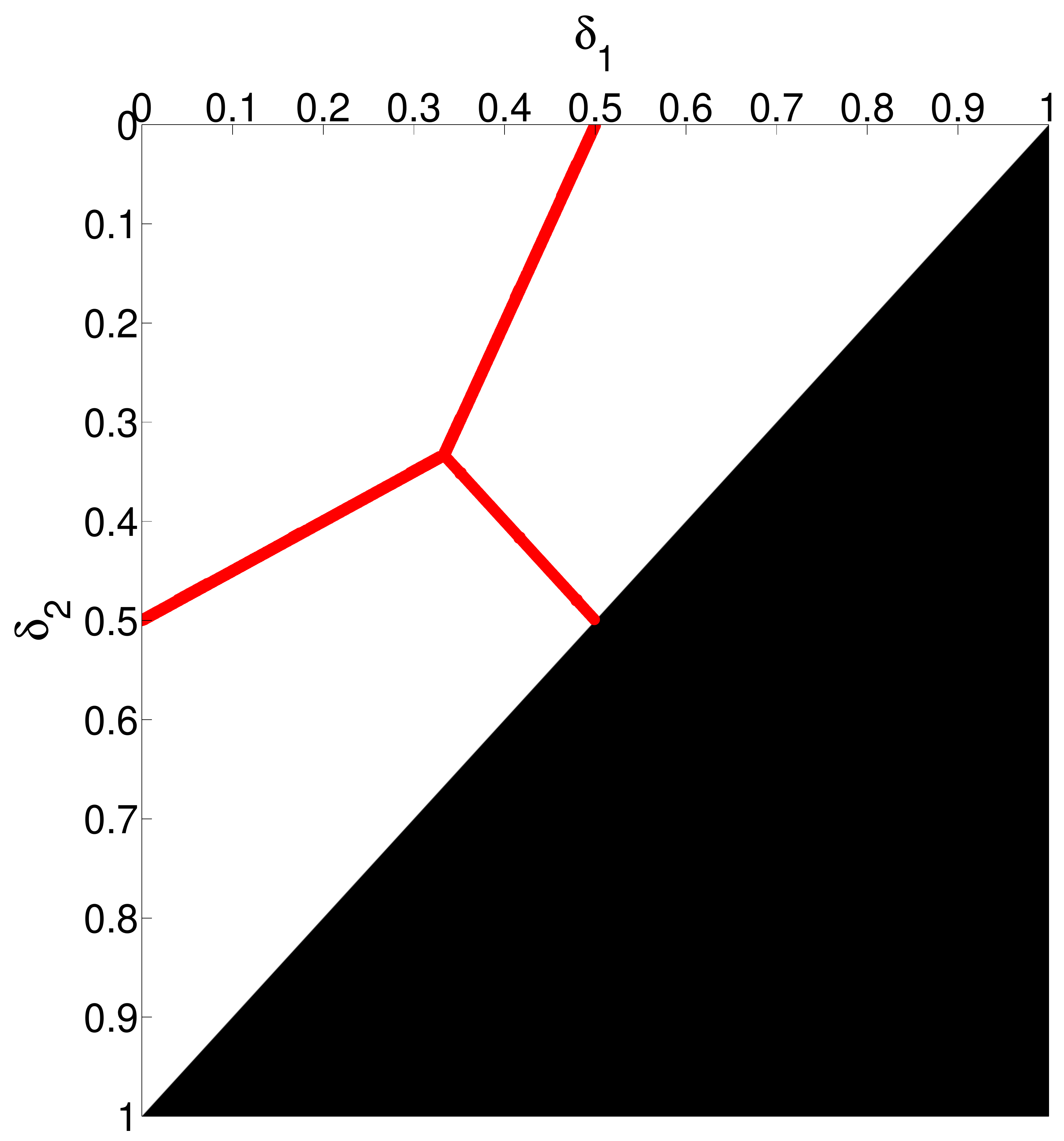}
\footnotesize{c)}
\end{minipage}
\medskip
\caption{Illustration of configurations of $\delta_1$ and $\delta_2$ for which it is not possible to estimate a sharp signal in problem~\eqref{eq:rof0} (cases a) and b)~) and in problem~\eqref{eq:rof0d} (case c)~). The region of feasible solutions is where $\delta_1 + \delta_2 \le 1$ is satisfied and is denoted by the white region; the red lines denote the configurations that can not lead to a sharp signal for: a) a signal as in Proposition~\ref{proposition} such that $L_1 = 3$ and $L_2 = 3$; b)  a signal as in Proposition~\ref{proposition} such that $L_1 = 15$ and $L_2 = 24$; c) a signal as in Theorem~\ref{the:grad}.\label{fig:nod} }
\end{figure}

\subsection{Analysis of the AM Algorithm}
\label{sec:aam}
The study of the total variation denoising algorithm is important because it represents the first step of the AM algorithm, when the blur is initialized with a Dirac delta.
In the next theorem we show how the exact AM algorithm can not leave the no-blur solution $(f, \delta)$ when $f$ is defined as in~\eqref{eq:deff} and the first step of the AM algorithm gives as a solution a sharp signal.

\begin{theorem}\label{the:am}
Let $f$, $u^0$ and $k^0$ be the same as in Proposition~\ref{proposition}. Then, for a $\lambda \in [\lambda_{min}^l,\lambda_{max}^l)$, $\lambda \in [\lambda_{min}^c,\lambda_{max}^c)$ or $\lambda \in [\lambda_{min}^r,\lambda_{max}^r)$  the AM algorithm converges to the solution $ k = \delta$.
\end{theorem}
\begin{proof}
See Section \ref{proof:the:am}.
\end{proof}

The result in Theorem~\ref{the:am} confirms what it has been concluded by Theorem~\ref{the:youngs}.  In the next section we show a variant of the AM algorithm, which is commonly used in many other blind deconvolution papers, that behave in a completely different way than the AM algorithm for the same class of signals of Proposition~\ref{proposition}. 

\subsection{The Projected Alternating Minimization (PAM) Algorithm}
\label{sec:pam}
Many methods in the literature minimize problem~\eqref{eq:dglob}  by using a variant of the AM algorithm. This variant consists in alternating between minimizing the unconstrained convex problem~\eqref{eq:am_u} in $u$ as in the AM algorithm, 
and an unconstrained convex problem in $k$
\begin{equation}\label{eq:pam_k}
\begin{array}{ll}
k^{t+1/3} \leftarrow &\displaystyle \arg\min_k  \displaystyle\sum_{x \in \mathcal{F}} ((k \circ u)[x] - f[x])^2,
\end{array}
\end{equation}
followed by two sequential projections where one applies the constraints on the blur $k$, \ie,
\begin{align}
k^{t+2/3} \leftarrow \max\{k^{t+1/3},0\},\quad~~
k^{t+1} \leftarrow  \frac{k^{t+2/3}}{\| k^{t+2/3}\|_1}.
\end{align}
We call this iterative procedure the PAM algorithm. The choice of imposing the constraints sequentially rather than during the gradient descent on $k$ seems a rather innocent and acceptable approximation of the correct procedure (AM). However, this is not the case and we will see that with this arrangement one can achieve the desired solution.

\begin{figure}[t]
\centering
\begin{minipage}[c]{.15\textwidth}
\centering
\includegraphics[width=\textwidth]{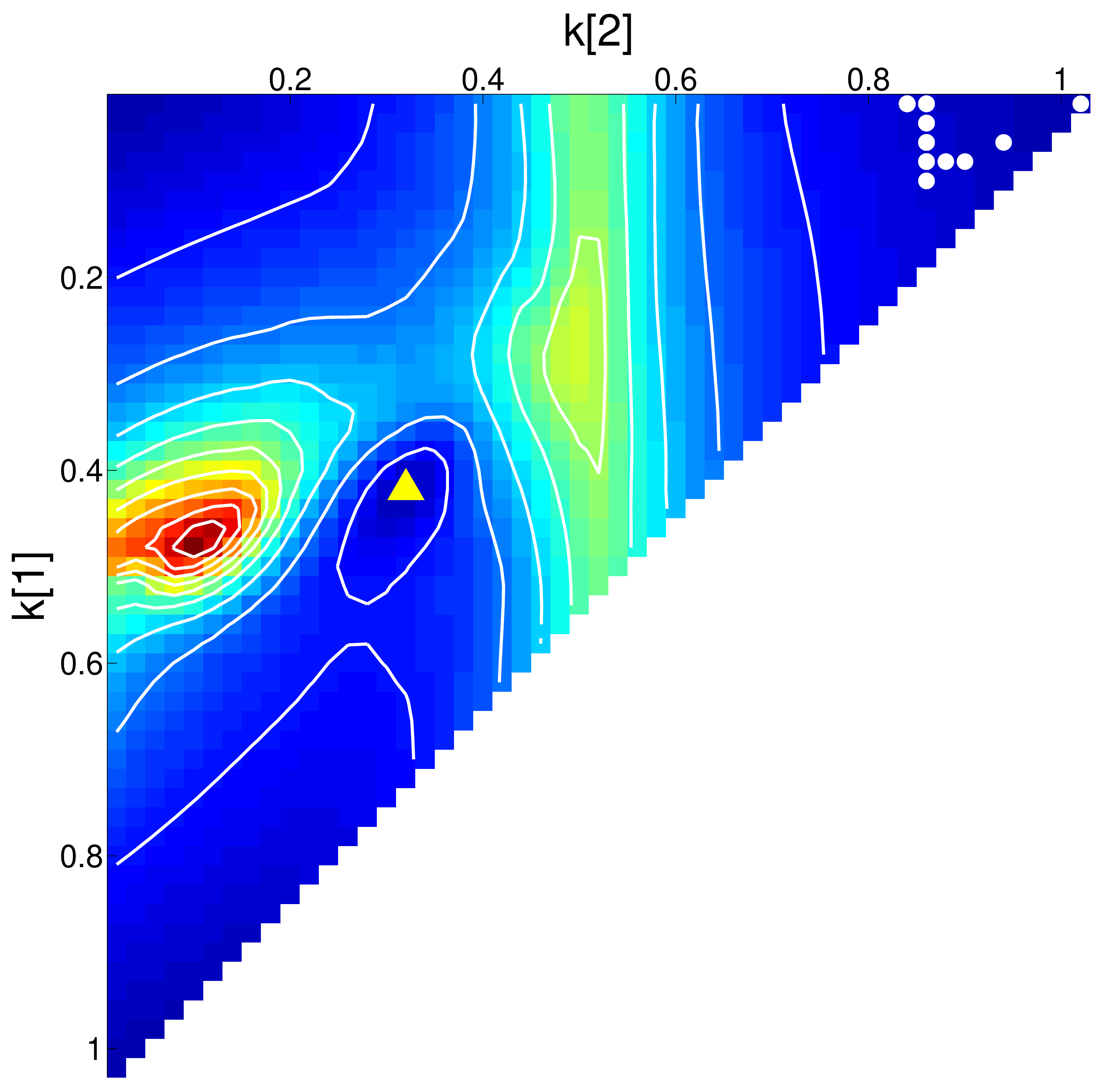}
\footnotesize{$\lambda = 0.0001$} 
\end{minipage}
\begin{minipage}[c]{.15\textwidth}
\centering
\includegraphics[width=\textwidth]{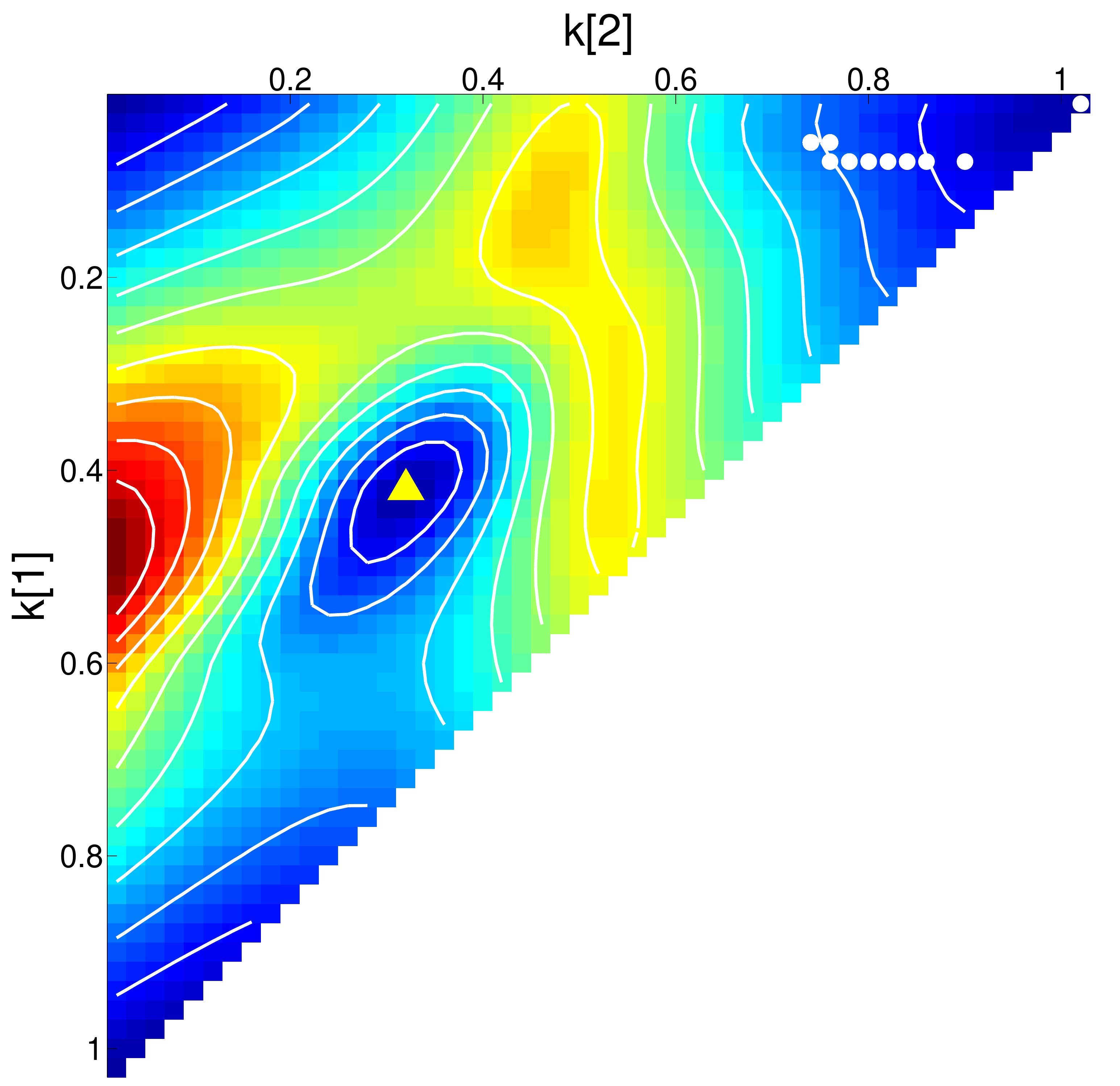}
\footnotesize{$\lambda = 0.001$}
\end{minipage}
\begin{minipage}[c]{.15\textwidth}
\centering
\includegraphics[width=\textwidth]{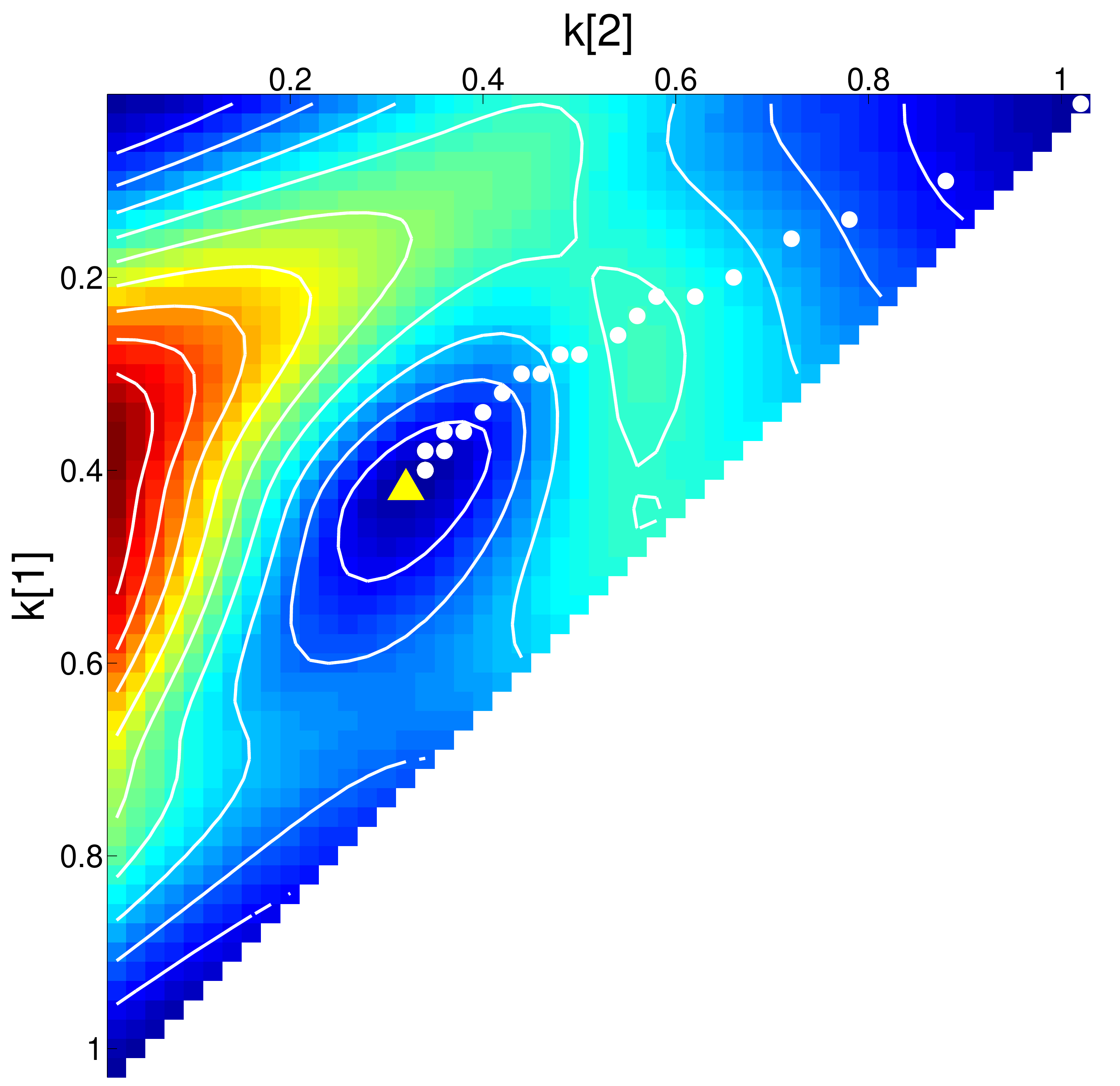}
\footnotesize{$\lambda = 0.01$}
\end{minipage}
\medskip
\caption{Illustration of Theorem~\ref{the:youngs} (best viewed in color). In this example we show a 1D experiment where we blur a step function with $k_0 = [0.4; 0.3; 0.3]$. We visualize the cost function of eq.~\eqref{eq:dglob} for three different values of the parameter $\lambda$. Since the blur integrates to $1$, only two of the three components are free to take values on a triangular domain (the upper-left triangle in each image). We denote with a yellow triangle the true blur $k_0$ and with white dots the intermediate blurs estimated during the minimization via the PAM algorithm. Blue pixels have lower values than the red pixels. Dirac delta blurs are located at the three corners of each triangle. At these locations, as well as at the true blur, there are local minima. Notice how the path of the estimated blur on the rightmost image ascends and then descends a hill in the cost functional.\label{fig:1dtoy} }
\end{figure}

\subsection{Analysis of the PAM Algorithm}
\label{sec:apam}
Our first claim is that this procedure does not minimize the original problem~\eqref{eq:glob}. To support this claim we start by showing some experimental evidence in Fig.~\ref{fig:1dtoy}. In this test we work on a 1D version of the problem. We blur a hat function with one blur of size $3$ pixels, and we show the minimum of eq.~\eqref{eq:dglob} for all possible feasible blurs. Since the blur has only $3$ nonnegative components and must add up to $1$, we only have $2$ free parameters bound between $0$ and $1$. Thus, we can produce a 2D plot of the minimum of the energy with respect to $u$ as a function of these two parameters.
The blue color denotes a small cost, while the red color denotes a large cost. The figures reveal three local minima at the corners, due to the $3$ different shifted versions of the no-blur solution, and the local minimum at the true solution ($k_0=[0.4,0.3,0.3]$) marked with a yellow triangle. We also show with white dots the path followed by $k$ estimated via the PAM algorithm by starting from one of the no-blur solutions (upper-right corner). Clearly one can see that the PAM algorithm does not follow a minimizing path in the space of solutions of problem~\eqref{eq:dglob}, and therefore does not minimize the same energy.

Furthermore, we show how the PAM algorithm succeeds in estimating the true blur in two steps whereas the AM algorithm estimates a Dirac delta.

\begin{theorem}\label{the:pam}
Let $f$, $u^0$ and $k^0$ be the same as in Proposition~\ref{proposition}. Let also constraint $u^0$ to be a zero-mean signal.

Then, if there exists a  $\lambda \in [\lambda_{min}^l,\lambda_{max}^l)$, $\lambda \in [\lambda_{min}^c,\lambda_{max}^c)$ or $\lambda \in [\lambda_{min}^r,\lambda_{max}^r)$  the PAM algorithm estimates the true blur $k = k_0$ (or a shifted version of it) in two steps, when starting from the no-blur solution pair $(f, \delta)$.
\end{theorem}
\begin{proof}
See Section~\ref{proof:the:pam}.
\end{proof}

\begin{figure}[t]
\centering
\begin{minipage}[c]{.145\textwidth}
\centering
\includegraphics[width=\textwidth]{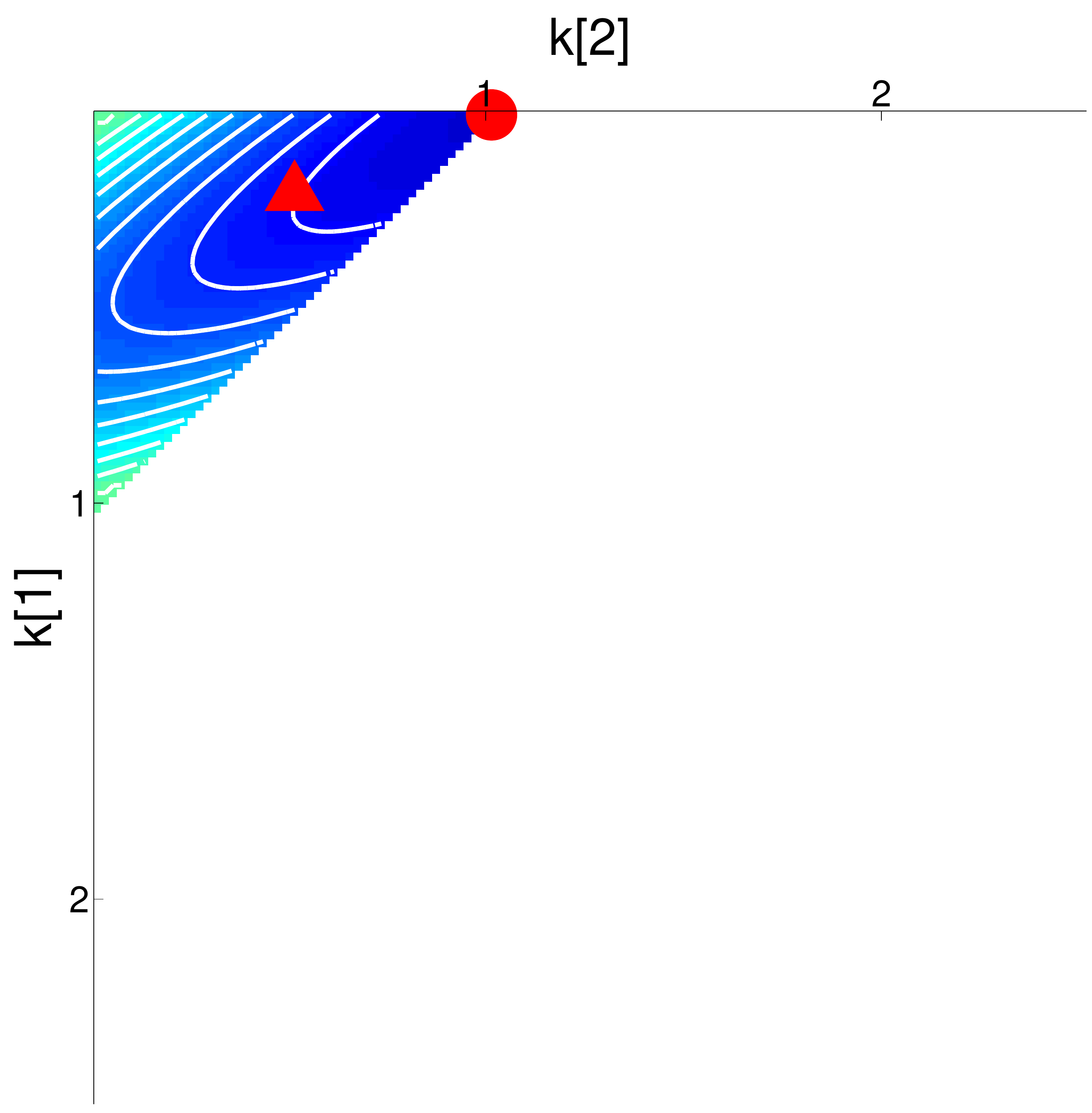} 
\end{minipage}
\begin{minipage}[c]{.145\textwidth}
\centering
\includegraphics[width=\textwidth]{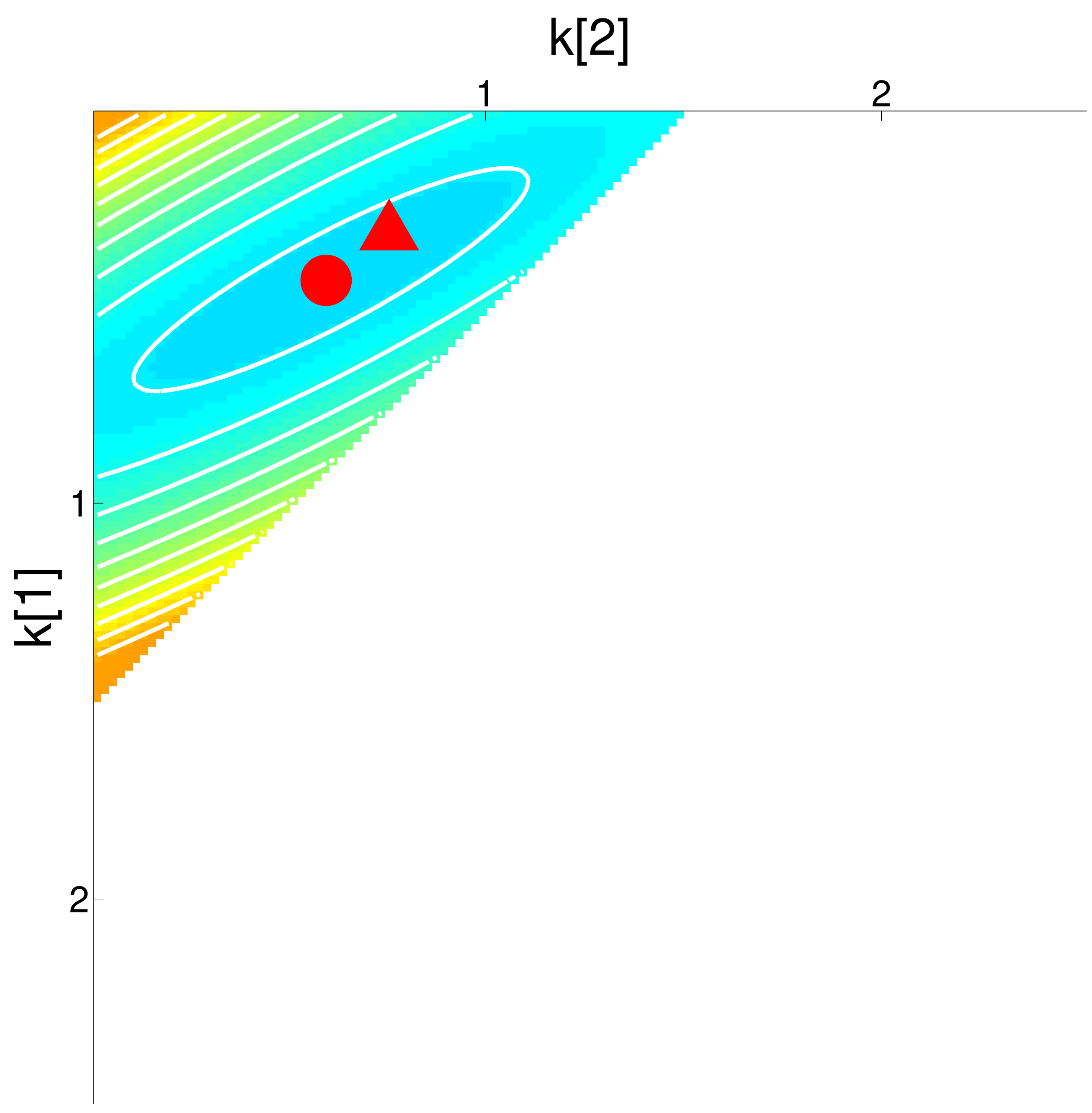}
\end{minipage}
\begin{minipage}[c]{.145\textwidth}
\centering
\includegraphics[width=\textwidth]{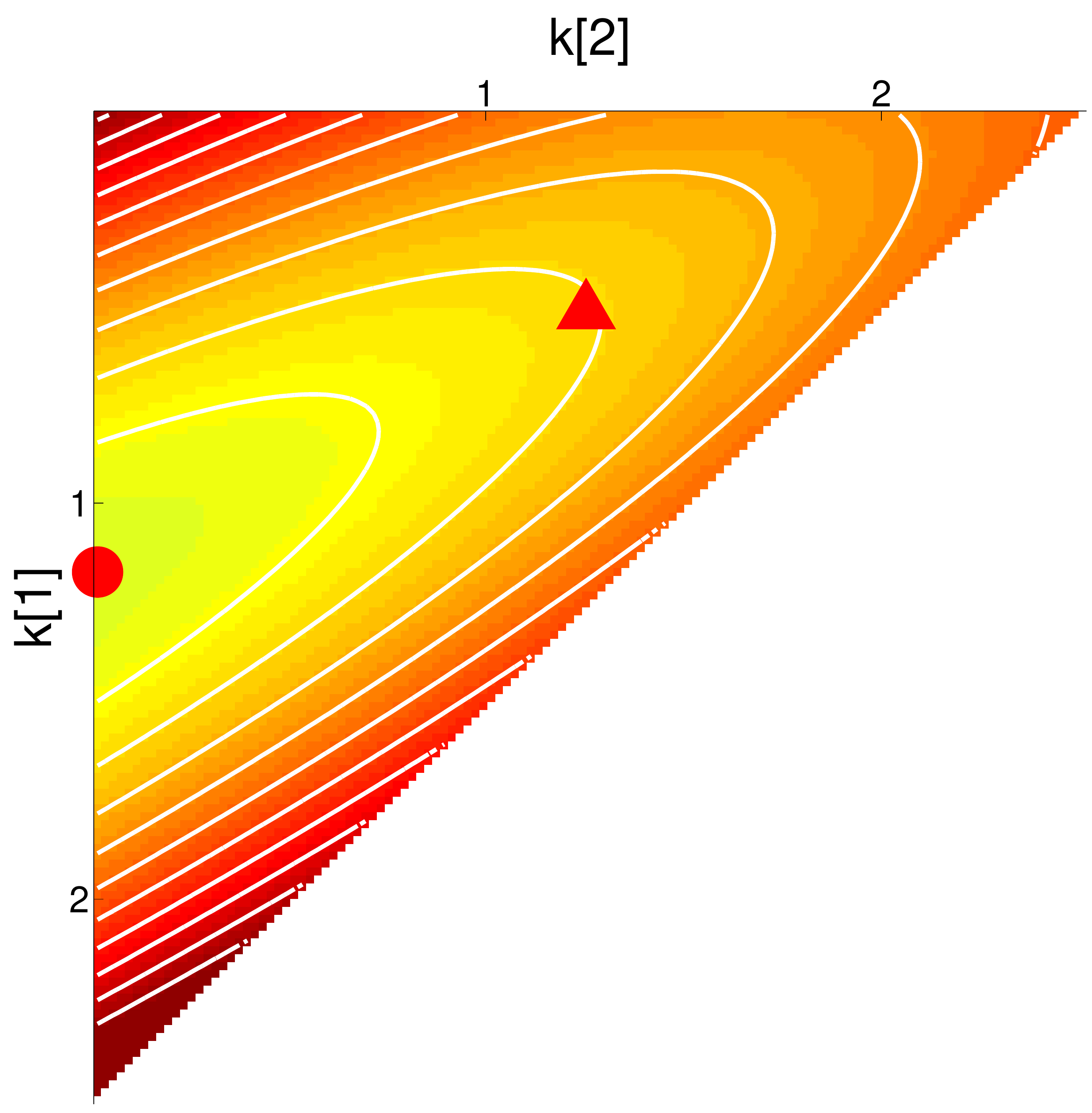}
\end{minipage}
\begin{minipage}[c]{.01\textwidth}
\centering
\footnotesize{\begin{turn}{90}$\lambda = 0.1$\end{turn}}
\end{minipage}
\centering
\begin{minipage}[c]{.145\textwidth}
\centering
\includegraphics[width=\textwidth]{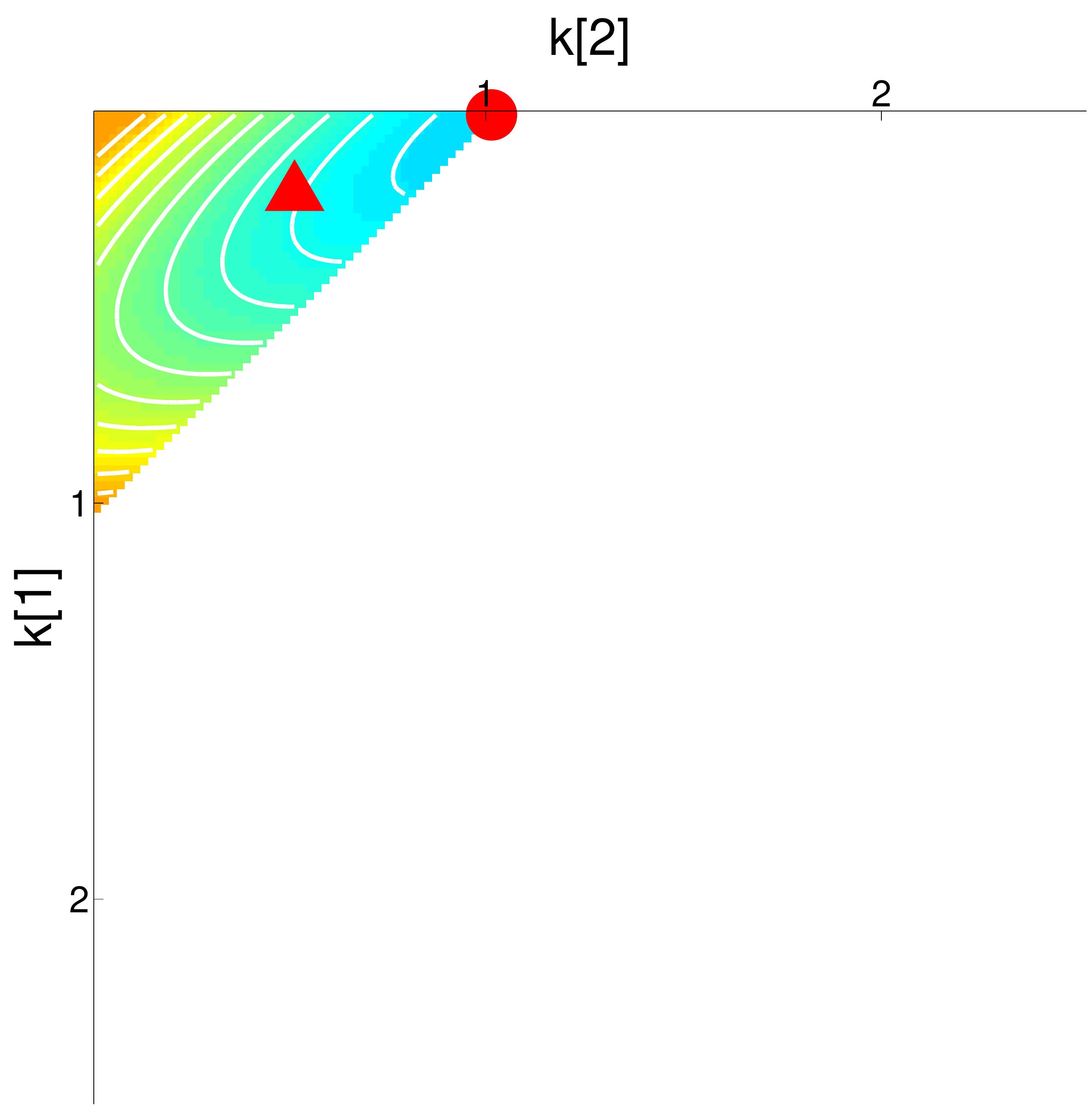} 
\end{minipage}
\begin{minipage}[c]{.145\textwidth}
\centering
\includegraphics[width=\textwidth]{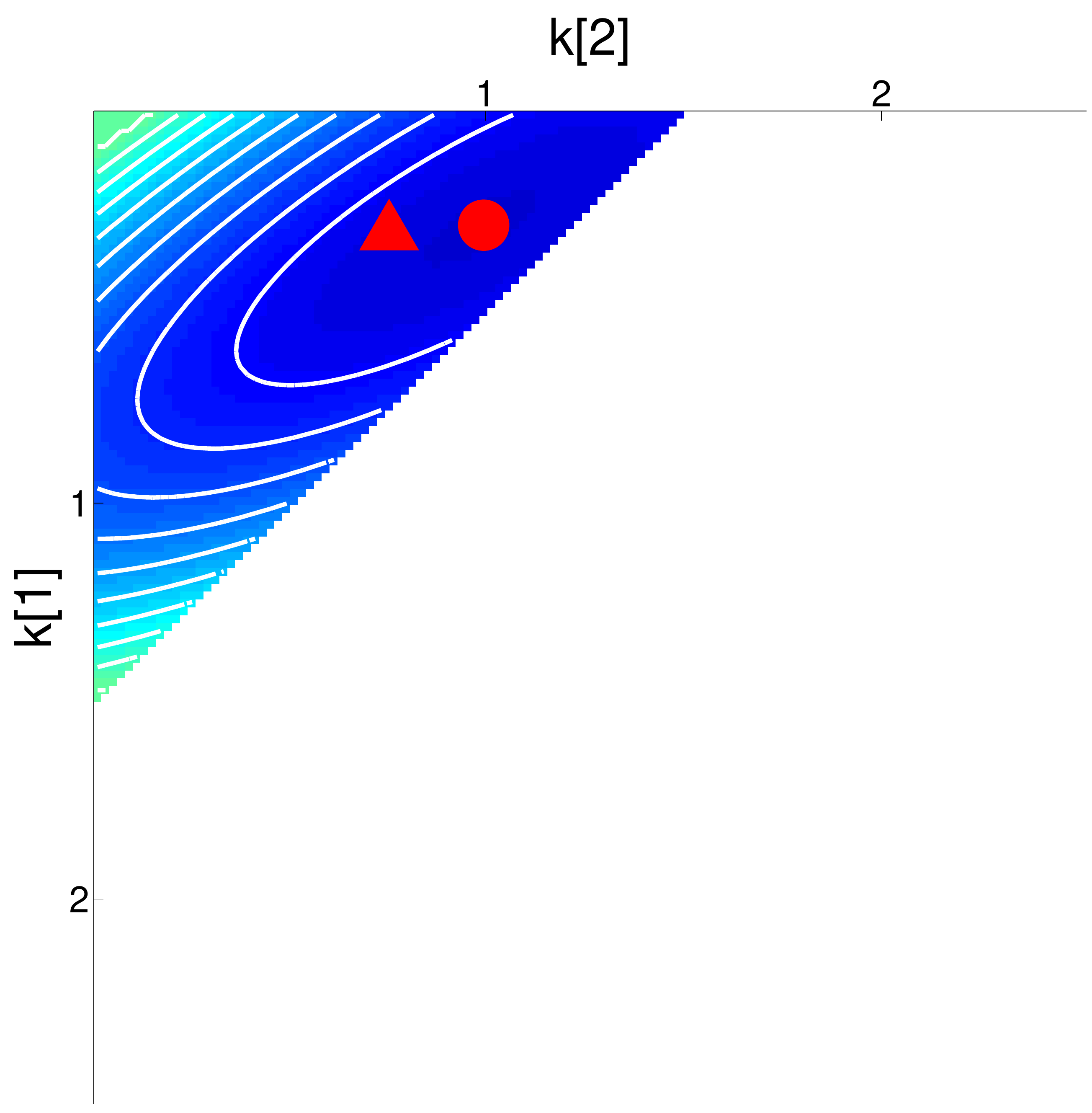}
\end{minipage}
\begin{minipage}[c]{.145\textwidth}
\centering
\includegraphics[width=\textwidth]{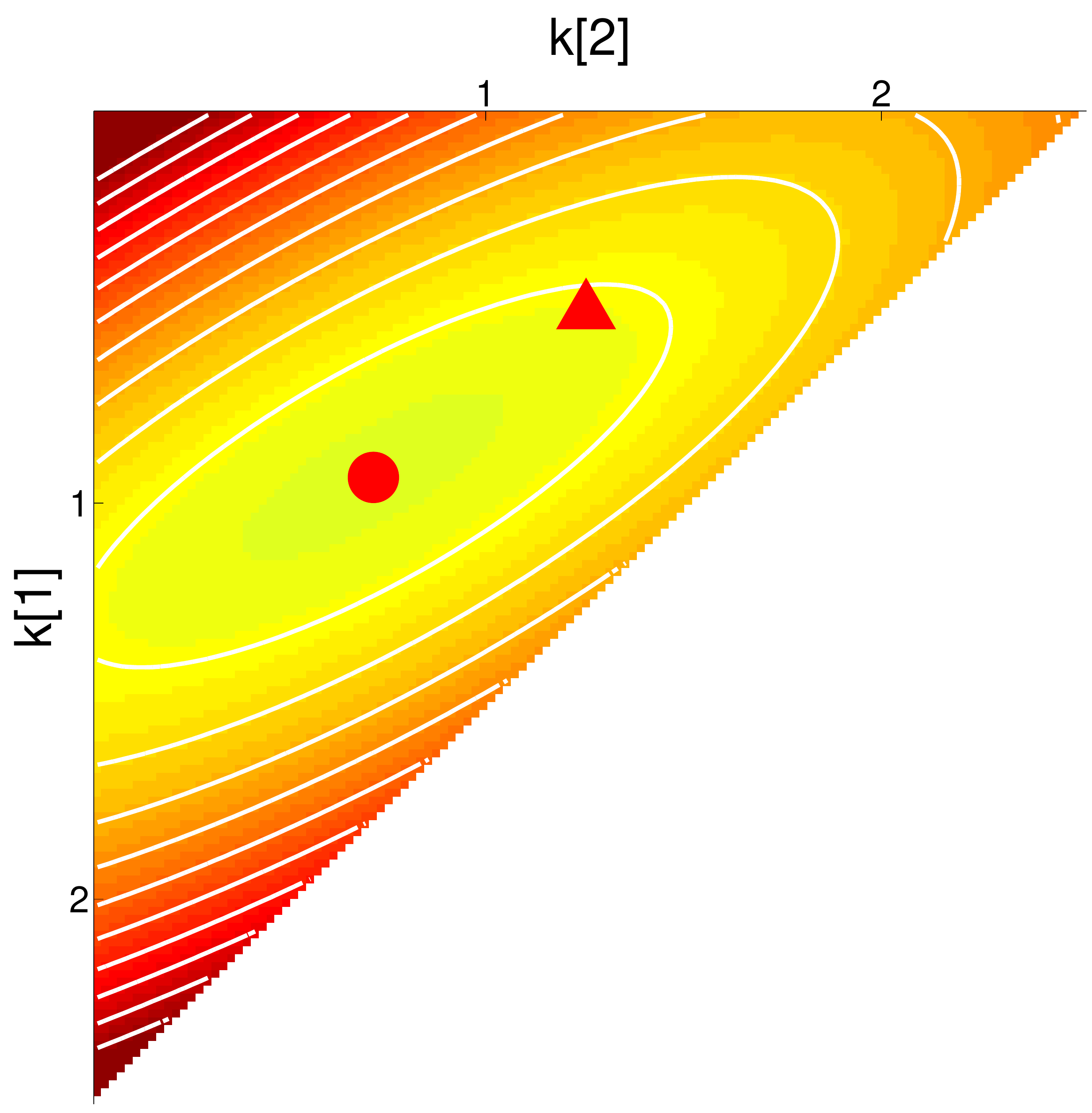}
\end{minipage}
\begin{minipage}[c]{.01\textwidth}
\centering
\footnotesize{\begin{turn}{90}$\lambda = 1.5$\end{turn}}
\end{minipage}
\begin{minipage}[c]{.145\textwidth}
\centering
\includegraphics[width=\textwidth]{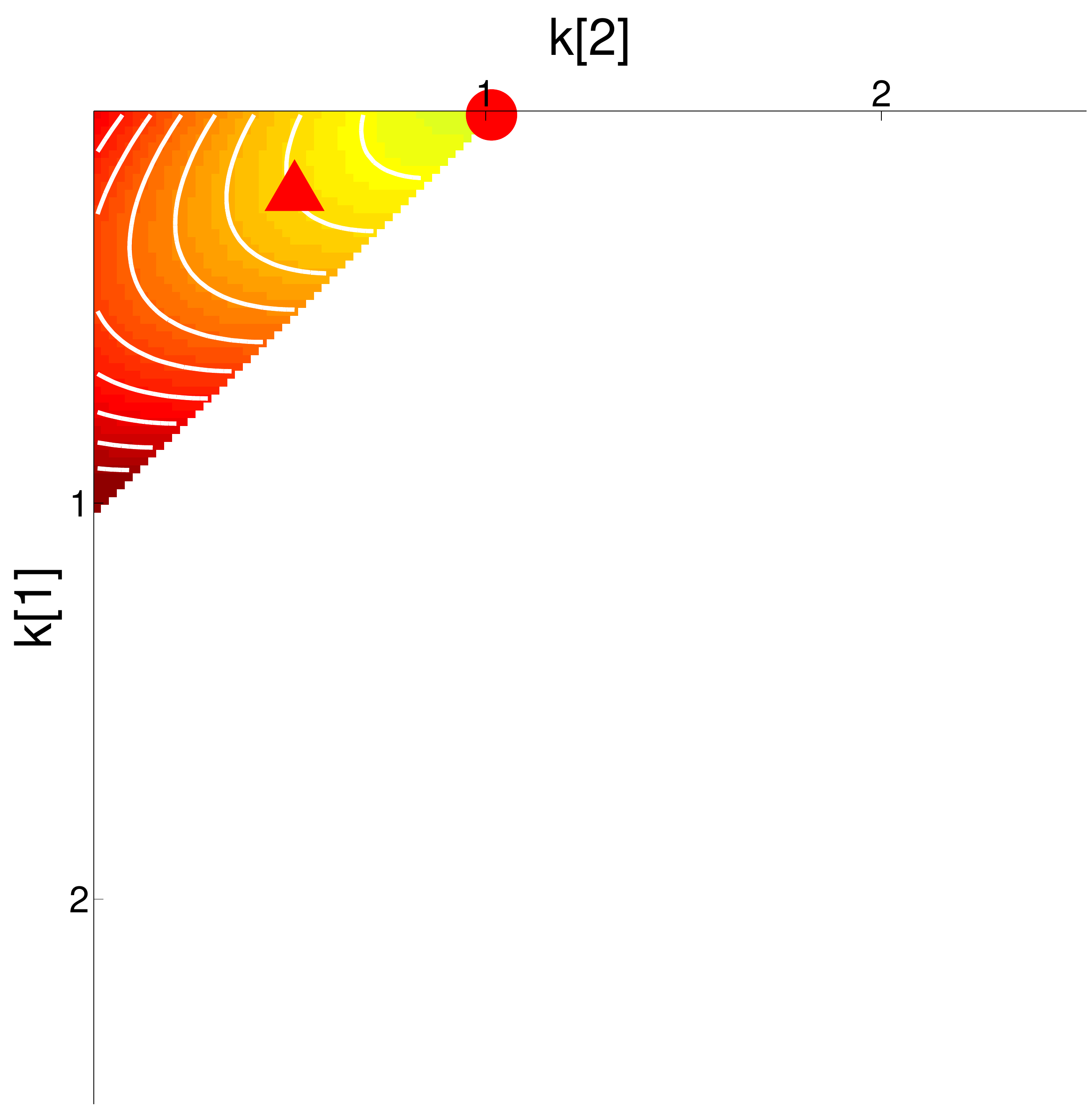} 
\footnotesize{$||k||_1 = 1$}
\end{minipage}
\begin{minipage}[c]{.145\textwidth}
\centering
\includegraphics[width=\textwidth]{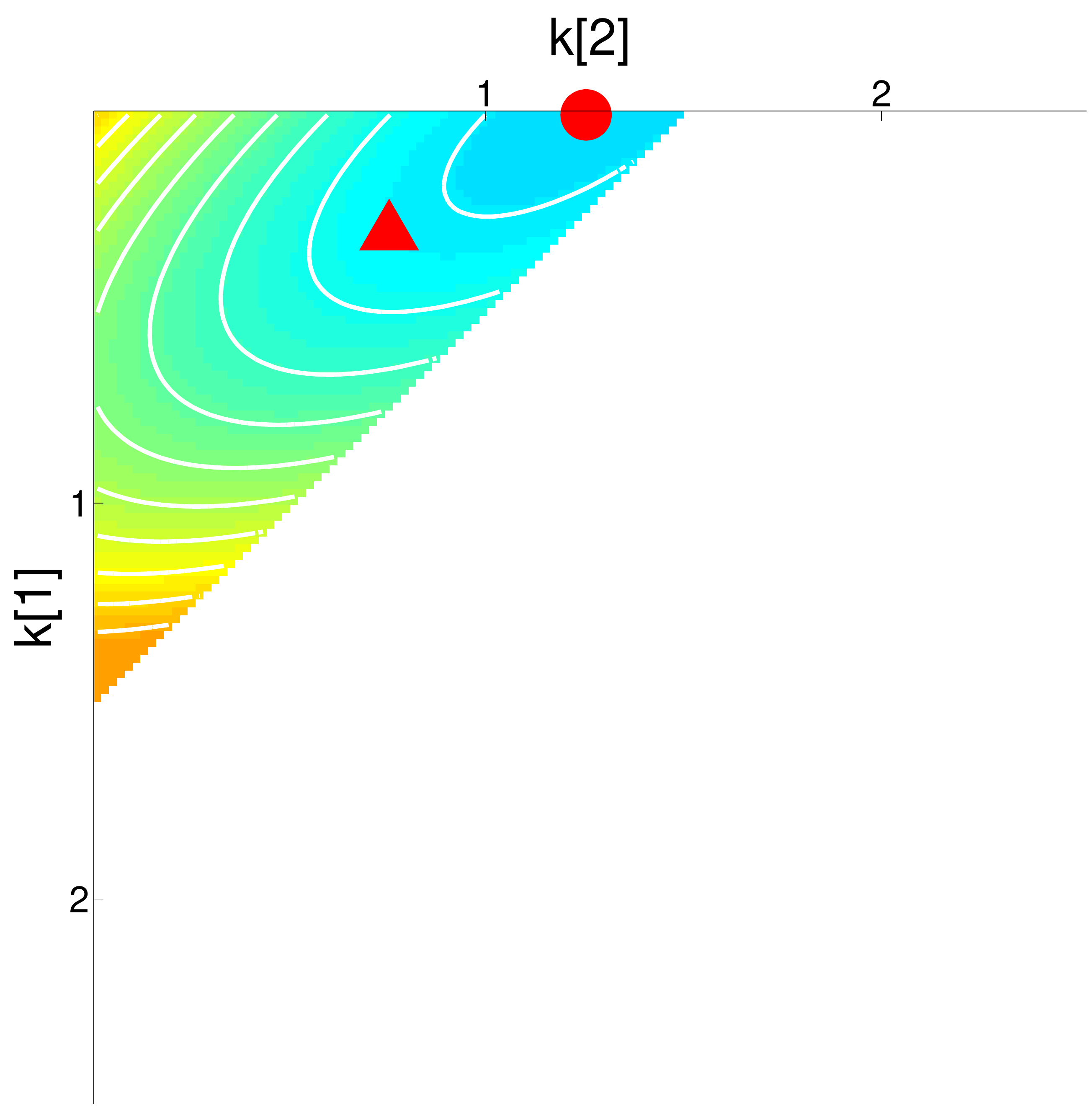}
\footnotesize{$||k||_1 = 1.5$}
\end{minipage}
\begin{minipage}[c]{.145\textwidth}
\centering
\includegraphics[width=\textwidth]{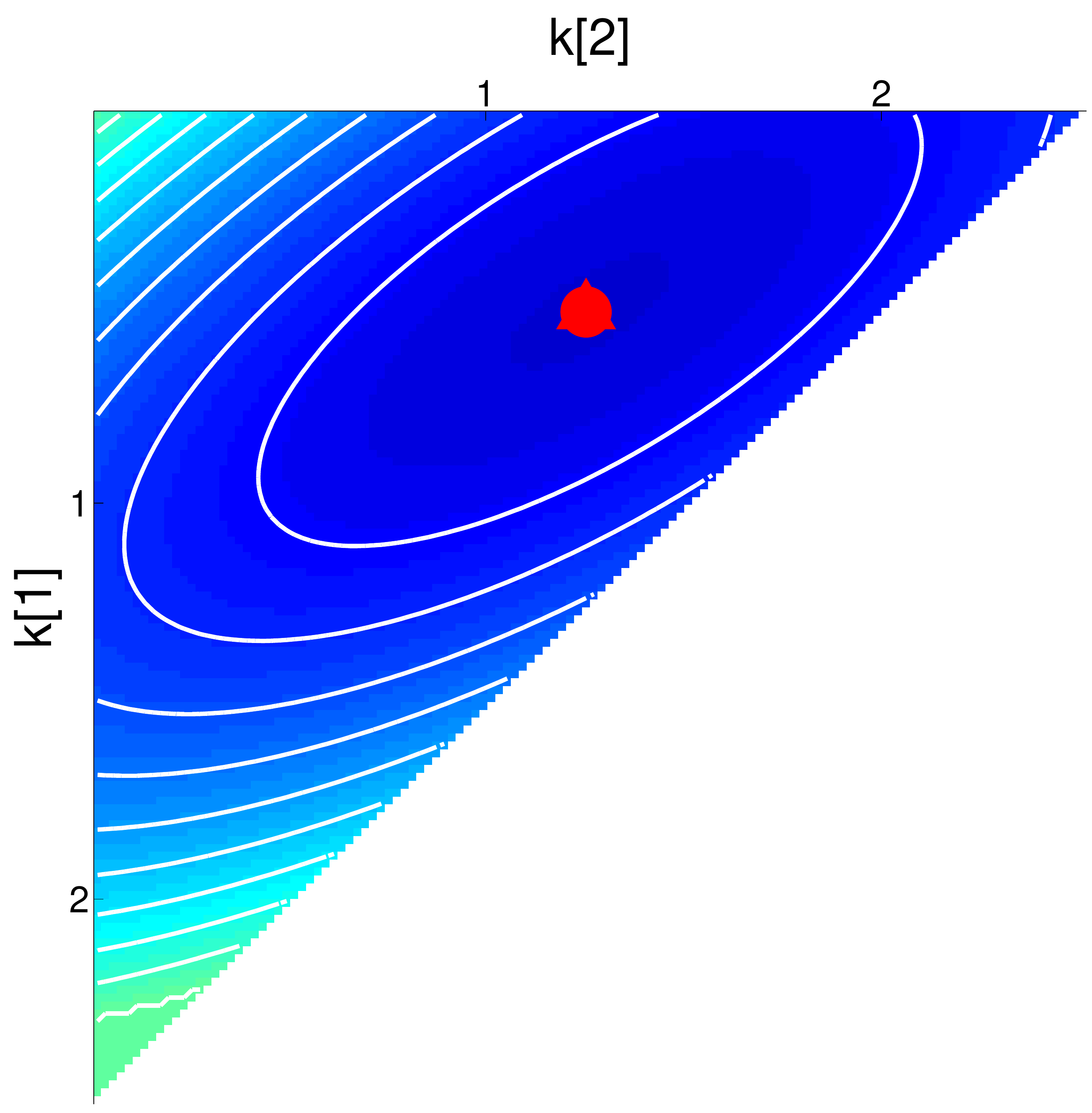}
\footnotesize{$||k||_1 = 2.5$}
\end{minipage}
\begin{minipage}[c]{.01\textwidth}
\centering
\footnotesize{\begin{turn}{90}$\lambda = 2.5$\end{turn}}
\end{minipage}
\medskip
\caption{Illustration of Theorem~\ref{the:pam} (best viewed in color). Each row represents the visualization of the cost function for a particular value of the parameter $\lambda$. Each column shows the cost function for three different blur normalizations: $||k||_1 = 1,~1.5,~\mbox{and}~2.5$.  We denote the scaled true blur $k_0=[0.2,~0.5,~0.3]$ (with $||k||_1 = 1$) with a red triangle and with a red dot the cost function minimum. The color coding is such that: blue $<$ yellow $<$ red; each row shares the same color coding for cross comparison. 
\label{fig:norms} }
\end{figure}

\subsection{Discussion}
\label{sec:dis}
Theorems~\eqref{the:pam} and~\eqref{the:am} show that with 1D zero-mean step signals and no-blur initialization, for some values of $\lambda$ PAM converges to the correct blur (and only in $2$ steps) while AM does not. We used a step signal to keep the discussion simple and because it is fundamental to illustrate the behavior of both algorithms at edges in a 2D image. However, as already mentioned in sec.~\ref{sec:ex}, extensions of the above results to blurs with a larger support and beyond step functions are possible by dividing the signal in many step functions and by a careful treatment of the boundaries.

In Fig.~\ref{fig:norms} we illustrate two further aspects of Theorem~\ref{the:pam} (it is recommended that images are viewed in color): 1) the advantage of a non unitary normalization of $k$ during the optimization step (which is a key feature of PAM) and 2) the need for a sufficiently large regularization parameter $\lambda$. In the top row images we set $\lambda=0.1$. Then, we show the cost $\|k \ast u^1 - f\|_2^2$, with $u^1 = \arg\min_u \|u-f\|_2^2+\lambda J(u)$, for all possible 1D blurs $k$ with a $3$-pixel support under the constraints $\|k\|_1 = 1,~1.5,~2.5$ respectively. This is the cost that PAM minimizes at the second step when initialized with $k = \delta$. Because $k$ has three components and we fix its normalization, we can illustrate the cost as a 2D function of $k[1]$ and $k[2]$ as in Fig.~\ref{fig:1dtoy}. However, as the normalization of $k$ grows, the triangular domain of $k[1]$ and $k[2]$ increases as well. Since the optimization of the blur $k$ is unconstrained, the optimal solution will be searched both within the domain and across normalization factors. Thanks to the color coding scheme, one can immediately see that the case of $\|k\|=1$ achieves the global minimum, and hence the solution is the Dirac delta.
However, as we set $\lambda=1.5$ in the second row or $\lambda=2.5$ in the bottom row, 
we can see a shift of the optimal value for non unitary blur normalization values and also for a shift of the global minimum to the desired solution (bottom-right plot).

\section{Implementation}
In Algorithm~\ref{algo} we show the pseudo-code of our adaptation of TVBD. At each iteration we perform just one gradient descent step on $u$ and on $k$ as proposed by Chan and Wong~\cite{Chan1998}  because we experimentally noticed that this speeds up the convergence of the algorithm.
The gradient descent results in the following update for the sharp image $u$ at the $t$-th iteration
\begin{align}\label{eq:steps}
u^{t+1} \!\leftarrow\! u^t - \epsilon_u \left(k^{t}_- \bullet (k^{t} \circ u^{t} - f) - \lambda \nabla \cdot \frac{\nabla u^{t}}{|\nabla u^{t}|}\right)
\end{align}
for some step $\epsilon>0$ where $k_-(\mathbf{x}) = k(-\mathbf{x})$ and $\bullet$ denotes the discrete convolution operator where the result is the full convolution region, \ie, if $f = u \bullet k$, with $k \in \mathbf{R}^{h \times w}$, $u \in \mathbf{R}^{m \times n}$, then we have $f \in \mathbf{R}^{(m + h - 1) \times (n + w -1)}$ with zero padding as boundary condition. 
The iteration on the blur kernel $k$ is instead given by
\begin{align}
k^{t+1/3} \leftarrow k^t - \epsilon_k \left(u^{t+1}_- \circ  (k^{t} \circ u^{t+1} - f) \right ).
\end{align}

\begin{algorithm}[t!]
\caption{Blind Deconvolution Algorithm \label{algo}}
\small{
 \SetAlgoLined
 \KwData{$f$, size of blur, initial large $\lambda$, final $\lambda_{min}$}
 \KwResult{$u$,$k$ }
 $u_0 \leftarrow \mbox{pad}(f)$\;
 $k_0 \leftarrow \mbox{uniform}$\;
 \While{not converged}{
 $u^{t+1} \!\leftarrow\! u^t - \epsilon_u \left(k^{t}_- \bullet (k^{t} \circ u^{t} - f) - \lambda \nabla \cdot \frac{\nabla u^{t}}{|\nabla u^{t}|}\right)$\;
 $k^{t+1/3} \leftarrow k^t - \epsilon_k \left(u^{t+1}_- \circ  (k^{t} \circ u^{t+1} - f) \right ) $\;
 $k^{t+2/3} \leftarrow  \max\{k^{t+1/3},0\}$\;
 $k^{t+1} \leftarrow \frac{k^{t+2/3}}{\| k^{t+2/3}\|_1}$\;
 $\lambda \leftarrow \max\{0.99\lambda,\lambda_{min}\}$\;
 $t \leftarrow t + 1$\;
  }
  $u \leftarrow u^{t+1}$\;
  $k \leftarrow k^{t+1}$\;
}
\end{algorithm}
From Theorem~\eqref{the:pam} we know that a big value for the parameter $\lambda$ helps avoiding the no-blur solution, but in practice it also makes the estimated sharp image $u$ too ``cartooned''. We found that iteratively reducing the value of $\lambda$ as specified in Algorithm~\ref{algo} helps getting closer to the true solution. 

In the following paragraphs we highlight some other important features of Algorithm~\ref{algo}.\\
\textbf{Pyramid scheme.} 
While all the theory holds at the original input image size, to speed up the algorithm we also make use of a pyramid scheme, where we scale down the blurry image and the blur size until the latter is $3 \times 3$ pixels. We then launch our deblurring algorithm from the lowest scale, then upsample the results and use them as initialization for the following scale. This procedure provides a significant speed up of the algorithm. On the smallest scale, we initialize our optimization from a uniform blur.\\
\textbf{Color images.}
For the blur estimation many methods first convert color images to grayscale. In contrast with this common practice, we extended our algorithm to work directly with color images. Recently, many papers have proposed algorithms for color denoising using TV regularization~\cite{Saito2011,Goldlucke2010,Blomgren1996}. In our implementation we use the method of Blomgren and Chan~\cite{Blomgren1996}, and all the results on color images that we show in this papers are obtained by solving the blind-deconvolution problem on the original color space.

\section{Experiments}
We evaluated the performance of our algorithm on several images and compared with state-of-the-art algorithms. We provide our unoptimized Matlab implementation on our website\footnote{http://www.cvg.unibe.ch/dperrone/tvdb/}. Our blind deconvolution implementation processes images of $255\times 225$ pixels with blurs of about $20\times20$ pixels in around 2-5 minutes, while our non-blind deconvolution algorithm takes about $10-30$ seconds.

In our experiments we used the dataset from~\cite{Levin2011Understanding} in the same manner illustrated by the authors. For the whole test we used $\lambda_{min} =0.0006$. We used the non-blind deconvolution algorithm from~\cite{Levin2007} with $\lambda = 0.0068$ and for each method we compute the cumulative histogram of the deconvolution error ratio across test examples, where the $i$-th bin counts the percentage of images for which the error ratio is smaller than $i$. The deconvolution error ratio, as defined in~\cite{Levin2011Understanding}, measures the ratio between the SSD deconvolution error with the estimated and correct kernels. In Fig.~\ref{fig:ssd} we show the cumulative histogram of the error ratio for Algorithm~\ref{algo}, the algorithm of Levin \etal~\cite{Levin2007}, the algorithm of Fergus \etal~\cite{Fergus2006}, the algorithm of Babacan \etal~\cite{Babacan2012}, the algorithm of Wipf and Zhang~\cite{Wipf2013} and the one of Cho and Lee~\cite{Cho2009}. Algorithm~\ref{algo} performs on par with the one from Levin \etal~\cite{Levin2007}, with a slightly higher number of restored images with small error ratio. 
 
In Fig.~\ref{fig:jia} we show a comparison between our method and the one proposed by Xu and Jia~\cite{Xu2010}. Their algorithm is able to restore sharp images even when the blur size is large by using an edge selection scheme that selects only large edges. This behavior is automatically minimicked by Algorithm~\ref{algo} thanks to the TV prior. Also, in the presence of noise, Algorithm~\ref{algo} performs visually on a par with the state-of-the-art algorithms as shown in Fig.~\ref{fig:zhong}. 

Recently K{\"o}hler \etal~\cite{Kohler2012} introduced a dataset of images blurred by camera shake blur. Even if this kind of artifact produces space-varying blur, many algorithms that assume shift-invariant blur perform well on many images of the dataset. In Fig.~\ref{Kohler} we show an example from the dataset of~\cite{Kohler2012} where our algorithm is as robust to camera shake blur as other shift-invariant algorithms.

For testing purposes we also adapted our algorithm to support different boundary conditions by substituting the convolution operator described in~\eqref{eq:vconv}  with the one in~\eqref{eq:conv}. 

In  Fig.~\ref{fig:ssdBound} we show a comparison on the dataset of~\cite{Levin2011Understanding} between our original approach and the adaptations with different boundary conditions. For each boundary condition we evaluated the ratios between the SSD deconvolution errors of the estimated and correct kernels. The implementations with the different boundary conditions perform worse than our free-boundary implementation, even if pre-processing the blurry image with the method of Liu and Jia \cite{Liu2008} considerably improves the performance of the \emph{periodic} boundary condition.

Recent algorithms estimate the blur by using filtered versions of $u$ and $f$ in the data fitting term (typically the gradients or the Laplacian) . This choice might improve the estimation of the blur because it reduces the error at the boundary when using any of the previous approximations, but it might result also in a larger sensitivity to noise. In Fig.~\ref{fig:ssdBoundFilt} we show how with the use of the filtered images for the blur estimation the performance of the \emph{periodic} and \emph{replicate} boundary conditions improves, while the others get slightly worse. Notice that our implementation still achieves better results than other boundary assumptions.

\begin{figure*}[t]
\centering 
\begin{minipage}[t]{0.29\linewidth}
\centering
\includegraphics[width=\textwidth,height=4cm]{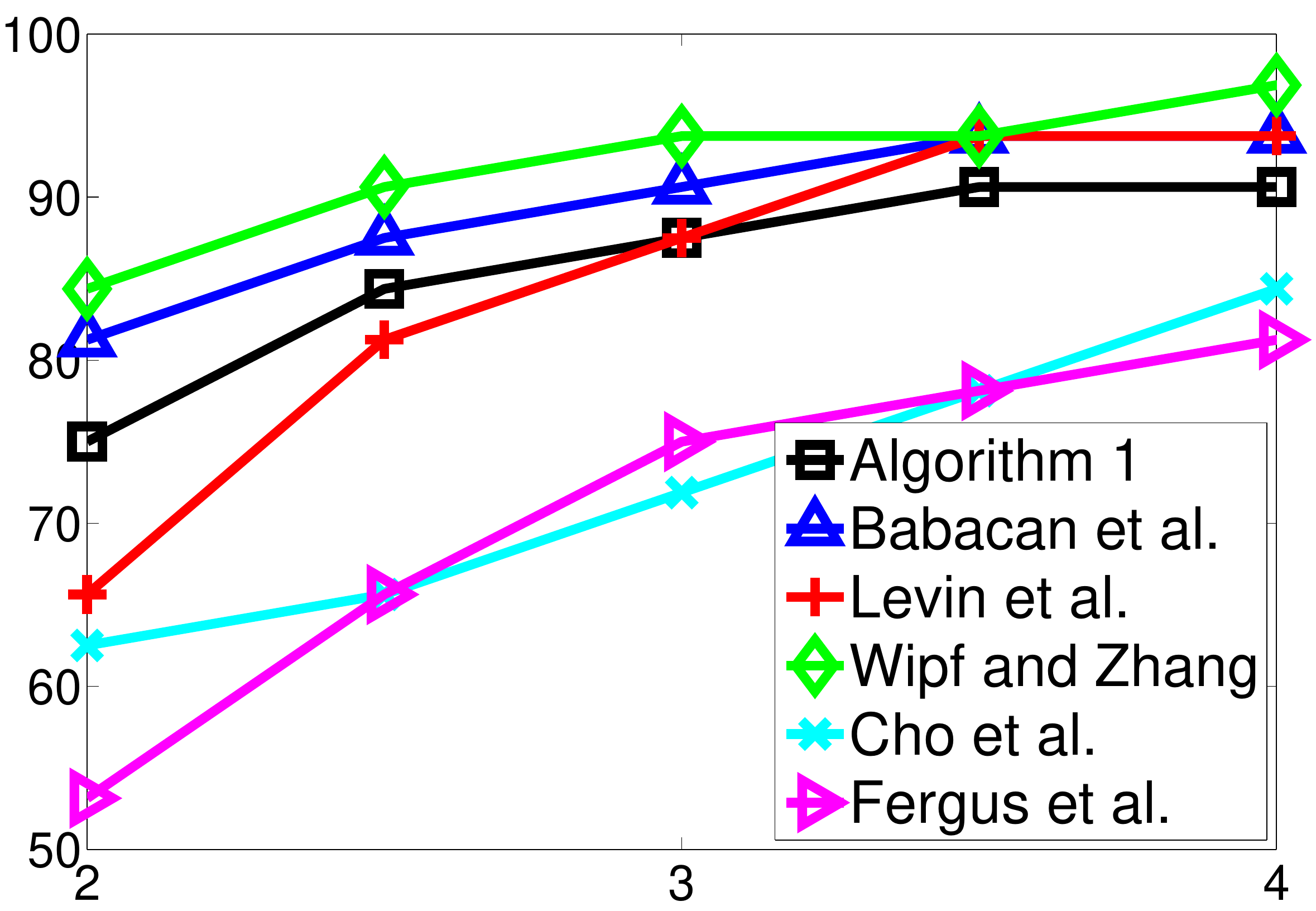}
\caption{Comparison between Algorithm~\ref{algo} and recent state-of-the-art algorithms.}
\label{fig:ssd}
\end{minipage}
\hspace{0.05cm}
\begin{minipage}[t]{0.29\linewidth}
\centering
\includegraphics[width=\textwidth,height=4cm]{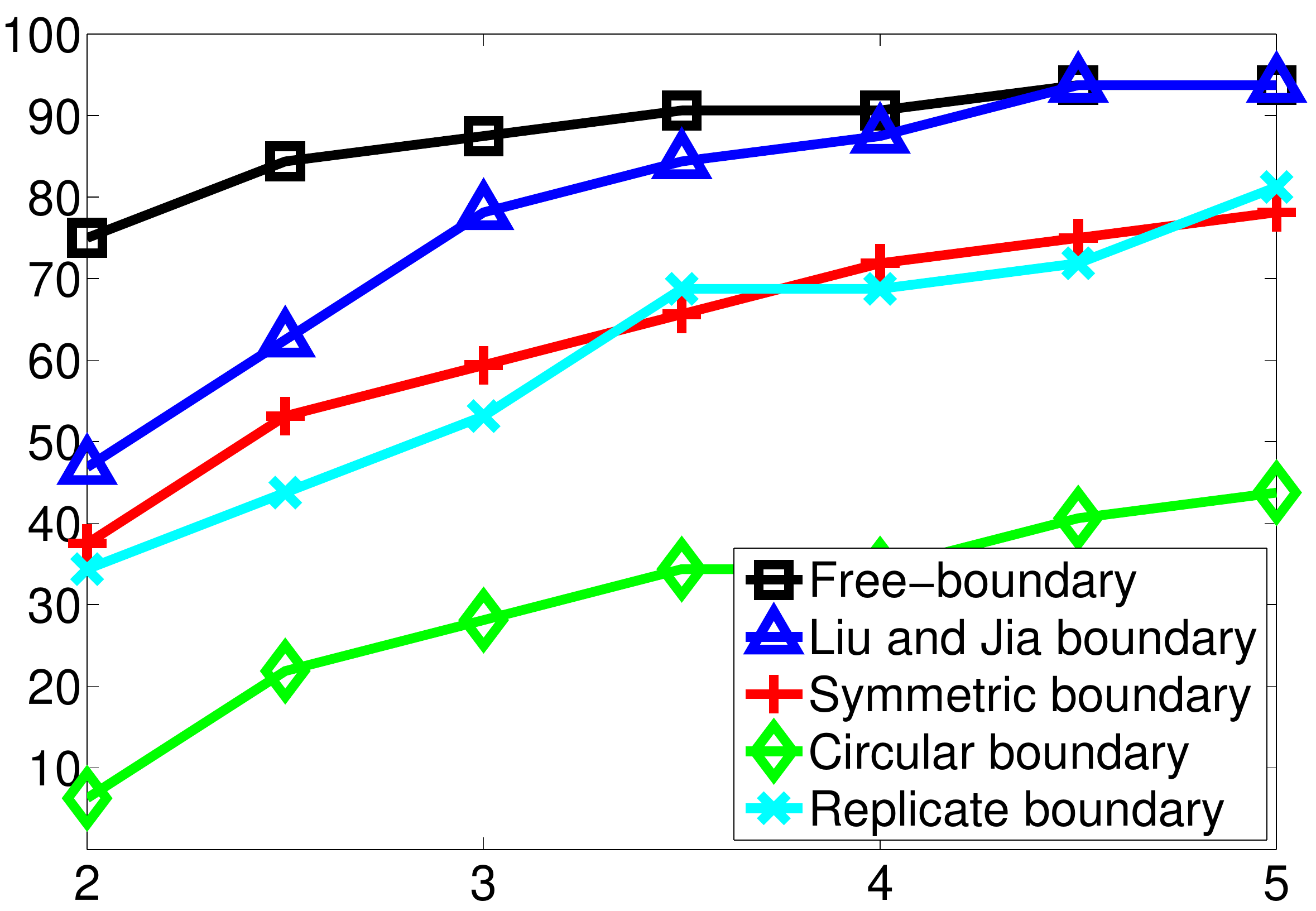}
\caption{Comparison of Algorithm~\ref{algo} with different boundary conditions.}
\label{fig:ssdBound}
\end{minipage}
\hspace{0.05cm}
\begin{minipage}[t]{0.29\linewidth}
\centering
\includegraphics[width=\textwidth,height=4cm]{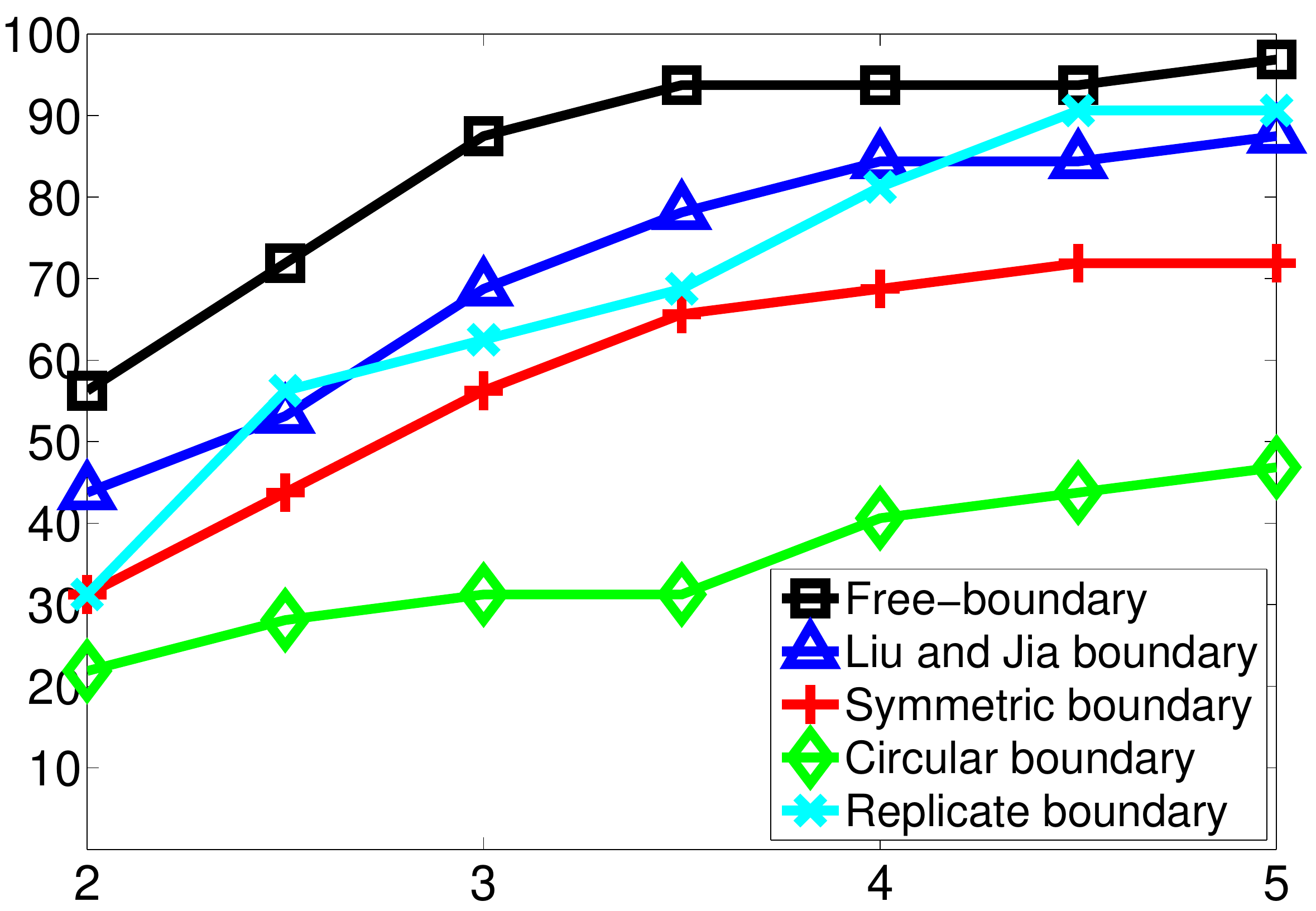}
\caption{As in Fig.~\ref{fig:ssdBound}, but using a filtered version of the images for the blur estimation.}
\label{fig:ssdBoundFilt}
\end{minipage}
\end{figure*}

\begin{figure*}[t] 
 \centering 
 \begin{minipage}[t]{.3\textwidth} 
  \centering 
 \includegraphics[width=\textwidth]{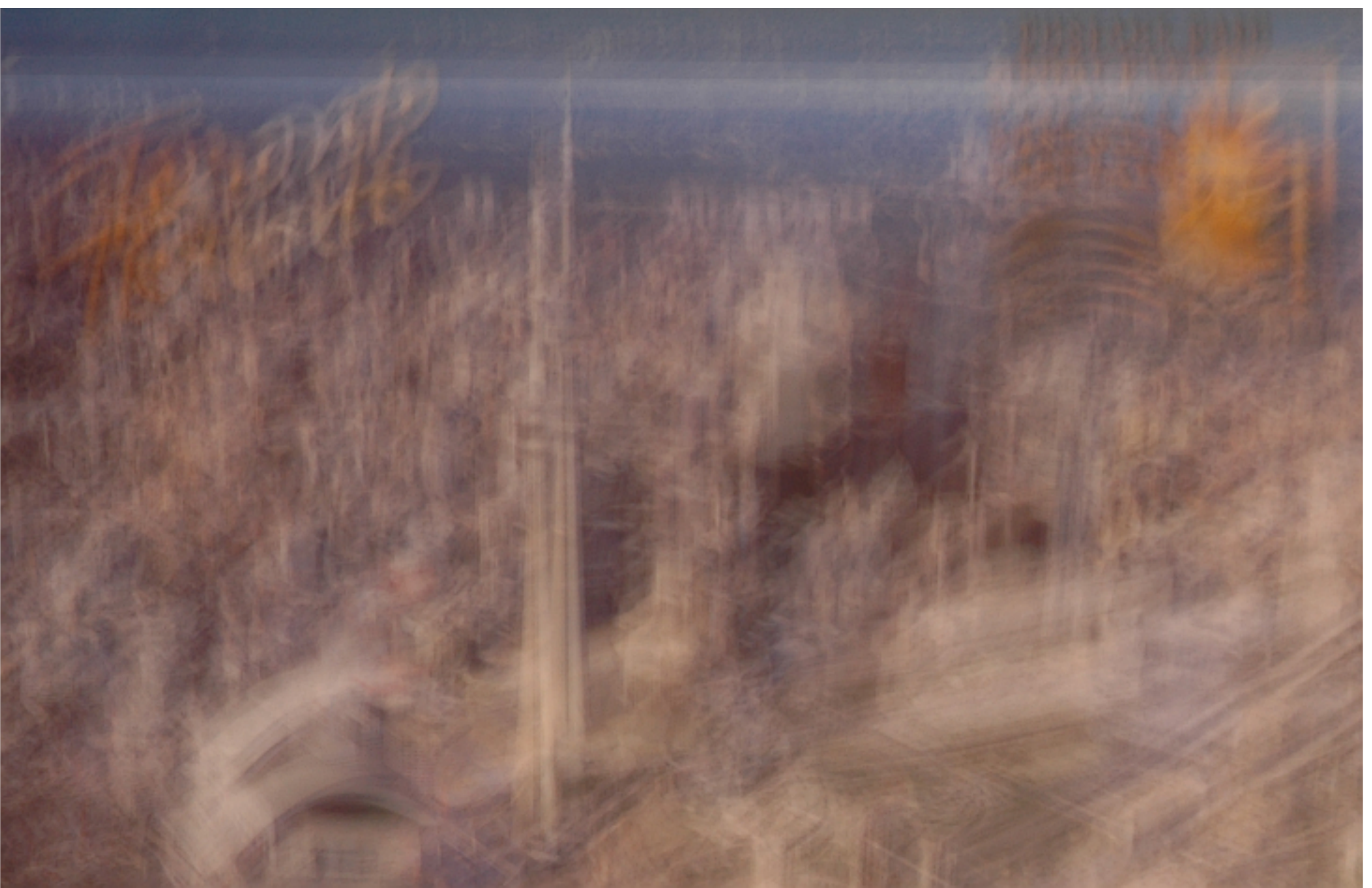} \footnotesize{Blurry Input. }
 \end{minipage} 
 \begin{minipage}[t]{.3\textwidth} 
  \centering 
 \includegraphics[width=\textwidth]{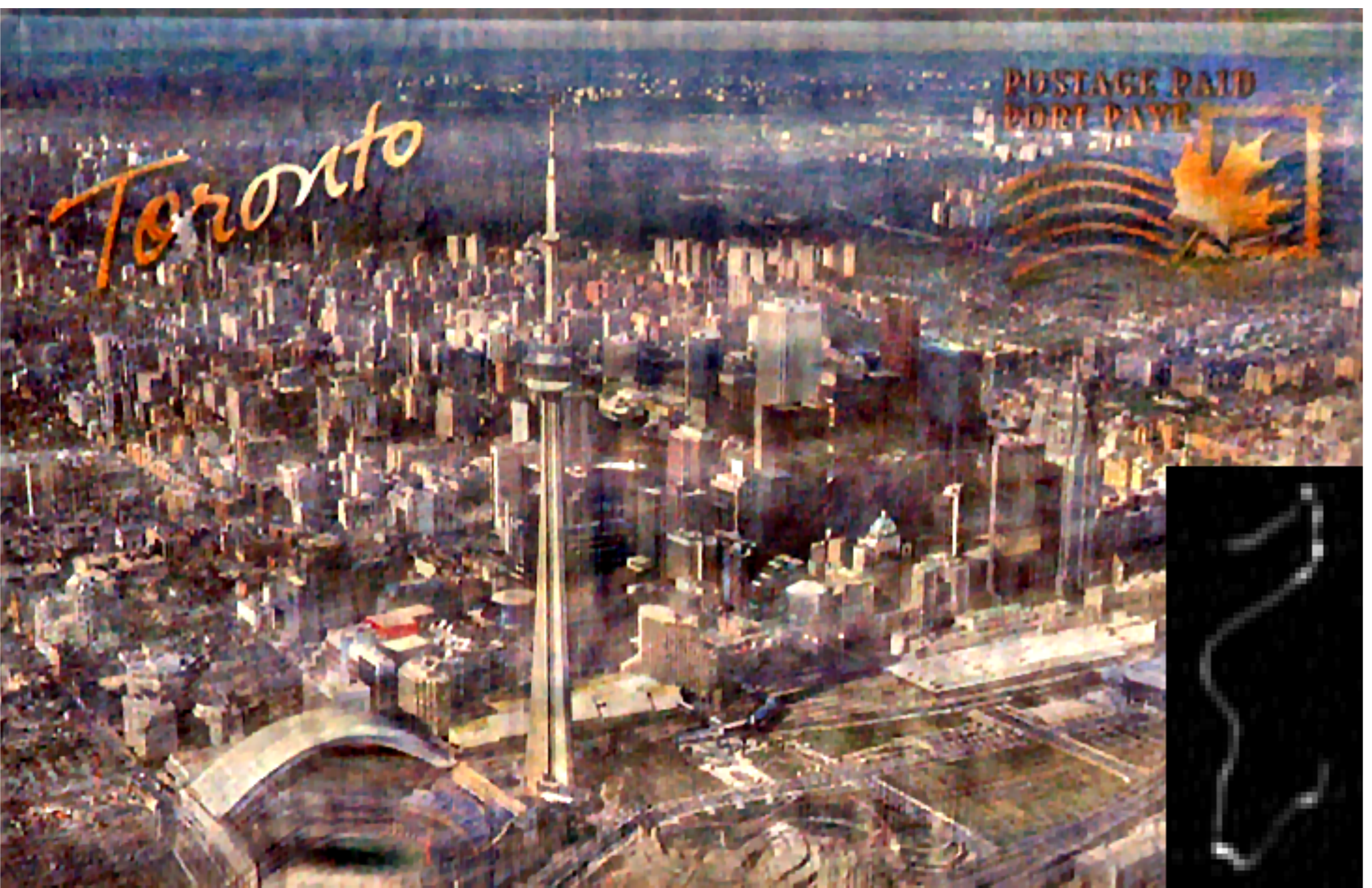} \footnotesize{Restored image and blur with Xu and Jia~\cite{Xu2010}. }
 \end{minipage} 
 \begin{minipage}[t]{.3\textwidth} 
  \centering 
 \includegraphics[width=\textwidth]{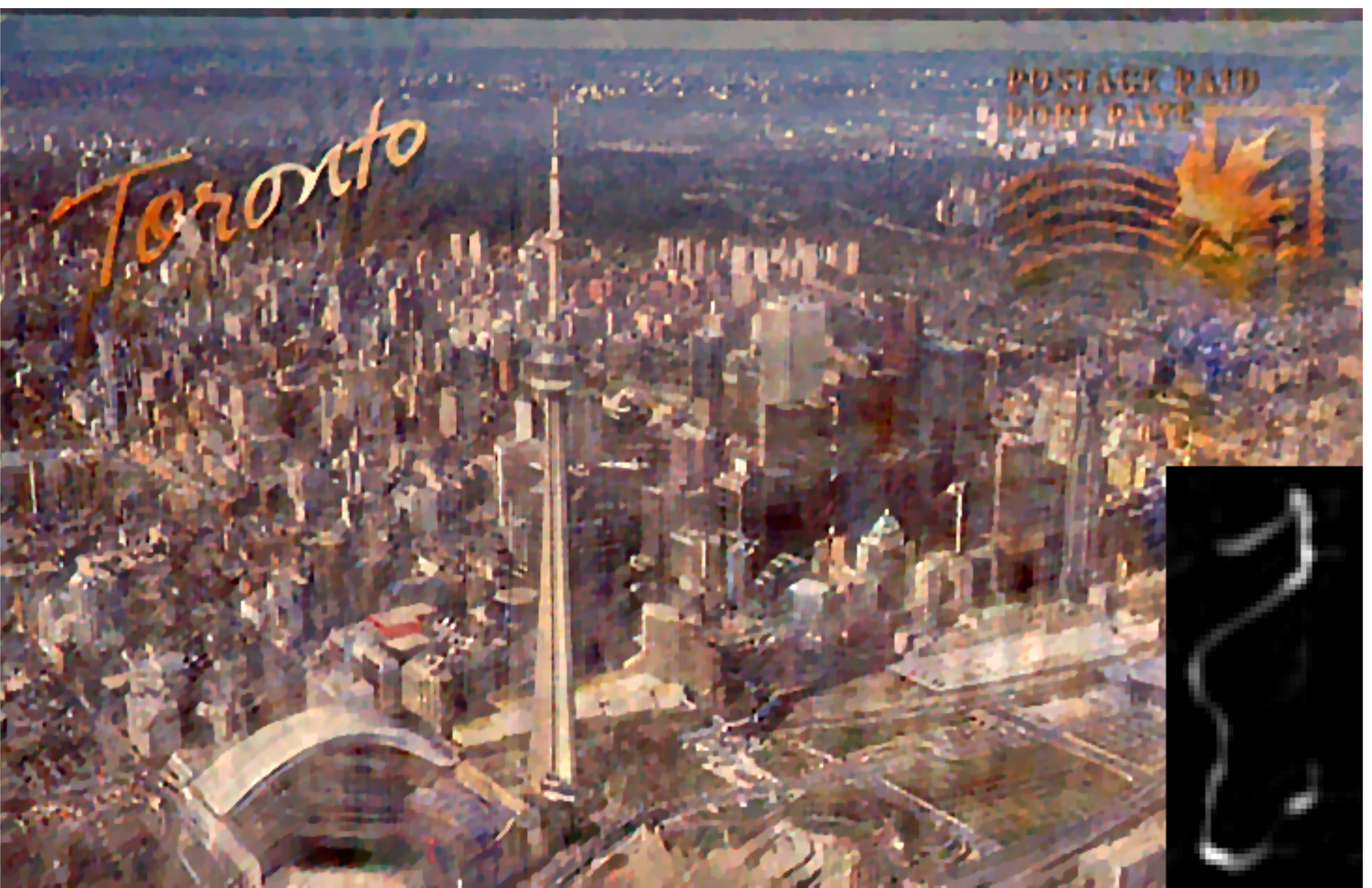} \footnotesize{Restored image and blur with Algorithm~\ref{algo}.}
 \end{minipage} 
 \caption{Example of blind-deconvolution image and blur (bottom-right insert) restoration. \label{fig:jia}}
 \end{figure*} 

 \begin{figure}[h] 
 \centering 
  \begin{minipage}[t]{.153\textwidth} 
  \centering 
 \includegraphics[width=\textwidth]{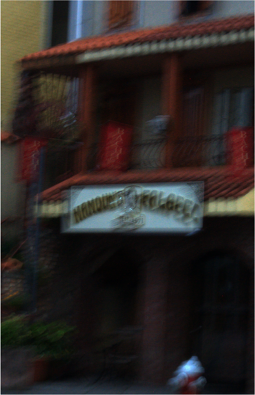} \footnotesize{Blurry Input. }
 \end{minipage} 
 \begin{minipage}[t]{.15\textwidth} 
  \centering 
 \includegraphics[width=\textwidth]{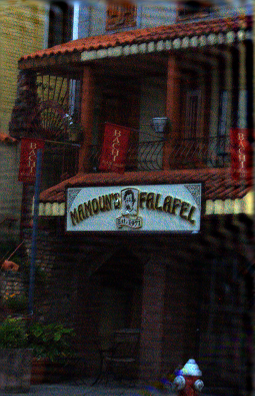} \footnotesize{Cho and Lee~\cite{Cho2009}. }
 \end{minipage} 
 \begin{minipage}[t]{.1505\textwidth} 
  \centering 
 \includegraphics[width=\textwidth]{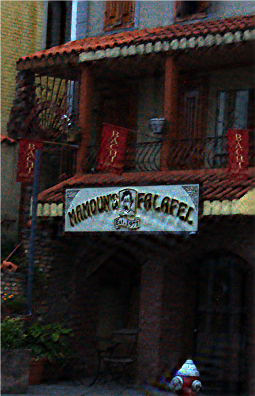} \footnotesize{Levin \etal~\cite{Levin2011}. }
 \end{minipage} 
  \begin{minipage}[t]{.1505\textwidth} 
  \centering 
 \includegraphics[width=\textwidth]{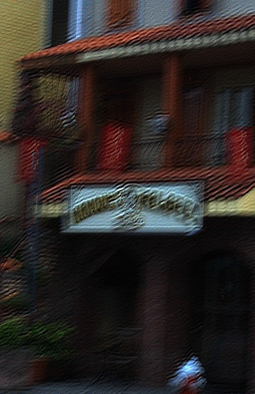} \footnotesize{Goldstein and Fattal~\cite{Goldstein2012}. }
 \end{minipage} 
  \begin{minipage}[t]{.15\textwidth} 
  \centering 
 \includegraphics[width=\textwidth]{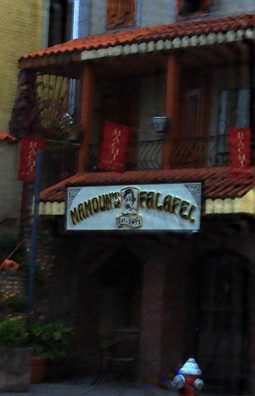} \footnotesize{Zhong \etal~\cite{Zhong2013}. }
 \end{minipage} 
 \begin{minipage}[t]{.152\textwidth} 
  \centering 
 \includegraphics[width=\textwidth]{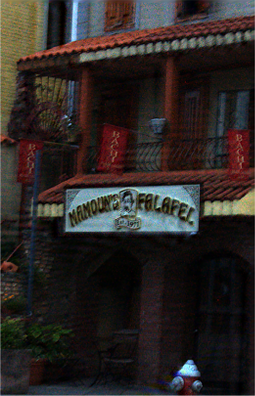} \footnotesize{Algorithm~\ref{algo}. }
 \end{minipage}
 \caption{Examples of blind-deconvolution restoration. \label{fig:zhong}}
 \end{figure} 

\begin{figure}[h] 
 \centering 
  \begin{minipage}[t]{.152\textwidth} 
  \centering 
 \includegraphics[width=\textwidth]{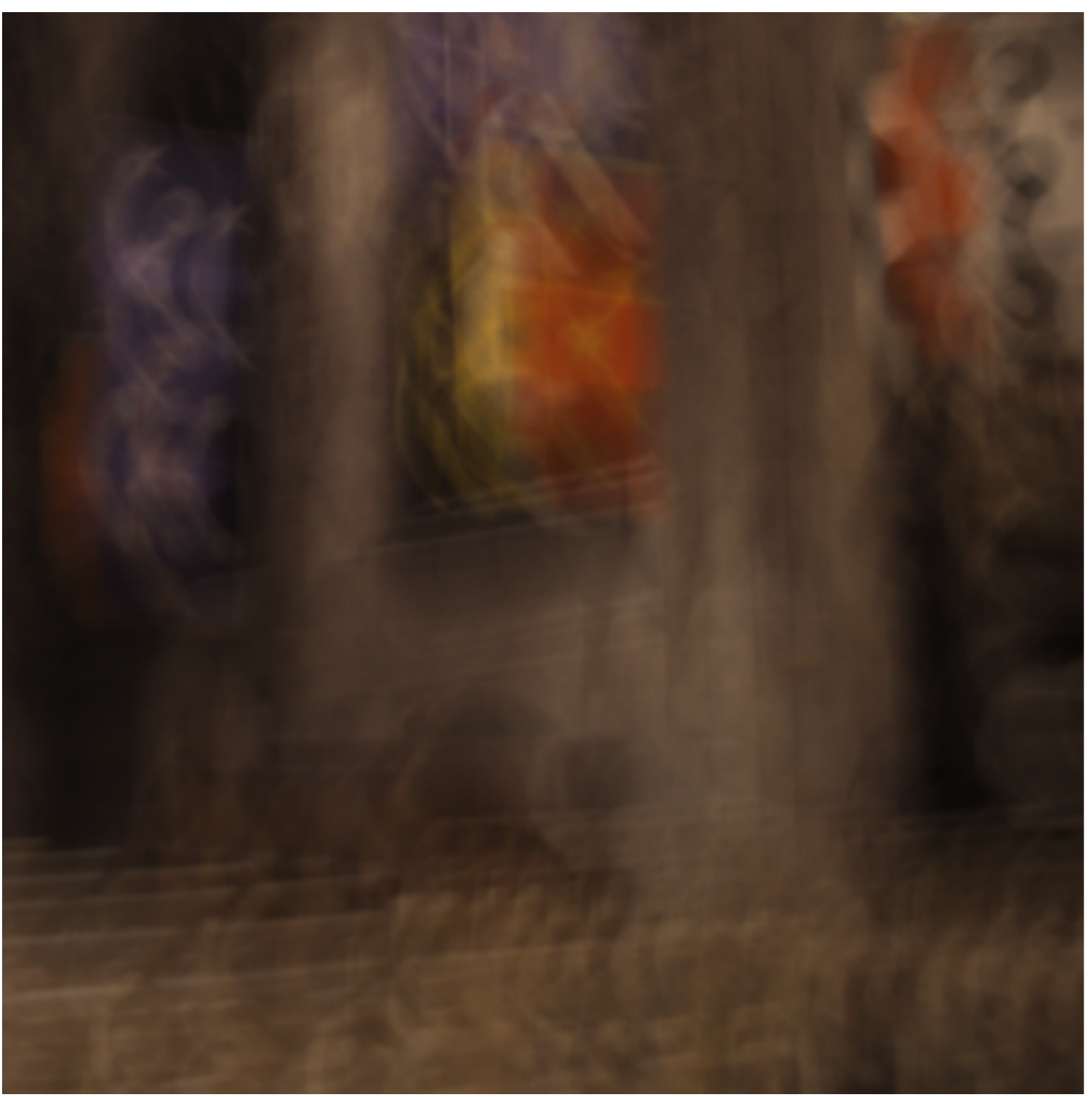} \footnotesize{Blurry Input. }
 \end{minipage}
 \begin{minipage}[t]{.152\textwidth} 
  \centering 
 \includegraphics[width=\textwidth]{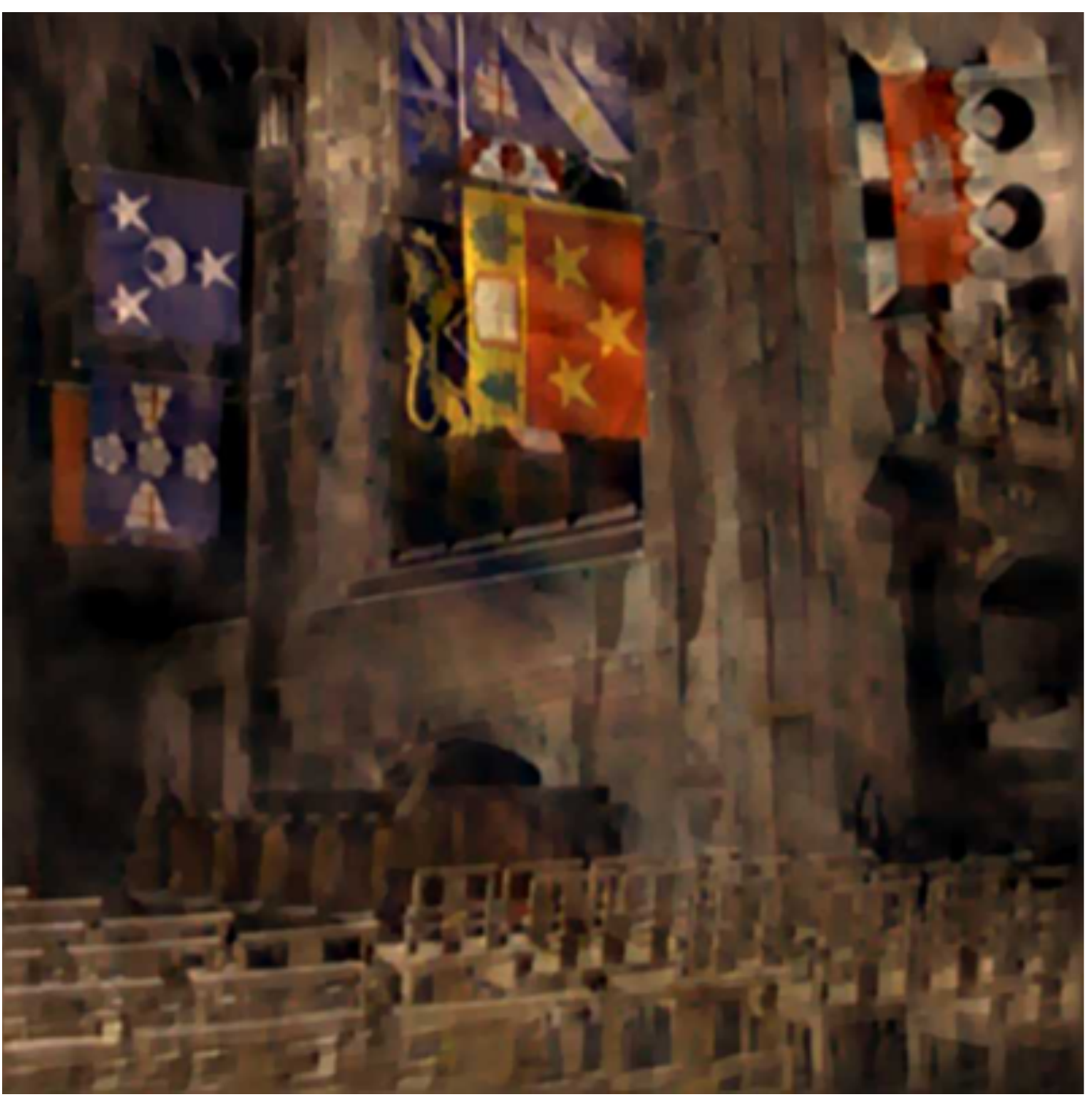} \footnotesize{Restored image with Cho and Lee~\cite{Cho2009}. }
 \end{minipage}
  \begin{minipage}[t]{.152\textwidth} 
  \centering 
   \includegraphics[width=\textwidth]{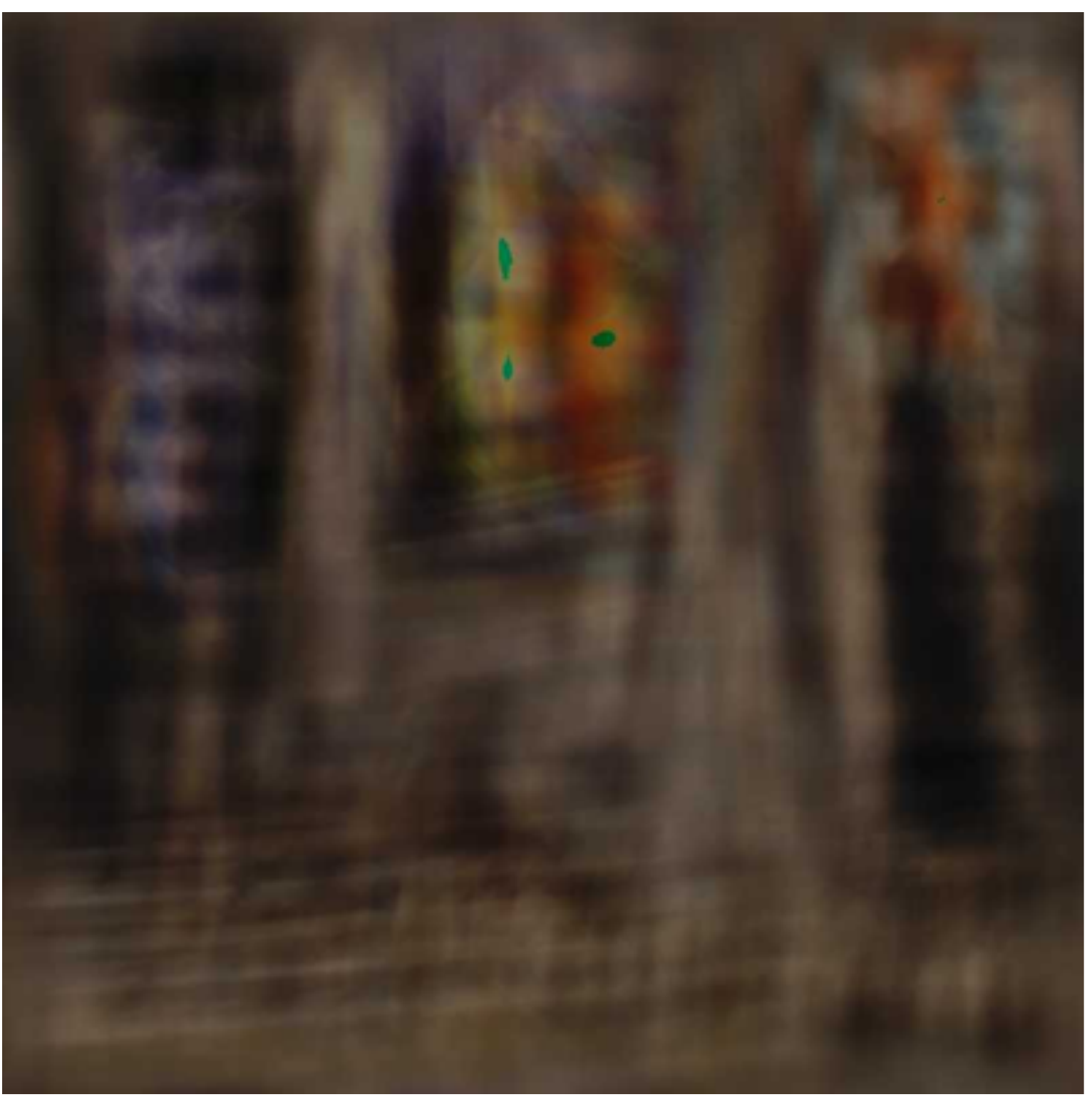} \footnotesize{Restored image with Fergus \etal~\cite{Fergus2006}. }
\end{minipage}
   \begin{minipage}[t]{.152\textwidth} 
  \centering 
 \includegraphics[width=\textwidth]{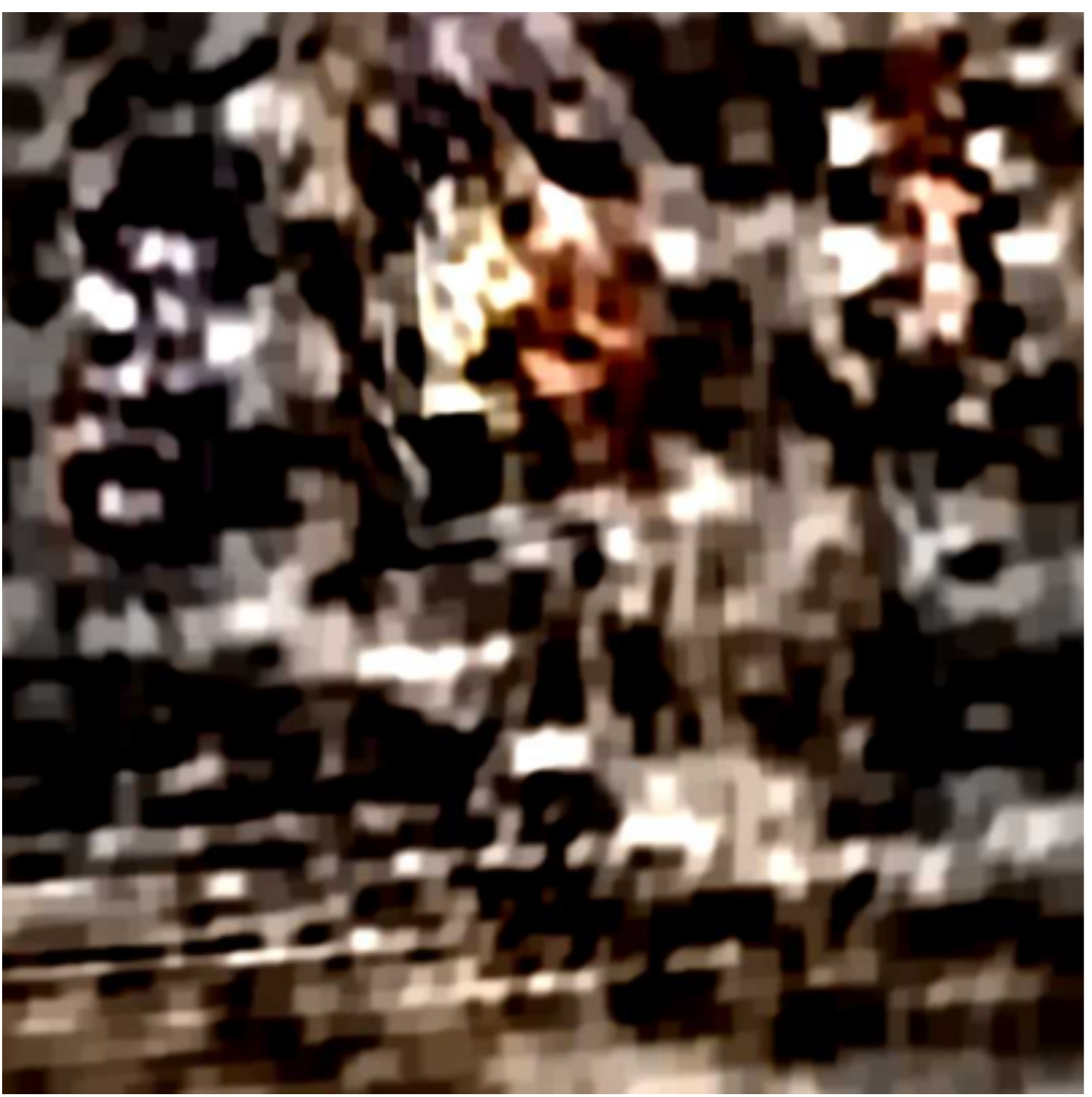} \footnotesize{Restored image with Krishnan \etal~\cite{Krishnan2011}. }
\end{minipage}
  \begin{minipage}[t]{.152\textwidth} 
  \centering 
 \includegraphics[width=\textwidth]{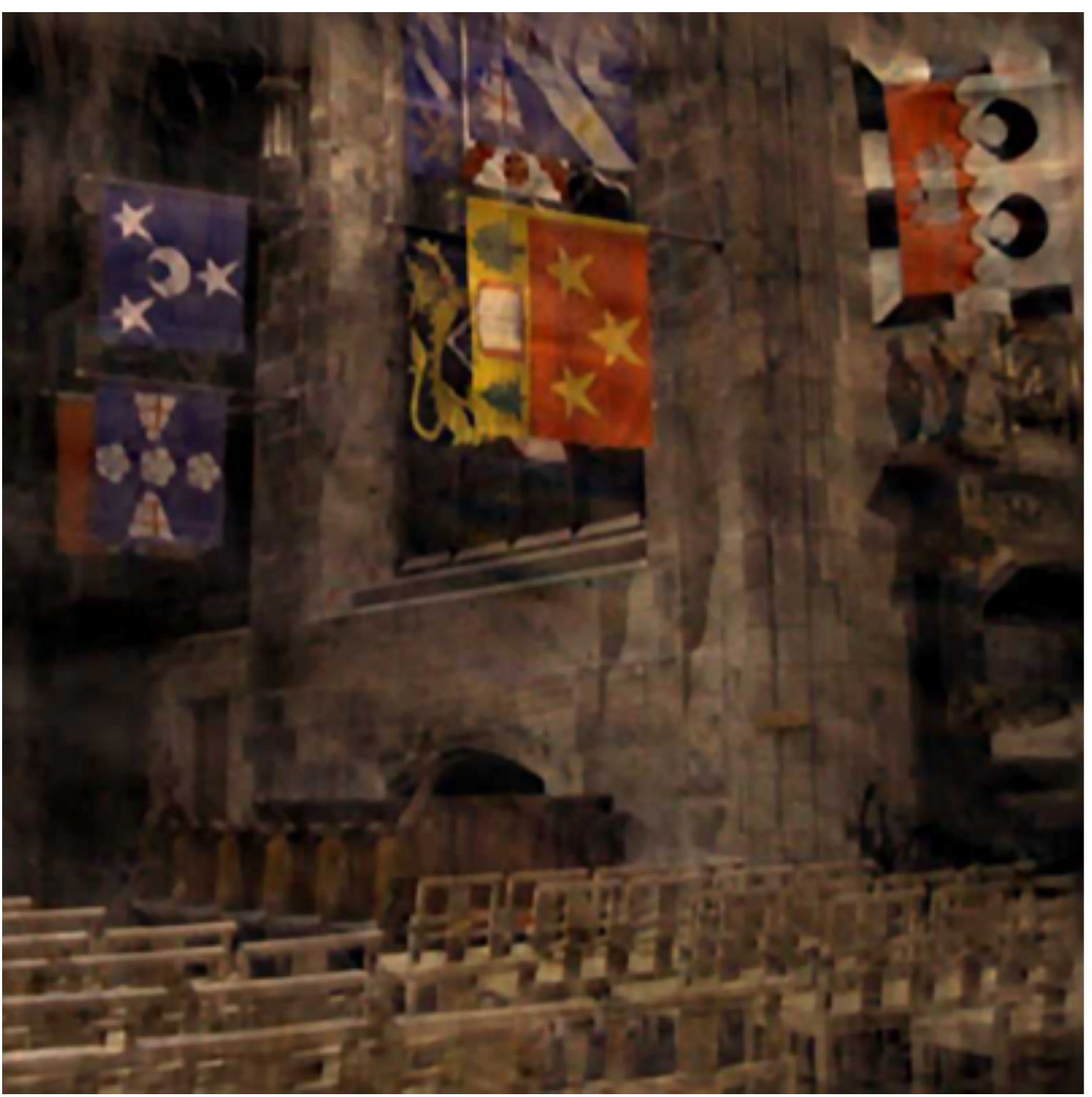} \footnotesize{Restored image with Xu and Jia~\cite{Xu2010}. }
\end{minipage}
  \begin{minipage}[t]{.152\textwidth} 
  \centering 
 \includegraphics[width=\textwidth]{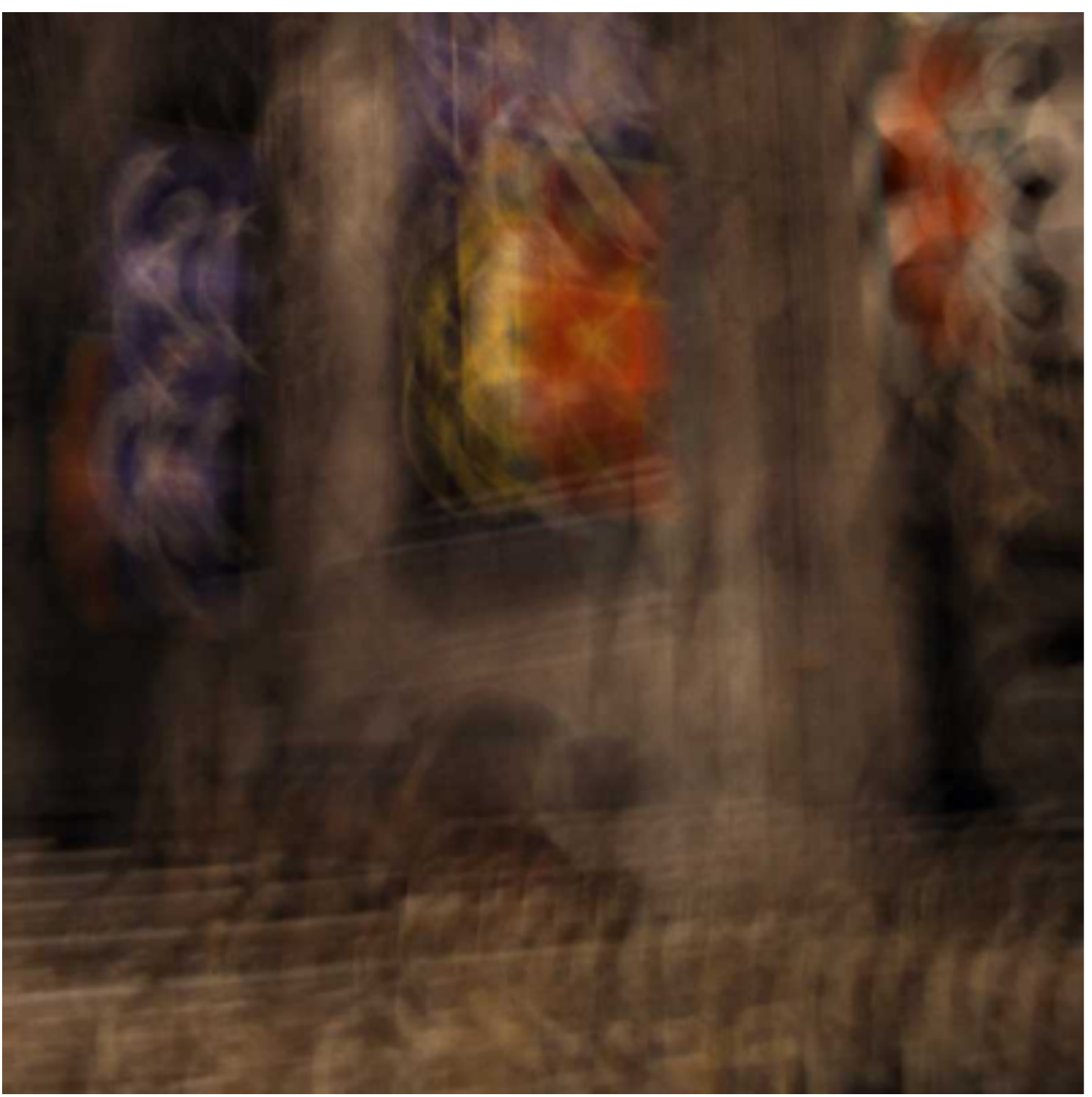} \footnotesize{Restored image with Whyte \etal~\cite{Whyte2011}. }
\end{minipage}
  \begin{minipage}[t]{.152\textwidth} 
  \centering 
 \includegraphics[width=\textwidth]{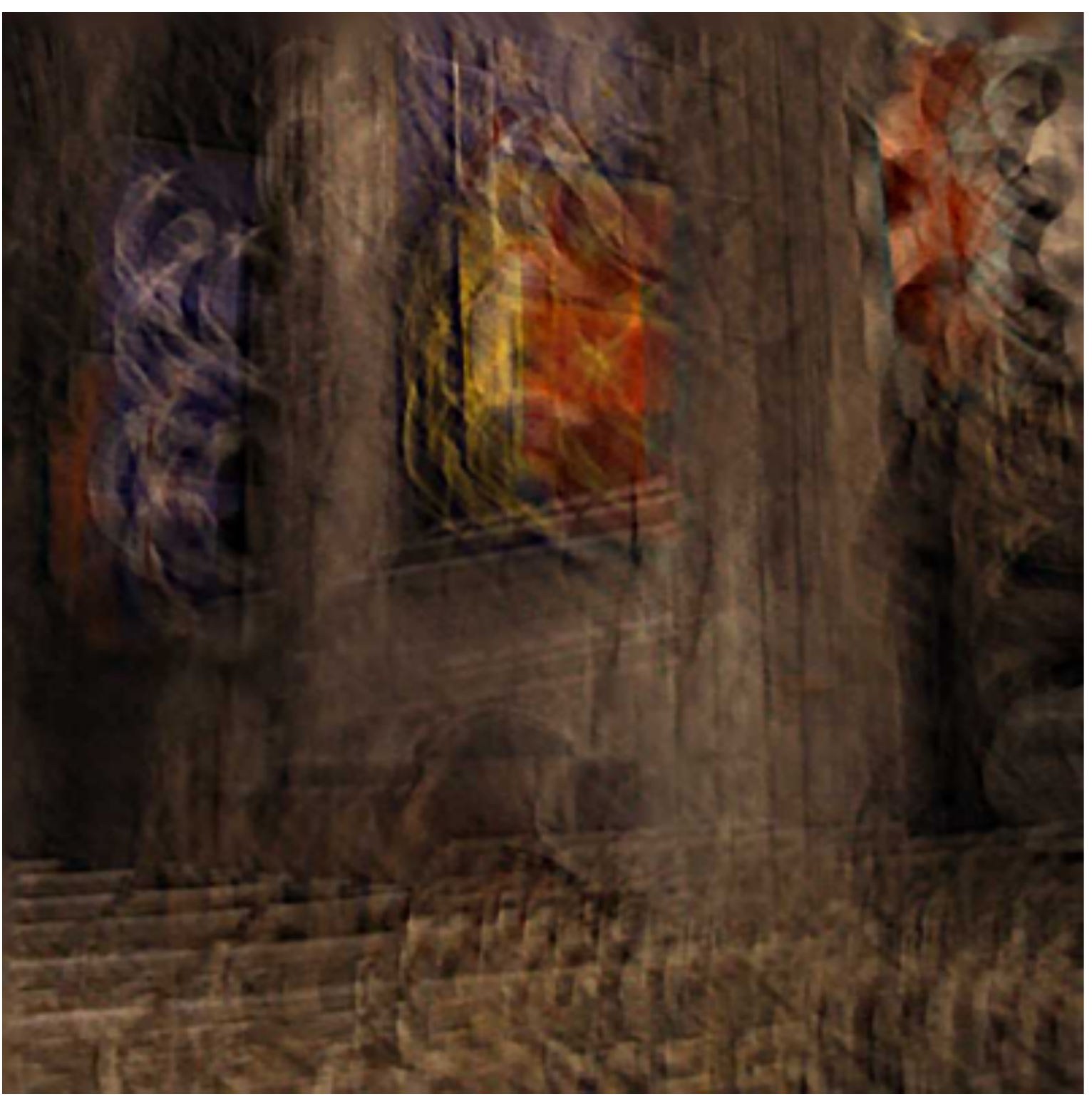} \footnotesize{Restored image with Shan \etal~\cite{Shan2008}. }
\end{minipage}
  \begin{minipage}[t]{.152\textwidth} 
  \centering 
 \includegraphics[width=\textwidth]{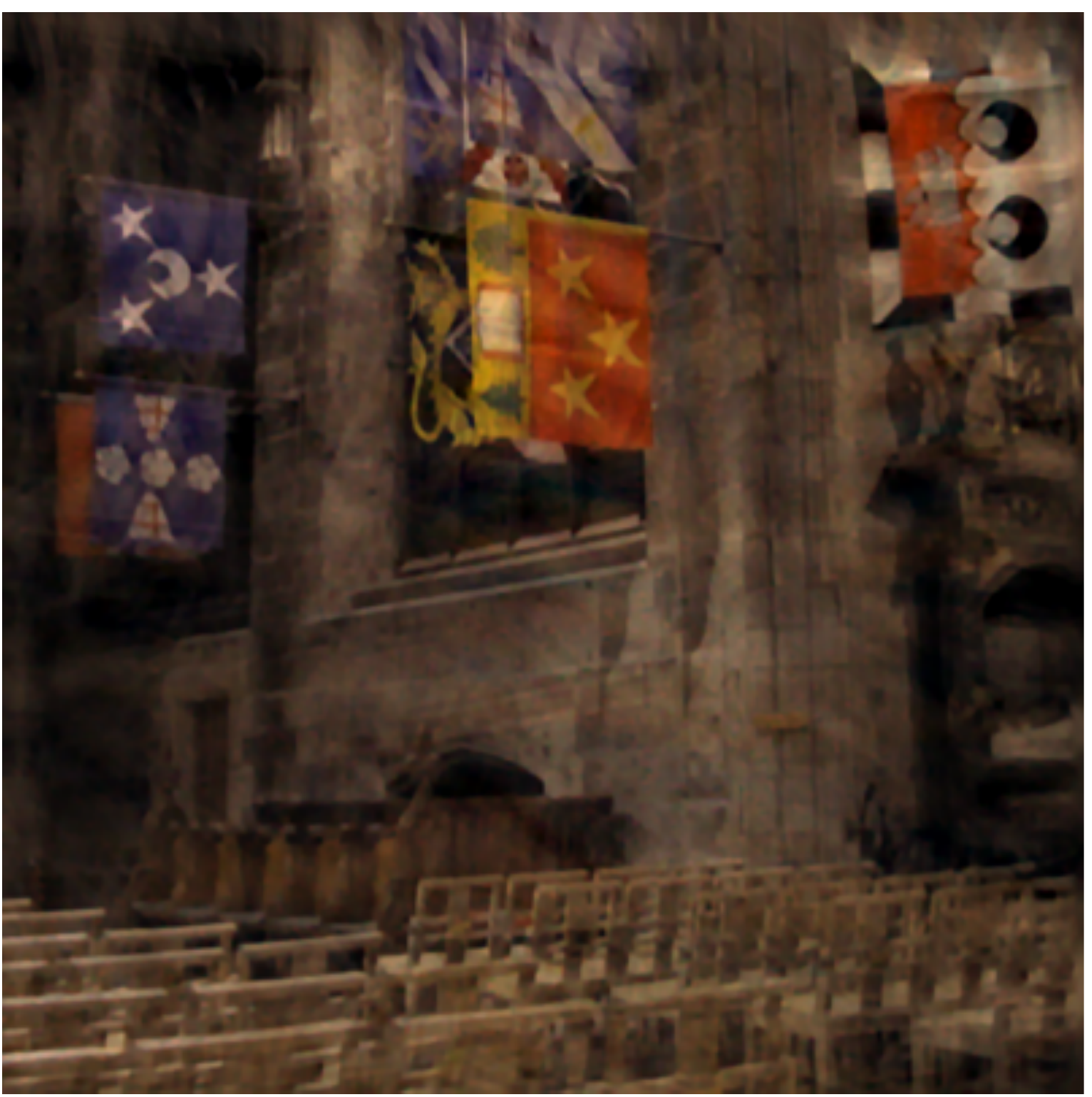} \footnotesize{Restored image with Xu \etal~\cite{Xu2013}. }
\end{minipage}
  \begin{minipage}[t]{.152\textwidth} 
  \centering 
 \includegraphics[width=\textwidth]{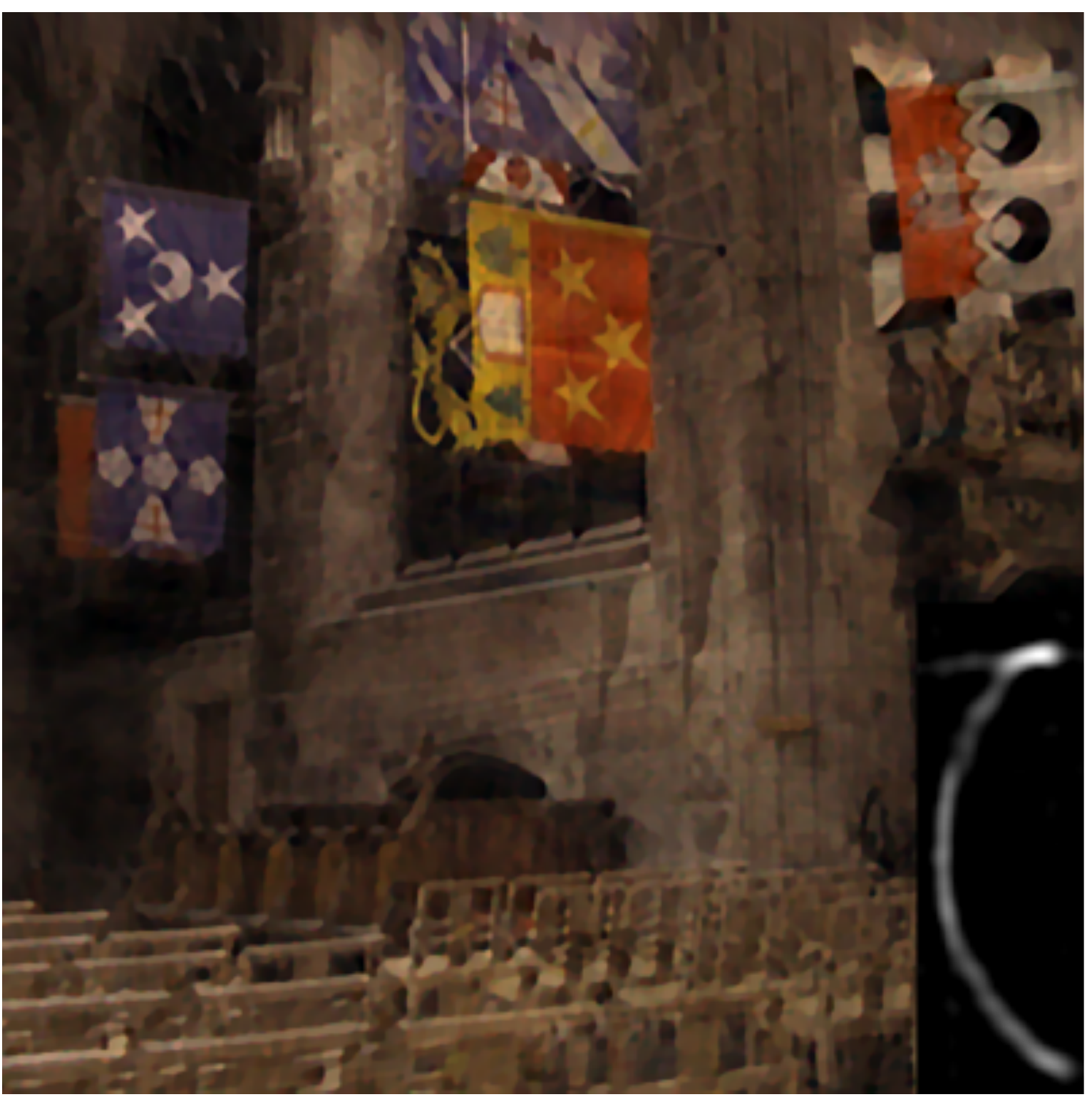} \footnotesize{Restored image with our algorithm. }
 \end{minipage}
  \caption{Examples of blind-deconvolution reconstruction from dataset~\cite{Kohler2012}. \label{Kohler}} 
 \end{figure}

\section{Conclusions}
In this paper we shed light on approaches to solve blind deconvolution. First, we confirmed that the problem formulation of total variation blind deconvolution as a maximum a priori in both sharp image and blur ($\mbox{MAP}_{u,k}$) is prone to local minima and, more importantly, does not favor the correct solution. Second, we also confirmed that the original implementation \cite{Chan1998} of total variation blind deconvolution (PAM) can successfully achieve the desired solution. This discordance was clarified by dissecting PAM in its simplest steps. The analysis revealed that such algorithm does not minimize the original $\mbox{MAP}_{u,k}$ formulation. This analysis applies to a large number of methods solving $\mbox{MAP}_{u,k}$ as they might exploit the properties of PAM; moreover, it shows that there might be principled solutions to $\mbox{MAP}_{u,k}$. We believe that by further studying the behavior of the PAM algorithm one could arrive at novel useful formulations for blind deconvolution. Finally, we have showed that the PAM algorithm is neither worse nor better than the state of the art algorithms despite its simplicity.

\appendices
\section{}
In the appendix we show the proofs of the propositions and of the theorems presented in the paper.
\subsection{Proof of Proposition~\ref{proposition} }
\label{proof:proposition}
\begin{IEEEproof}
Solving problem~\eqref{eq:rof0} is equivalent to solving the following problem
\begin{align} \label{eq:rof02}
\bar{u}[x] = \arg\min_u  \frac{1}{2}\sum_{x = -L_1+1}^{L_2-1} (u[x] - f[x])^2\quad\quad\quad\quad~\\
 + \lambda \sum_{x = -L_1+1}^{L_2-2} |u[x+1] - u[x]|.\nonumber
\end{align}
where $\bar{u}[x] = \hat{u}[x]$ for $x \in [-L_1+1,L_2-1]$, $\hat{u}[-L_1] =  \hat{u}[-L_1+1]$ and $\hat{u}[L_2] =  \bar{u}[L_2-1]$. In the following, unless specified, we will always refer to problem~\eqref{eq:rof02}. 

The solution of problem~\eqref{eq:rof02} can also be written as $\bar{u}[x] = \hat{s}[x]-\hat{s}[x-1]$, $x\in[-L_1+1,L_2-1]$, where $\hat{s}$ is found by solving the \emph{taut string} problem (\eg, see Davies and Kovac~\cite{Davies2001})
\begin{align} \label{eq:dual}
\hat{s}[x] = \arg\min_s  & \sum_{x=-L_1+1}^{L_2-1}\sqrt{1+\left|s[x]-s[x-1]\right|^2}\\
\mbox{s.t.}& \max_{x\in[-L_1+1,L_2-1]} \left|s[x]-r[x]\right|\le \lambda\quad \mbox{and}\nonumber\\
&\quad s[-L_1]=0,\quad s[L_2-1] = r[L_2-1]\nonumber
\end{align}
where $r[x] = \sum_{y=-L_1+1}^{x} f[y]$ with $x\in [-L_1+1,L_2-1]$.

Given the explicit form of $f$ in eq.~\eqref{eq:deff} we obtain that
\begin{equation}
\label{eq:defr}
r[x] = \left \{ 
\begin{array}{ll}
U_1(x + L_1) \\ \quad\quad\quad\quad\quad\quad\quad\quad\quad x \in [-L_1 + 1, -2]\\
U_1 (L_1-1) + \delta_1 (U_2 - U_1) \\ \quad\quad\quad\quad\quad\quad\quad\quad\quad x = -1\\
U_2(x + 1) + (\delta_1 - \delta_2)(U_2 - U_1) + U_1 (L_1 - 1) \\ \quad\quad\quad\quad\quad\quad\quad\quad\quad x \in [0, L_2-1].
\end{array}
\right .
\end{equation}
Notice that $r[x]$ has three discontinuities at $x = -2$, $x = -1$ and $x = 0$. Let consider solving the taut string problem by enforcing in turn only the constraint $|s[-2]-r[-2]|\le \lambda$,  $|s[-1]-r[-1]|\le \lambda$ and $|s[0]-r[0]|\le \lambda$. 

For the first case the cost of the taut string problem is minimum for the shortest path $s$ through a point at $x=-2$. We can decompose such path into the concatenation of the shortest path from $x=-L_1+1$ to $x=-2$ and the shortest path from $x=-2$ to $x=L_2-1$. Given that each of these paths are only constrained at the end points, a direct solution will give a line segment between the end points, \ie,
\begin{equation}
\label{eq:defs}
s[x] = \left \{ 
\begin{array}{rl}
\frac{x+L_1}{L_1-2} s[-2] & x \in [-L_1+1, -2]\\
\frac{L_2-1-x}{L_2+1} s[-2] + \frac{x+2}{L_2+1}r[L_2-1] & x \in [-1, L_2-1].
\end{array}
\right.
\end{equation}
The value $s[-2]$ that yields the shortest path and satisfies the constraint
\begin{align}
r[-2] - \lambda \leq s[-2] \leq r[-2] +\lambda
\end{align}
is $s[-2] = U_1(L_1 - 2) +\lambda$ when $\lambda < (U_2-U_1)\frac{L_1-2}{L_1+L_2-1}(L_2+\delta_1-\delta_2)$
and $s[-2] = U_1(L_1 - 2) +  (U_2-U_1)\frac{L_1-2}{L_1+L_2-1}(L_2+\delta_1-\delta_2)$ otherwise.

Now, we will show that, given $\lambda\ge (U_2-U_1) (L_2-L_2\delta_1-\delta_2)$, the above shortest path $s$ is also the solution $\hat{s}$ to the  taut string problem~\eqref{eq:dual} with all the constraints. It will suffice to show that this path satisfies all the constraints in the taut string problem. Then, since it is the shortest path with a single constraint, it must also be the shortest path for problem~\eqref{eq:dual}.
To verify all the constraints, we only need to consider $2$ cases:
\begin{equation}
\label{eq:cases}
\begin{array}{l}
\textstyle x=-1\rightarrow\nonumber\\ \quad \textstyle \left| (L_2 - L_2 \delta_1 - \delta_2)(U_2 - U_1) + \lambda L_2 \right| \le \lambda (L_2+1)\nonumber\\
\textstyle x=0\rightarrow\nonumber\\ \quad \textstyle \left|(L_2-1)(1 - \delta_1 + \delta_2) (U_2 - U_1)+ (L_2-1)\lambda \right| \le\lambda(L_2+1)\nonumber
\end{array}
\end{equation}
as all the others are directly satisfied when these are. The first inequality leads to the condition $\lambda\ge (U_2-U_1) (L_2-L_2\delta_1-\delta_2)$ and the second inequality to the condition $\lambda\ge (U_2-U_1)\frac{L_2-1}{2}(1-\delta_1+\delta_2)$. However, as long as $L_1,L_2 > 2$, $(U_2-U_1) (L_2-L_2\delta_1-\delta_2) \ge (U_2-U_1)\frac{L_2-1}{2}(1-\delta_1+\delta_2)$ for $\delta_1 + \delta_2 \le 1$, therefore it is sufficient to have the condition $\lambda\ge (U_2-U_1) (L_2-L_2\delta_1-\delta_2)$ to satisfy both inequalities. We can then obtain $\bar{u}[x] = \hat{s}[x]-\hat{s}[x-1]$ from eq.~\eqref{eq:defs} and write
\begin{equation}\label{eq:solum1}
\bar{u}[x] = \left \{ 
\begin{array}{ll}
U_1+\frac{\lambda}{L_1-2} & x \in [-L_1+1, -2]\\
\frac{U_1 + U_2L_2}{L_2+1}+\frac{(\delta_1-\delta_2)(U_2-U_1)-\lambda}{L_2+1} & x \in [-1, L_2-1]
\end{array}
\right .
\end{equation}
If $\frac{L_1-2}{L_1+L_2-1}(L_2+\delta_1-\delta_2) \le L_2-L_2\delta_1-\delta_2$ the two conditions on $\lambda$ mentioned above can never be satisfied. With some algebraic manipulation we obtain that if $\delta_2 \le L_2 - (L_1 + L_2 - 2) \delta_2$ a $\lambda$ such that the solution~\eqref{eq:solum1} is obtained does not exist. 

In a similar manner we can consider the second point $x = -1$ and solve the taut string problem imposing only the constraint $|s[-1] - r[-1]| \le \lambda$. Also in this case we can decompose the shortest path $s$ into the concatenation of the shortest path from $x = -L_1+1$ to $x = -1$ and then one from $x = -1$ to $x = L_2 -1$. A direct solution will give a line segment between the end points,
\begin{equation}
\label{eq:defs}
s[x] = \left \{ 
\begin{array}{rl}
\frac{x+L_1}{L_1-1} s[-1] & x \in [-L_1+1, -1]\\
\frac{L_2-1-x}{L_2} s[-1] + \frac{x+1}{L_2}r[L_2-1] & x \in [0, L_2-1].
\end{array}
\right.
\end{equation}
The value $s[-1]$ that yields the shortest path and satisfies the constraint
\begin{align}
r[-1] - \lambda \leq s[-1] \leq r[-1] +\lambda
\end{align}
is $s[-1] = U_1(L_1 - 1) + \delta_1(U_2-U_1) +\lambda$ when $\lambda < (U_2-U_1)\frac{L_2(L_1-1)-L_2\delta_1-(L_1-1)\delta_2}{L_1+L_2-1}$
and $s[-1] = U_1(L_1 - 1) + \delta_1(U_2-U_1) +  (U_2-U_1)\frac{L_2(L_1-1)-L_2\delta_1-(L_1-1)\delta_2}{L_1+L_2-1}$ otherwise.

Now, we will show that, given $ \lambda\ge (U_2-U_1)\max\left\{(L_1-2)\delta_1,(L_2-1)\delta_2\right\}$, the above shortest path $s$ is also the solution $\hat{s}$ to the  taut sting problem~\eqref{eq:dual} with all the constraints.To verify all the constraints, we only need to consider $2$ cases:
\begin{equation}
\label{eq:cases}
\begin{array}{l}
\textstyle x=-2\rightarrow\nonumber\\ \quad \textstyle \left| (L_1-2) (\delta_1 (U_2 - U_1) +  \lambda) \right| \le \lambda (L_1-1)\nonumber\\
\textstyle x=0\rightarrow\nonumber\\ \quad \textstyle \left| (L_2-1) (\delta_2(U_2 - U_1) + \lambda)) \right| \le \lambda L_2\nonumber
\end{array}
\end{equation}
as all the others are directly satisfied when these are. By direct substitution, one can find that  $ \lambda\ge (U_2-U_1)\max\left\{(L_1-2)\delta_1,(L_2-1)\delta_2\right\}$ satisfies all the above constraints as long as $L_1,L_2 > 2$. We can then obtain $\bar{u}[x] = \hat{s}[x]-\hat{s}[x-1]$ from eq.~\eqref{eq:defs} and write
\begin{equation}\label{eq:solum2}
\hat{u}[x] = \left \{ 
\begin{array}{ll}
U_1+\frac{\delta_1(U_2-U_1)+\lambda}{L_1-1} & x \in [-L_1, -1]\\
U_2+\frac{-\delta_2(U_2-U_1)-\lambda}{L_2} & x \in [0, L_2]
\end{array}
\right .
\end{equation}
Also for this case there are configurations of $L_1$, $L_2$, $\delta_1$ and $\delta_2$ for which a $\lambda$ that gives~\eqref{eq:solum2} does not exist. We distinguish two cases: if  $(L_2-1)\delta_2 < (L_1-2)\delta_1$ then the condition $\frac{L_2(L_1-1)-L_2\delta_1-(L_1-1)\delta_2}{L_1+L_2-1} \le (L_2-1)\delta_2 $ must be satisfied; or, if  $(L_2-1)\delta_2 \ge (L_1-2)\delta_1$, the condition  $\frac{L_2(L_1-1)-L_2\delta_1-(L_1-1)\delta_2}{L_1+L_2-1} \le (L_2\1-2)\delta_1 $ must be satisfied. With simple algebraic manipulation we have for the first case if $\delta_2 < \frac{(L_1 - 2)\delta_1}{L_2-1}$ that $\delta_2 \ge L_2 -  (L_1 + L_2 -2)\delta_1$, or if  $\delta_2 \ge \frac{(L_1 - 2)\delta_1}{L_2-1}$ that $\delta_2 > \frac{L_1  - \delta_1 - 1}{L_1 + L_2 - 2}$ must be satisfied.

For the last point $x = 0$ we solve the taut string problem imposing only the constraint $|s[0] - r[0]| \le \lambda$, having a direct solution as following
\begin{equation}
\label{eq:defs}
s[x] = \left \{ 
\begin{array}{rl}
\frac{x+L_1}{L_1} s[0] & x \in [-L_1+1, 0]\\
\frac{L_2-1-x}{L_2-1} s[0] + \frac{x}{L_2-1}r[L_2-1] & x \in [1, L_2-1].
\end{array}
\right.
\end{equation}
The value $s[0]$ that yields the shortest path and satisfies the constraint
\begin{align}
r[0] - \lambda \leq s[0] \leq r[0] +\lambda
\end{align}
is $s[0] = U_2  + (\delta_1 - \delta_2)(U_2-U_1) + U_1(L_1 - 1)+\lambda$ when $\lambda < (U_2-U_1)\frac{L_2-1}{L_1+L_2-1}(L_1-\delta_1+\delta_2-1)$
and $s[-0] =  U_2  + (\delta_1 - \delta_2)(U_2-U_1) + U_1(L_1 - 1) + (U_2-U_1)\frac{L_2-1}{L_1+L_2-1}(L_1-\delta_1+\delta_2-1)$ otherwise.

For $  \lambda \ge (U_2-U_1) (L_1-\delta_1-(L_1-1)\delta_2-1)$, the above shortest path $s$ is also the solution $\hat{s}$ to the  taut sting problem~\eqref{eq:dual} with all the constraints. Indeed, it satisfies the following $2$ cases:
\begin{equation}
\label{eq:cases}
\begin{array}{l}
\textstyle x=-2\rightarrow\nonumber\\ \quad \textstyle \left| (L_1-2)(\delta_1-\delta_2+1)(U_2-U_1)+ (L_2-2)\lambda \right| \le \lambda L_1\nonumber\\
\textstyle x=-1\rightarrow\nonumber\\ \quad \textstyle \left| (L_1-1-\delta_1-(L_1-1)\delta_2)(U_2-U_1)+(L_1-1)\lambda \right| \le \lambda L_1\nonumber
\end{array}
\end{equation}
for $L_1,L_2 > 2$ and $\delta_1 + \delta_2 \le 1$. We can then obtain $\bar{u}[x] = \hat{s}[x]-\hat{s}[x-1]$ from eq.~\eqref{eq:defs} and write
\begin{equation}\label{eq:solum3}
\hat{u}[x] = \left \{ 
\begin{array}{ll}
\frac{U_1(L_1-1)+U_2+(\delta_1-\delta_2)(U_2-U_1)+\lambda}{L_1} & x \in [-L_1, 0]\\
U_2-\frac{\lambda}{L_2-1} & x \in [1, L_2]
\end{array}
\right .
\end{equation}
If $\frac{L_2-1}{L_1+L_2-1}(L_1-\delta_1+\delta_2-1) \le L_1-\delta_1-(L_1-1)\delta_2-1$ the two conditions on $\lambda$ mentioned above can never be satisfied. With some algebraic manipulation we obtain that if $\delta_2 \le \frac{(L_1-1) - \delta_1}{L_1 + L_2 - 2}$ a $\lambda$ such that the solution~\eqref{eq:solum3} is obtained does not exist. 

If the conditions $\delta_2 \le \frac{L_1  - \delta_1 - 1}{L_1 + L_2 - 2}$, $\delta_2 \le L_2 - (L_1 + L_2 -2) \delta_1$, $\delta_2 < \frac{(L_1 - 2)\delta_1}{L_2-1}$ and $\delta_2 \ge L_2 -  (L_1 + L_2 -2)\delta_1$ or $\delta_2 \ge \frac{(L_1 - 2)\delta_1}{L_2-1}$ and $\delta_2 \ge \frac{L_1  - \delta_1 - 1}{L_1 + L_2 - 2}$ are all satisfied none of the solutions~\eqref{eq:solum1},~\eqref{eq:solum2} and~\eqref{eq:solum3} can be obtained. That is the case only if the conditions $\delta_2 = \frac{L_1  - \delta_1 - 1}{L_1 + L_2 - 2}$ or $\delta_2 = L_2 - (L_1 + L_2 -2) \delta_1$ are true.
\end{IEEEproof}

\subsection{Proof of Theorem~\eqref{the:grad} }
\label{proof:the:grad}
\begin{IEEEproof}
From the definition of $f_x[x]$ and $f[x]$ we have
\begin{equation}
f_x[x] = \left \{ 
\begin{array}{ll}
0 & x \in[-L_1+1,-3]\\
\delta_1(U_2-U_1) & x = -2\\
(1 - \delta_1 - \delta_2)(U_2-U_1) & x = -1\\
\delta_2(U_2 - U_1) & x = 0\\
0 & x \in[1,L_2 - 2]\\
\end{array}
\right .
\end{equation}
The problem~\eqref{eq:rof0d} is equivalent to solving 
\begin{align} \label{eq:rof0dc}
\bar{u}_x[x] = \arg\min_{u_x} \frac{1}{2}\sum_{x = -L_1+1}^{L_2-2} (u_x[x] - f_x[x])^2\quad\quad\quad\quad~\\
 + \lambda \sum_{x = -L_1+1}^{L_2-2} |u_x[x]|.\nonumber
\end{align}
where $\hat{u}_x[x] = \bar{u}_x[x]$ for $x \in [-L_1+1, L_2-2]$ and $\hat{u}_x[-L_1] = 0$, $\hat{u}_x[L_2 -1] = 0$.

The solution of problem~\eqref{eq:rof0dc} can be computed in closed form and it is equal to 

\begin{equation} \label{eq:soft}
\bar{u}_x[x] = \max(f_x[x] - \lambda \sign(f_x[x]),0).
\end{equation}
It is possible to obtain three different Delta dirac functions using~\eqref{eq:soft}. 

If  $\delta_1 > \max(\delta_2, \frac{1 - \delta_2}{2})$, applying~\eqref{eq:soft} with $\lambda \ge (U_2 - U_1)\max(\delta_1, 1 - \delta_1 - \delta_2)$ would lead to
\begin{equation}
u_x[x] = \left \{ 
\begin{array}{ll}
0 & x \neq -2\\
(\delta_1 - \max(\delta_2, 1 - \delta_1 - \delta_2)) (U_2 - U_1) & x = -2
\end{array}
\right .
\end{equation}

If  $1 > \max(2 \delta_1 +\delta_2,2 \delta_2 +\delta_1)$, applying~\eqref{eq:soft} with $\lambda \ge (U_2 - U_1)\max(\delta_1, \delta_2)$ would lead to
\begin{equation}
u_x[x] = \left \{ 
\begin{array}{ll}
0 & x \neq -1\\
(1 - \delta_1 - \delta_2 - \max(\delta_1,\delta_2))(U_2 - U_1) & x = -1
\end{array}
\right .
\end{equation}

If  $\delta_2 > \max(\delta_1, \frac{1 - \delta_1}{2})$, applying~\eqref{eq:soft} with $\lambda \ge (U_2 - U_1)\max(\delta_1, 1 - \delta_1 - \delta_2)$ would lead to
\begin{equation}
u_x[x] = \left \{ 
\begin{array}{ll}
0 & x \neq 0\\
(\delta_2 - \max(\delta_1, 1 - \delta_1 - \delta_2))(U_2 - U_1) & x = 0
\end{array}
\right .
\end{equation}

If none of the conditions $\delta_1 > \max(\delta_2, \frac{1 - \delta_2}{2})$, $1 > \max(2 \delta_1 +\delta_2,2 \delta_2 +\delta_1)$ and $\delta_2 > \max(\delta_1, \frac{1 - \delta_1}{2})$ is satisfied then it is not possible to obtain a Dirac delta function from $f_x[x]$ solving problem~\eqref{eq:rof0dc}. This corresponds to the conditions $\delta_2 = \delta_1 \ge 1/3$, $\delta_1 = (1 - \delta_2)/2 \ge 1/3$ or $\delta_2 = (1 - \delta_1)/2 \ge 1/3$.
\end{IEEEproof}

\subsection{Proof of Theorem~\eqref{the:am} }
\label{proof:the:am} 
\begin{IEEEproof}
In this proof we will show that the cost $||k \ast \hat{u} - f||^2_2$ is minimum for $k = \delta$, where the constraint $\sum_x k[x] = 1$ is enforced and $\hat{u}$ is obtained by solving the first step of the AM algorithm~\eqref{eq:am_u} for the given values of $\lambda$. To make calculations easier we write  $k$ as a 3-element blur kernel where $\hat{\delta}_1\doteq k[1]$, $\hat{\delta}_2\doteq k[-1]$ and $k[0] = 1-\hat{\delta}_1-\hat{\delta}_2$, $\hat{\delta_1} + \hat{\delta}_2 \le 1$ and $\hat{\delta}_1,\hat{\delta}_2 \ge 0$. Notice that in this form the constraint $\sum_x k[x] = 1$ is implicitly enforced.

For a $\lambda \in [\lambda_{min}^c,\lambda_{max}^c)$ from Proposition~\ref{proposition} we have that the minimizer of problem~\eqref{eq:am_u} is
\begin{equation}
\hat{u}[x] = \left \{ 
\begin{array}{ll}
\hat{U}_1 & x \in [-L_1, -1]\\
\hat{U}_2 & x \in [0, L_2].
\end{array}
\right .
\end{equation}
The cost $||k \ast \hat{u} - f||^2_2$, can be then split in $4$ regions
\begin{equation}\label{eq:cost1}
\begin{array}{l}
||k \ast \hat{u} - f||^2_2  =  \sum_x ((k \ast \hat{u})[x] - f[x])^2 = \\
+ (L_1 - 3) (\hat{U_1} - U_1)^2 \\
+ (\hat{U}_1 + \hat{\delta}_1 (\hat{U}_2 - \hat{U}_1) - ( U_1 +\delta_1(U_2 - U_1)))^2\\
+ ((\hat{U}_2 - \hat{\delta}_2 (\hat{U}_2 - \hat{U}_1) - (U_2 - \delta_2(U_2 - U_1)))^2\\
+ (L_2 - 2) (\hat{U}_2 - U_2)^2\\
\end{array}
\end{equation}

The first and forth terms do not depend on $k$, so only the other two terms contribute to the estimation of $k$. Notice that the inequalities $\hat{U}_1 \ge U_1 +\delta_1(U_2 - U_1) > U_1$, and $\hat{U}_2 \le U_2 - \delta_2(U_2 - U_1) < U_2$ hold. This means that $\hat{U}_1 - ( U_1 +\delta_1(U_2 - U_1))$ is positive, therefore the value of $\hat \delta_1 \ge 0$ that minimizes the second term is $\hat \delta_1 = 0$, and, because $\hat{U}_2 - ( U_2 -\delta_2(U_2 - U_1))$ is negative, the value of $\hat \delta_2 \ge 0$ that minimizes the third term is $\hat \delta_2 = 0$. This shows that the $k$ that minimizes the cost~\eqref{eq:cost1} is the Dirac delta where $k[1] = \hat{\delta}_1 = 0$, $k[-1] =\hat{\delta}_2 = 0$ and $k[0] = 1$.

For a$\lambda \in [\lambda_{min}^l,\lambda_{max}^l)$ from Proposition~\ref{proposition} we have that the minimizer of problem~\eqref{eq:am_u} is
\begin{equation}
\hat{u}[x] = \left \{ 
\begin{array}{ll}
\hat{U}_1 & x \in [-L_1, -2]\\
\hat{U}_2 & x \in [-1, L_2].
\end{array}
\right .
\end{equation}
The cost $||k \ast \hat{u} - f||^2_2$, can be then split in $5$ regions
\begin{equation}\label{eq:cost2}
\begin{array}{l}
||k \ast \hat{u} - f||^2_2  =  \sum_x ((k \ast \hat{u})[x] - f[x])^2 = \\
+ (L_1 - 2) (\hat{U_1} - U_1)^2 \\
+ (\hat{U}_1 + \hat{\delta}_1 (\hat{U}_2 - \hat{U}_1) - U_1)^2\\
+ ((\hat{U}_2 - \hat{\delta}_2 (\hat{U}_2 - \hat{U}_1) - (U_1 + \delta_1(U_2 - U_1)))^2\\
+ ((\hat{U}_2 - (U_2 - \delta_2(U_2 - U_1)))^2\\
+ (L_2 - 2) (\hat{U}_2 - U_2)^2\\
\end{array}
\end{equation}
The first, forth and fifth terms do not depend on $k$, so only the other two terms contribute to the estimation of $k$. Notice that the inequalities $\hat{U}_1  > U_1$ and $\hat{U}_2 \le U_1 + \delta_1(U_2 - U_1) < U_2 - \delta_2(U_2 - U_1) < U_2$ hold. This means that, because $\hat{U}_1 - U_1$ is positive, the value of $\hat \delta_1 \ge 0$ that minimizes the third term is $\hat \delta_1 = 0$, and that $\hat{U}_2  - \hat{U}_2 - U_1 + \delta_1(U_2 - U_1)$ is negative, therefore the value of $\hat \delta_2 \ge 0$ that minimizes the second term is $\hat \delta_2 = 0$. This shows that the $k$ that minimizes the cost~\eqref{eq:cost1} is a  Dirac delta where $k[1] = \hat{\delta}_1 = 0$, $k[-1] =\hat{\delta}_2 = 0$ and $k[0] = 1$.

For a $\lambda \in [\lambda_{min}^r,\lambda_{max}^r)$ from Proposition~\ref{proposition} we have that the minimizer of problem~\eqref{eq:am_u} is
\begin{equation}
\hat{u}[x] = \left \{ 
\begin{array}{ll}
\hat{U}_1 & x \in [-L_1, 0]\\
\hat{U}_2 & x \in [1, L_2].
\end{array}
\right .
\end{equation}
The cost $||k \ast \hat{u} - f||^2_2$, can be then split in $5$ regions
\begin{equation}\label{eq:cost3}
\begin{array}{l}
||k \ast \hat{u} - f||^2_2  =  \sum_x ((k \ast \hat{u})[x] - f[x])^2 = \\
+ (L_1 - 3) (\hat{U_1} - U_1)^2 \\
+ (\hat{U}_1  - (U_1 + \delta_1(U_2 - U_1)))^2\\
+ (\hat{U}_1 + \hat{\delta}_1 (\hat{U}_2 - \hat{U}_1) - (U_2 - \delta_2(U_2 - U_1)))^2\\
+ ((\hat{U}_2 - \hat{\delta}_2 (\hat{U}_2 - \hat{U}_1) - U_2)^2\\
+ (L_2 - 3) (\hat{U}_2 - U_2)^2\\
\end{array}
\end{equation}
The first, second and fifth terms do not depend on $k$, so only the other two terms contribute to the estimation of $k$. Notice that the inequalities $\hat{U}_1 \ge U_2 - \delta_2(U_2 - U_1) > U_1 + \delta_1(U_2 - U_1) > U_1$ and $\hat{U}_2 < U_2$ hold. This means that  $\hat{U}_1 - (U_2 - \delta_2(U_2 - U_1))$ is positive, therefore the value of $\hat \delta_1 \ge 0$ that minimizes the third term is $\hat \delta_1 = 0$, and that $\hat{U}_2 - U_2$ is negative, therefore the value of $\hat \delta_2 \ge 0$ that minimizes the second term is $\hat \delta_2 = 0$. This shows that the $k$ that minimizes the cost~\eqref{eq:cost3} is a Dirac delta where $k[1] = \hat{\delta}_1 = 0$, $k[-1] =\hat{\delta}_2 = 0$ and $k[0] = 1$.
\end{IEEEproof}

\subsection{Proof of Theorem~\eqref{the:pam}}
\label{proof:the:pam}

\begin{IEEEproof}
Notice that If $u^0$ is a zero-mean signal, because total variation denoising preserves the mean of the original signal and $\sum_x k^0[x] = 1$ we have that also $\hat{u}$ is a zero-mean signal. We can consider the different conditions on $\lambda$ separately. 

For a $\lambda \in [\lambda_{min}^l,\lambda_{max}^l)$ from Proposition~\ref{proposition} we have that the minimizer of problem~\eqref{eq:am_u} is
\begin{equation}
\hat{u}[x] = \left \{ 
\begin{array}{ll}
\hat{U}_1 & x \in [-L_1, -2]\\
\hat{U}_2 & x \in [-1, L_2]
\end{array}
\right .
\end{equation}

Since we can always express a zero-mean step as another scaled zero-mean step, we can write
\begin{equation}
\hat{u}[x] = a u^0[x+1]
\end{equation}
for some constant $a$. We then solve the second step of the PAM algorithm
\begin{equation} 
\hat{k} = \arg\min_k ||k \ast \hat{u} - f||_2^2
\end{equation}
where we can write 
\begin{equation} 
\begin{array}{l}
\|(k \ast \hat{u})[x] - f[x]\|_2^2  \nonumber\\ 
= \|a \sum_y k[y] u^0[x+1 - y] - \sum_y k[y] u^0[x - y]\|_2^2\nonumber\\
  = \|a \sum_y k[y-1] u^0[x - y] - \sum_y k^0[y] u^0[x - y]\|_2^2\nonumber\\
 = \|\sum_y(a k[y-1]- k^0[y]) u^0[x-y]\|_2^2
 \end{array}
\end{equation}
and have $\hat{k}[x-1] = k^0[x]/a$. Finally, by applying the last two steps of the PAM algorithm one obtains $\hat{k}[x-1] = k^0[x]$.

For $\lambda \in [\lambda_{min}^c,\lambda_{max}^c)$ from Proposition~\ref{proposition} we have that the minimizer of problem~\eqref{eq:am_u} is
\begin{equation}
\hat{u}[x] = \left \{ 
\begin{array}{ll}
\hat{U}_1 & x \in [-L_1, -1]\\
\hat{U}_2 & x \in [0, L_2]
\end{array}
\right .
\end{equation}

In this case we can write
\begin{equation}
\hat{u}[x] = a u^0[x]
\end{equation}
for some constant $a$. We then solve the second step of the PAM algorithm
\begin{equation} 
\hat{k} = \arg\min_k ||k \ast \hat{u} - f||_2^2
\end{equation}
where we can write 
\begin{equation} 
\begin{array}{l}
\|(k \ast \hat{u})[x] - f[x]\|_2^2  \nonumber\\ 
= \|a \sum_y k[y] u^0[x - y] - \sum_y k[y] u^0[x - y]\|_2^2\nonumber\\
  = \|a \sum_y k[y] u^0[x - y] - \sum_y k^0[y] u^0[x - y]\|_2^2\nonumber\\
 = \|\sum_y(a k[y]- k^0[y]) u^0[x-y]\|_2^2
 \end{array}
\end{equation}
and have $\hat{k}[x] = k^0[x]/a$. Finally, by applying the last two steps of the PAM algorithm one obtains $\hat{k}[x] = k^0[x]$.

Finally, for a $\lambda \in [\lambda_{min}^r,\lambda_{max}^r)$ from Proposition~\ref{proposition} we have that the minimizer of problem~\eqref{eq:am_u} is
\begin{equation}
\hat{u}[x] = \left \{ 
\begin{array}{ll}
\hat{U}_1 & x \in [-L_1, 0]\\
\hat{U}_2 & x \in [1, L_2]
\end{array}
\right .
\end{equation}

In this case we can write
\begin{equation}
\hat{u}[x] = a u^0[x-1]
\end{equation}
for some constant $a$. We then solve the second step of the PAM algorithm
\begin{equation} 
\hat{k} = \arg\min_k ||k \ast \hat{u} - f||_2^2
\end{equation}
where we can write 
\begin{equation} 
\begin{array}{l}
\|(k \ast \hat{u})[x] - f[x]\|_2^2  \nonumber\\ 
= \|a \sum_y k[y] u^0[x -1 - y] - \sum_y k[y] u^0[x - y]\|_2^2\nonumber\\
  = \|a \sum_y k[y+1] u^0[x - y] - \sum_y k^0[y] u^0[x - y]\|_2^2\nonumber\\
 = \|\sum_y(a k[y+1]- k^0[y]) u^0[x-y]\|_2^2
 \end{array}
\end{equation}
and have $\hat{k}[x+1] = k^0[x]/a$. Finally, by applying the last two steps of the PAM algorithm one obtains $\hat{k}[x+1] = k^0[x]$.  
\end{IEEEproof}

\ifCLASSOPTIONcaptionsoff
  \newpage
\fi

\bibliographystyle{IEEEtran}
\bibliography{papers}
\end{document}